\documentclass[11pt]{article}
\usepackage[utf8]{inputenc} 
\usepackage[T1]{fontenc}    
\usepackage{wrapfig}

\usepackage{setspace}
\usepackage{cprotect}
\usepackage{amsmath,amssymb,amsthm}
\usepackage[noend]{algorithmic}
\usepackage[ruled,vlined]{algorithm2e}
\usepackage{hyperref}
\usepackage{fullpage}
\usepackage{makeidx}
\usepackage{enumerate}
\usepackage{graphicx,float,psfrag,epsfig}
\usepackage{epstopdf}
\usepackage{color}
\usepackage{enumitem}
\usepackage{caption}
\usepackage{bigints}
\usepackage{mathtools}
\usepackage[mathscr]{euscript}

\newcommand{\simone}[1]{#1}
\usepackage{amsmath}
\usepackage{amssymb}
\usepackage{mathtools}
\usepackage{amsthm}
\usepackage{amsfonts}
\usepackage{soul}
\usepackage{csquotes}
\usepackage{amsmath}
\usepackage{amsthm}
\usepackage{amsfonts}
\usepackage{hyperref}
\usepackage{comment}
\usepackage{graphicx}
\usepackage{xcolor}
\usepackage{parskip}
\usepackage{hyperref}       
\usepackage{url}            
\usepackage{booktabs}       
\usepackage{amsfonts}       
\usepackage{nicefrac}       
\usepackage{microtype}      
\usepackage{xcolor}         
\usepackage{microtype}
\usepackage{graphicx}
\usepackage{subfigure}  
\usepackage{amsfonts}
\usepackage{amsmath}
\usepackage{amssymb}
\usepackage{bbm}
\usepackage{multirow}
\usepackage{mathtools}
\usepackage{relsize}

\DeclareMathOperator{\Span}{span}
\DeclareMathOperator{\Rows}{rows}

\def\Pp{P_{\Phi}}
\def\Ppm{P_{\Phi_{-1}}}

\newcommand{\kntk}{K_{\textup{NTK}}}

\newcommand{\varrf}{\varphi_{\textup{RF}}}
\newcommand{\tvarrf}{\tilde \varphi_{\textup{RF}}}
\newcommand{\varntk}{\varphi_{\textup{NTK}}}
\newcommand{\tvarntk}{\tilde \varphi_{\textup{NTK}}}

\def\P{\mathbb{P}}

\def\R{\mathbb{R}}

\newcommand{\opnorm}[1]{\left\lVert#1\right\rVert_{\textup{op}}}

\def\b0{{0}}

\def\RR{\mathbb{R}}

\def\>{\rangle}

\newcommand{\E}{\mathbb{E}}

\newcommand{\distas}[1]{\mathbin{\overset{#1}{\sim}}}

\newcommand{\bigO}[1]{\mathcal{O}\left(#1\right)}

\def\Span{\textrm{Span}}

\newcommand{\norm}[1]{\left\|#1\right\|}
\newcommand{\subGnorm}[1]{\left\|#1\right\|_{\psi_2}}
\newcommand{\subEnorm}[1]{\left\|#1\right\|_{\psi_1}}

\newcommand{\abs}[1]{\left|#1\right|}

\newcommand{\evmin}[1]{\lambda_{\rm min}\left(#1\right)}

\def\Lip{\mathrm{Lip}}

\def\PP{\mathbb{P}}

\def\det{\mathop{\rm det}\nolimits}

\def\min{\mathop{\rm min}\nolimits}
\def\max{\mathop{\rm max}\nolimits}

\numberwithin{equation}{section}

\newtheoremstyle{myexample} 
    {\topsep}                    
    {\topsep}                    
    {\rm }                   
    {}                           
    {\bf }                   
    {.}                          
    {.5em}                       
    {}  

\newtheoremstyle{myremark} 
    {\topsep}                    
    {\topsep}                    
    {\rm}                        
    {}                           
    {\bf}                        
    {.}                          
    {.5em}                       
    {}  

\newtheorem{claim}{Claim}[section]
\newtheorem{lemma}[claim]{Lemma}

\newtheorem{assumption}{Assumption}

\newtheorem{theorem}{Theorem}
\newtheorem{proposition}[claim]{Proposition}

\newtheorem{definition}[claim]{Definition}

\theoremstyle{myremark}

\theoremstyle{myremark}

\theoremstyle{myexample}

\author{Simone Bombari\thanks{Institute of Science and Technology Austria (ISTA). Emails: \texttt{\{simone.bombari, marco.mondelli\}@ist.ac.at}.}\;,
\;\;Marco Mondelli\footnotemark[1]}

\title{How Spurious Features Are Memorized: \\
Precise Analysis for Random and NTK Features}

\begin{document}

\newtheorem*{theoremcentering}{Theorem \ref{thm:maincentering}}
\newtheorem*{theoremcentered}{Theorem \ref{thm:centered}}
\newtheorem*{corhammer}{Corollary \ref{cor:hammer}}
\newtheorem*{cormem}{Corollary \ref{cor:memcap}}
\newtheorem*{theoremoptim}{Theorem \ref{thm:optimization}}

\maketitle

\begin{abstract}

Deep learning models are known to overfit and memorize spurious features in the training dataset. While numerous empirical studies have aimed at understanding this phenomenon, a rigorous theoretical framework to quantify it is still missing. In this paper, we consider spurious features that are uncorrelated with the learning task, and we provide a precise characterization of how they are memorized via two separate terms: \emph{(i)} the \emph{stability} of the model with respect to individual training samples, and \emph{(ii)} the \emph{feature alignment} between the spurious feature and the full sample. While the first term is well established in learning theory and it is connected to the generalization error in classical work, the second one is, to the best of our knowledge, novel. Our key technical result gives a precise characterization of the feature alignment for the two prototypical settings of random features (RF) and neural tangent kernel (NTK) regression. We prove that the memorization of spurious features weakens as the generalization capability increases and, through the analysis of the feature alignment, we unveil the role of the model and of its activation function. 
Numerical experiments show the predictive power of our theory on standard datasets (MNIST, CIFAR-10).
\end{abstract}

\section{Introduction}

Neural networks often use features that are not inherently relevant for the intended task. This phenomenon can be caused by positive spurious correlations between certain patterns and the learning task \cite{Geirhos_2020, xiao2021noise}, but it occurs even when the patterns are rare \cite{yang2022understanding} or simply irrelevant \cite{Hermann2020whatshapes}, leading the model to \emph{memorize} spurious relations present in the training data, which are not predictive for the sampling distribution. An extensive empirical effort has aimed at mitigating this phenomenon \cite{plumb2022finding, chang21augmentation}. In fact, the benefits of solving this problem range from robustness to distribution-shift \cite{geirhos2018imagenettrained, zhou21combating}, fairness \cite{zliobaite15}, and data-privacy \cite{Leino2020}. However, avoiding to overfit spurious features is not always feasible, since memorization can be optimal for accuracy and over-parameterized models often exhibit their best performance when trained long enough to achieve $0$ training error \cite{Nakkiran2020Deep, feldman2020}.

In this regard, a related (but separate) body of work has characterized the role of benign overfitting \cite{belkin2021, bartlett20benign}, and it has precisely described the in-distribution generalization 
of interpolating models, such as random features and neural tangent kernels \cite{mei2022generalization, ghorbani2021linearized, montanari2022interpolation}.
However, this powerful theoretical machinery does not cover the memorization of spurious features, as noise is generally modelled to be in the labels, rather than in the input data. More generally, while practical work has tried to understand the impact of spurious features and disentangle them from core features in deep learning models \cite{Hermann2020whatshapes, singla2022salient}, theoretical approaches remain predominantly directed to understand how learning is impacted by the complexity of the features \cite{qiu2023complexity}, or the degree of overparameterization \cite{sagawa2020overparam}, without capturing 
the role of the 
architecture.

Our paper bridges this gap, offering an analytically tractable framework to understand and quantify the memorization of spurious 
features. We consider a setting similar to \cite{yang2022understanding}, and we in particular look at the case where the spurious features are not correlated with the true label of the sample 
(thus the term \emph{memorization}). Formally, we model the sample $z$ as composed by two distinct parts, \emph{i.e.}, $z \equiv [x, y]$, where $x$ is the core feature and $y$ the spurious one, see Figure \ref{fig:cat} for an illustration. The memorization of spurious features is captured by the correlation between the true label $g$ of the training sample and the output of the model evaluated on the spurious sample $z^s \equiv [-, y]$, where ``$-$'' corresponds to removing the core feature $x$ (\emph{e.g.}, replacing it with all zeros). In fact, 
$z^s$ is independent of the label $g$, as the spurious feature $y$ is un-informative. However, due to memorization, the output of the model evaluated on $z^s$ can still be correlated with $g$. Our analysis describes quantitatively this phenomenon \simone{in the setting of} 
generalized linear \simone{regression}. 
Surprisingly, it turns out that 
the emergence of memorization can be reduced to the separate effect of two distinct components:


\begin{enumerate}
    \item The \emph{feature alignment} $\mathcal F(z^s, z)$, 
    see \eqref{eq:featalign}. This represents the similarity in feature space between the training sample $z$ and the spurious one $z^s$; it depends on the feature map of the model and on the rest of the training dataset. To the best of our knowledge, this is the first time that attention is raised over such an object.
    \item The \emph{stability} $\mathcal S_z$ of the model with respect to $z$, see Definition \ref{def:stability}. Similar notions of stability are provided in a rich line of work \cite{bousquet2002stability}, which relates them to generalization.
\end{enumerate}

Our technical contributions can be summarized as follows:
\begin{itemize}
    \item We connect the stability \simone{in generalized linear regression} to the feature alignment between samples, see Lemma \ref{lemma:proj}. Then, we show that this connection makes the memorization of spurious features 
    a natural consequence of the generalization error of the model. This is the case when $\mathcal F (z^s, z)$ can be well approximated by a constant $\gamma > 0$, independent of the original sample $z$.
    \item We focus on two settings widely analyzed in the theoretical literature, \emph{i.e.}, \emph{(i)} random features (RF) \cite{rahimi2007random}, and \emph{(ii)} the neural tangent kernel (NTK) \cite{JacotEtc2018}. Using tools from high dimensional probability, we prove the concentration of $\mathcal F (z^s, z)$ to a positive constant $\gamma$, see Theorems \ref{thm:RF} and \ref{thm:mainntk}. For the NTK, we obtain a closed-form expression for $\gamma$, which unveils the role of the activation function in the memorization of spurious features. 
\end{itemize}

In a nutshell, our results give a precise characterization of
the feature alignment of RF and NTK models. This in turn establishes how the memorization of spurious features grows with the generalization error, \simone{and how it depends on the chosen model (RF/NTK) and, in particular, on the activation function}. 
Finally, going beyond RF/NTK models trained on synthetic data, we empirically show that our theoretical predictions transfer to standard datasets (see Figure \ref{fig:ntk_real} that analyzes the impact of the activation function on MNIST and CIFAR-10) and different neural networks \simone{(see Figures \ref{fig:nn} and \ref{fig:nn_rebuttal} that consider fully connected, convolutional, and ResNet architectures).}

\begin{wrapfigure}{r}{0.45\textwidth}
  \vspace{-1.2em}
  \begin{center}
    \includegraphics[width=0.45\textwidth]{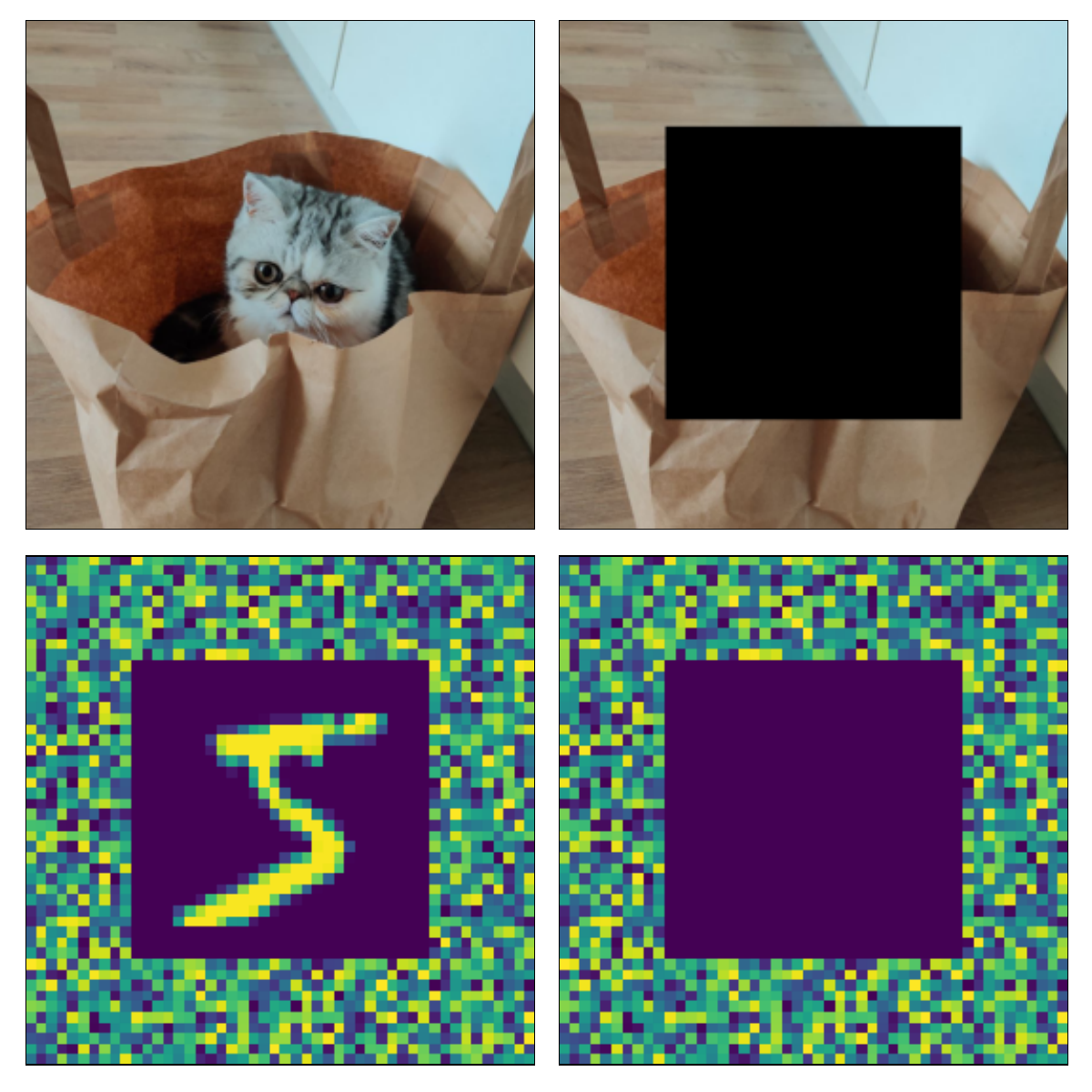}
  \end{center}
  \caption{Example of a training sample $z$ (top-left) and its spurious counterpart $z^s$ (top-right). 
  In experiments, we add a noise background ($y$) around the original images ($x$) before training (bottom-left). We then query the trained model only with the noise component (bottom-right).}
  \label{fig:cat}
\end{wrapfigure}


\section{Related work}

\paragraph{Spurious features.}

Spurious correlations refer to signals that are correlated but not causally related to the learning task \cite{Geirhos_2020, 
xiao2021noise}, and they have been shown to lead to poor out-of-distribution robustness \cite{geirhos2018imagenettrained, zhou21combating} or biased predictors \cite{zliobaite15, seonguk22bias, bishwa23biasfeat}. The phenomenon has been studied through the lens of overparameterization \cite{sagawa2020overparam} and simplicity bias \cite{Hermann2020whatshapes, shah2020simplicity, qiu2023complexity}, where the latter refers to models that are inherently prone to learn ``easy'' patterns first \cite{kalimeris19sgd}.

Our paper considers spurious features that are independent of the learning task and, hence, focuses on their memorization. Spurious features in the training set can in fact be memorized also if they are irrelevant or rare \cite{yang2022understanding, bansal2022measures}. This overfitting can then be used to retrieve information on the training set \cite{Leino2020, bombari2022differential}. 


\vspace{-0.2em}

\paragraph{Memorization and stability.}
Memorization 
measures the influence of a single sample on the final
trained model. \cite{feldman2020, feldman2020b} point out its advantages on learning from heavy-tailed data, and \cite{devansh17mem, stephenson2021on} investigate its emergence in neural networks. 
A related concept is that of leave-one-out stability, which has been studied in a classical line of work: \cite{hatmatrix} focus on under-parameterized linear models; \cite{bassily21, elisseeff2002, sayan2006} link it to generalization; and \cite{bousquet2002stability} discuss a wide range of variations on this object. 




\paragraph{Random features and neural tangent kernel.}
The random features (RF) model 
\cite{rahimi2007random, pennington17nonlinear, louart2018random} can be regarded as a two-layer neural network with random first layer weights. This model is theoretically appealing, as it is analytically tractable and offers deep-learning-like behaviours, such as, for example, the double descent phenomenon \cite{mei2022generalization}. 
The neural tangent kernel (NTK) can be regarded as the kernel obtained by linearizing a neural network around the initialization \cite{JacotEtc2018, bartlett2021deep}. 
A popular line of work has analyzed its spectrum \cite{fan2020spectra,adlam2020neural,wang2021deformed} and bounded its smallest eigenvalue \cite{theoreticalinsghts,tightbounds,montanari2022interpolation,bombari2022memorization}. 
The behaviour of the NTK has been used in practical work to study adversarial training \cite{loo2022evolution} and examples \cite{tsilivis2022what}, and to understand reconstruction attacks for dataset distillation \cite{loo2024understanding}.


\section{Preliminaries}\label{sec:prel}

\paragraph{Notation.}
Given a vector $v$, we denote by $\norm{v}_2$ its Euclidean norm. Given $v \in \R^{d_v}$ and $u \in \R^{d_u}$, we denote by $v \otimes u \in \R^{d_v d_u}$ their Kronecker product.
Given a matrix $A \in \R^{m\times n}$, we denote by $P_A \in \R^{n \times n}$ the projector over $\Span \{ \Rows (A) \}$. All the complexity notations $\Omega(\cdot)$, $\mathcal{O}(\cdot)$, $o(\cdot)$ and $\Theta(\cdot)$ are understood for sufficiently large data size $N$, input dimension $d$, number of neurons $k$, and number of parameters $p$. We indicate with $C,c>0$ numerical constants, independent of $N, d, k, p$.



\paragraph{Setting.}
Let $(Z, G)$ be a labelled training dataset, where $Z=[z_1, \ldots, z_N]^\top \in \R^{N \times d}$ contains the training data (sampled i.i.d.\ from a distribution $\mathcal P_Z$) on its rows and $G=(g_1, \ldots, g_N) \in \R^N$ contains the corresponding labels. We assume the label $g_i$ to be a (\simone{eventually noisy}) function of the sample $z_i$. 
Let $\varphi : \R^d \to \R^p$ be a generic feature map, from the input space to a feature space of dimension $p$. \simone{We consider the following \emph{generalized linear regression} model}
\begin{equation}\label{eq:glm}
    f(z, \theta) = \varphi(z)^\top \theta,
\end{equation}
where $\varphi(z) \in \R^p$ is the feature vector associated to the input 
$z$, and $\theta \in \R^p$ are 
trainable parameters of the model.
We 
minimize the empirical risk with a quadratic loss: 
\begin{equation}\label{eq:optloss}
    \min_\theta \norm{\Phi \theta - G}_2^2,
\end{equation}
where $\Phi := [\varphi(z_1), \ldots, \varphi(z_N)]^\top \in \R^{N \times p}$ is the feature matrix. 
We use the shorthand $K := \Phi \Phi^\top \in \R^{N \times N}$ for the kernel associated with the feature map. If $K$ is invertible (\emph{i.e.}, the model can fit any set of labels $G$), gradient descent converges to the interpolator which is the closest in $\ell_2$ norm to the initialization \simone{(see equation (33) in \cite{bartlett2021deep})}:
\begin{equation}\label{eq:thetastar}
    \theta^* = \theta_0 + \Phi^+ (G - f(Z, \theta_0)),
\end{equation}
where $\theta^*$ is the gradient descent solution, $\theta_0$ the initialization, $f(Z, \theta_0) := \Phi \theta_0$ the output of the model \eqref{eq:glm} at initialization, and $\Phi^+ := \Phi^\top K^{-1}$ the Moore-Penrose inverse.
Let $z \sim \mathcal P_Z$ be an independent test sample. Then, we define the \emph{generalization error} of the trained model as
\begin{equation}\label{eq:generr}
    \mathcal R = \E_{z \sim \mathcal P_Z} \left[ \left(f(z, \theta^*) - g \right)^2 \right],
\end{equation}
where $g$ denotes the ground-truth label of the test sample $z$.


\paragraph{Stability.} Let us 
introduce quantities related to \enquote{incomplete} datasets. 
We indicate with $\Phi_{-1} \in \R^{(N-1) \times p}$ the feature matrix of the training set \emph{without} the first sample $z_1$. For simplicity, we focus on the removal of the first sample, and similar considerations hold for the removal of any other sample. In other words, $\Phi_{-1}$ is equivalent to $\Phi$, without the first row. Similarly, using \eqref{eq:thetastar}, we indicate with $\theta_{-1}^* :=  \theta_0 + \Phi_{-1}^+ \left( G_{-1} - f(Z_{-1}, \theta_0) \right)$ the set of parameters the algorithm would have converged to if trained over $(Z_{-1}, G_{-1})$, the original dataset without the first pair sample-label $(z_1, g_1)$. 
We can now proceed with the definition of our notion of \enquote{stability}. 
\begin{definition}\label{def:stability}
Let $\theta^*$ ($\theta^*_{-1}$) be the parameters of the model $f$ given by \eqref{eq:glm} trained on the dataset $Z$ ($Z_{-1}$), as in \eqref{eq:thetastar}. We define the \emph{stability} $\mathcal S_{z_1} : \R^d \to \R$ with respect to the training sample $z_1$ as
\begin{equation}\label{eq:stabdef}
    \mathcal S_{z_1} := f(\cdot, \theta^*) - f(\cdot, \theta_{-1}^*).
\end{equation}
\end{definition}
This quantity indicates how the trained model changes if we add $z_1$ to the dataset $Z_{-1}$. If the training algorithm completely fits the data (as in \eqref{eq:thetastar}), then $\mathcal S_{z_1}(z_1) = g_1 - f(z_1, \theta_{-1}^*)$, which implies that
\begin{equation}\label{eq:genstab}
    \E_{z_1 \sim \mathcal P_Z} \left[\mathcal S^2_{z_1}(z_1) \right] = \E_{z_1 \sim \mathcal P_Z} \left[ \left( f(z_1, \theta_{-1}^*) - g_1 \right)^2 \right] = \E_{z \sim \mathcal P_Z} \left[ \left( f(z, \theta_{-1}^*) - g \right)^2 \right] =: \mathcal R_{Z_{-1}},
\end{equation}
where the purpose of the second step is just to match the notation used in \eqref{eq:generr}, and $\mathcal R_{Z_{-1}}$ denotes the generalization error of the algorithm that uses $Z_{-1}$ as training set.

\paragraph{Memorization of spurious features.}
The input samples are decomposed in two \emph{independent} components, \emph{i.e.}, $z \equiv [x, y]$. With this notation, we mean that $z \in \R^d$ is the concatenation of $x \in \R^{d_x}$ and $y \in \R^{d_y}$ ($d_x + d_y = d$). Here, $x$ is the \emph{core feature} that is useful to accomplish the task (\emph{e.g.}, the cat in top-left image of Figure \ref{fig:cat}), while $y$ is the \emph{spurious feature} containing noise (\emph{e.g.}, the background). 
Formally, we assume that, for $i\in \{1, \ldots, N\}$, $g_i = g(x_i)$, where $g$ is a labelling function, \emph{i.e.}, the label depends only on the core feature $x_i$ and it is independent of the spurious feature $y_i$ \simone{(a similar setting was previously considered in \cite{loureiro21})}. 
Even if $y_i$ is not useful for learning, the training algorithm may overfit it, memorizing its co-occurrence with the label $g_i$.
This phenomenon could then lead to unpredictable behaviours at test time, such as poor performance on samples $z_t = [x_t, y_i]$ with different information but the same spurious feature \cite{yang2022understanding}, or data privacy breaches, through extraction of information about $x_i$ via the knowledge of $y_i$ \cite{Leino2020, bombari2022differential}.
Specifically, in computer vision, an unusual background could expose information about the object in the foreground \cite{Leino2020}. In natural language processing, sensitive information ($x_i$) about an individual ($y_i$) could be extracted with proper prompting strategies: a language model might in fact be able to successfully auto-complete \textit{``The address of $y_i$ is...''} with \textit{``...$x_i$''}, as shown by \cite{bombari2022differential} on question-answering tasks. \simone{With slight abuse of notation, during the discussion, we will use the term \emph{feature} to indicate both elements in feature space $\varphi(z) \in \R^p$ (as in \eqref{eq:glm}), and portions of the samples $x$ and $y$. This choice is common in the related literature, and we will elaborate whenever it could generate confusion. }

To quantify the memorization of the spurious feature $y_i$, \simone{we will consider
\begin{equation}
    \textup{Cov}\left(f(z^s_i, \theta^*), g_i\right),
\end{equation}
which represents the covariance between the true label $g_i=g(x_i)$ and the output of the trained model \eqref{eq:thetastar} evaluated on 
$z^s_i = [x, y_i]$, which replaces 
the core feature $x_i$ with an independent feature $x$.}
As $x$ and $x_i$ are independent, this correlation is due to the memorization of $y_i$. In the experiments, similarly to \cite{singla2022salient, yang2022understanding}, we report the \emph{spurious accuracy}, \emph{i.e.}, 
the fraction of queries $f(z^s_i, \theta^*)$ that returns the correct label $g(x_i)$. For MNIST and CIFAR-10, instead of sampling an independent $x$, we simply set it to $0$, see Figure \ref{fig:cat}. \simone{Our code is publicly available at the GitHub repository \href{https://github.com/simone-bombari/spurious-features-memorization}{\texttt{simone-bombari/spurious-features-memorization}}.}

\section{Memorization and feature alignment}\label{sec:4}


We start by relating the output $f(z^s_i(x), \theta^*)$ evaluated on the spurious sample $z^s_i = [x, y_i]$ with the output $f(z_i, \theta^*)$
evaluated on the original sample $z_i$. For generalized linear regression, 
this can be 
elegantly done via the notion of \emph{stability} of Definition \ref{def:stability}.  By symmetry, from now on, we set $i=1$ without loss of generality. 
\begin{lemma}\label{lemma:proj}
    Let $\varphi : \R^d \to \R^p$ be a 
    feature map, such that the induced kernel $K \in \R^{N \times N}$ on the training set is invertible. Let $z_1 \in \R^d$ be an element of the training dataset $Z$, and $z \in \R^d$ a generic test sample. Let $\Ppm$ be the projector over $\Span \{\Rows (\Phi_{-1}) \}$ and $\mathcal S_{z_1}$ the stability with respect to $z_1$, as in Definition \ref{def:stability}. Define
    \begin{equation}\label{eq:featalign}
    \mathcal F_\varphi(z, z_1) := \frac{\varphi(z)^\top \Ppm^\perp \varphi(z_1)}{\norm{\Ppm^\perp \varphi(z_1)}_2^2}
    \end{equation}
    the \emph{feature alignment} between $z$ and $z_1$. Then, we have
    \begin{equation}\label{eq:lemmaproj}
    \mathcal S_{z_1}(z) = \mathcal F_\varphi(z, z_1) \, \mathcal S_{z_1}(z_1).
    \end{equation}
\end{lemma}
Classical work \cite{hatmatrix} considers a similar problem in the under-parametrized regime, exploiting the projector $H = \Phi (\Phi^\top \Phi)^{-1} \Phi^\top \in \R^{N \times N}$ ($p \leq N$ is needed for $\Phi^\top \Phi$ to be invertible), known as the \emph{hat matrix}. In contrast, Lemma \ref{lemma:proj} focuses on the over-parameterized regime and highlights the role of the projector $\Ppm$. The 
different behaviour of under-parameterized and over-parameterized models requires the proof of Lemma \ref{lemma:proj} to follow a different strategy, which is discussed in Appendix \ref{app:proj}.



In words, Lemma \ref{lemma:proj} relates the stability with respect to $z_1$ evaluated on the two samples $z$ and $z_1$ through the quantity $\mathcal F_\varphi(z, z_1)$, which captures the similarity between $z$ and $z_1$ in the feature space induced by $\varphi$. As a sanity check, 
the feature alignment between any sample and itself is equal to one, which trivializes \eqref{eq:lemmaproj}. Then, as $z$ and $z_1$ become less aligned, the stability $\mathcal S_{z_1}(z)=f(z, \theta^*)-f(z, \theta^*_{-1})$ starts to differ from $\mathcal S_{z_1}(z_1)=f(z_1, \theta^*)-f(z_1, \theta^*_{-1})$. 
We note that the feature alignment also depends on the rest of the training set $Z_{-1}$, as $Z_{-1}$ implicitly appears in the projector $\Ppm^\perp$. We also remark that the invertibility of $K$ directly implies that the denominator in \eqref{eq:featalign} is different from zero, see Lemma \ref{lemma:dendiv0} in Appendix \ref{app:proj}.

%


Armed with Lemma \ref{lemma:proj}, we now characterize the correlation between $f(z^s_1, \theta^*)$ and $g_1$. Let us replace $\mathcal F_\varphi(z^s_1, z_1)$ in \eqref{eq:lemmaproj} with a constant $\gamma_\varphi > 0$, independent from $z_1$. 
This is justified by Sections \ref{sec:rf} and \ref{sec:ntk}, where we prove the concentration of $\mathcal F_\varphi(z^s_1, z_1)$ for the RF and NTK model, respectively. Then, by using the definition of stability in \eqref{eq:stabdef}, we get
\begin{equation}\label{eq:pg}
    f(z_1^s, \theta^*) = f(z_1^s, \theta^*_{-1}) + \gamma_\varphi \, \mathcal S_{z_1}(z_1) = f(z_1^s, \theta^*_{-1}) + \gamma_\varphi \left( g_1 -  f(z_1, \theta^*_{-1})\right).
\vspace{0.3em}
\end{equation}
Note that $f(z_1^s, \theta^*_{-1})$ is independent from $g_1$, as it doesn't depend on $x_1$. In fact, $z^s_i = [x, y_i]$ is independent of $x_1$, and $\theta^*_{-1}$ is not trained on $x_1$. Thus, if the algorithm is stable, in the sense that $\mathcal S_{z_1}(z_1)$ is close to 0, we have that $f(z_1^s, \theta^*)$ has little dependence on $g_1$. Conversely, if $\mathcal S_{z_1}(z_1)$ grows, then $f(z_1^s, \theta^*)$ will start picking up the correlation with $g_1$. Concretely, we can look at the covariance between $f(z_1^s, \theta^*)$ and $g_1$, in the probability space of $z_1$:
\vspace{0.1em}
\begin{equation}\label{eq:pg2}
\begin{aligned}
    &\textup{Cov}\left(f(z^s_1, \theta^*), g_1\right) = \gamma_\varphi   \textup{Cov}\left( \mathcal S_{z_1}(z_1), g_1 \right)  \\
    &\leq\gamma_\varphi \sqrt{ \textup{Var}\left( \mathcal S_{z_1}(z_1) \right) \textup{Var}\left(g_1 \right)} \leq \gamma_\varphi  \sqrt{ \mathcal R_{Z_{-1}}} \sqrt{\textup{Var}\left(g_1 \right)}.
\end{aligned}
\end{equation}
Here, the first step uses \eqref{eq:pg} and the independence between $f(z_1^s, \theta^*_{-1})$ and $g_1$, the second step is an application of Cauchy-Schwarz, and the last step follows from \eqref{eq:genstab}.

The terms 
$\gamma_\varphi$ and $\sqrt{\mathcal R_{Z_{-1}}}$ in the RHS of \eqref{eq:pg2} show that the memorization of spurious features is affected by two factors: \emph{(i)} the similarity between $z^s_1$ and $z_1$ (formalized by $\mathcal F_\varphi(z^s_1, z_1)$), and \emph{(ii)} the generalization error of the model. While the dependence on the generalization error is not entirely surprising (overfitting causes both the inability of the model to generalize and the memorization of spurious correlations between $g_1$ and $y_1$), the key role of the feature alignment (in the form defined in \eqref{eq:featalign}) is a-priori far from obvious. Furthermore, via the analysis of $\mathcal F_\varphi(z^s_1, z_1)$ 
in the next two sections, we will show how the memorization is affected by the choice of the feature map and, specifically, of the activation function. 


Similarly, we can lower bound the \emph{spurious loss}, defined as
\begin{equation}
    \mathcal L := \E_{z_1} \left[ \left( f(z_1^s, \theta^*) - g_1 \right)^2 \right].
\end{equation}
By introducing the shorthands $\bar f := \E_{z_1} [ f(z_1^s, \theta^*) ]$ and $\bar g := \E_{z_1} [ g_1 ]$, we have the lower bound:
\begin{equation}\label{eq:newiclr}
\begin{aligned}
    \mathcal L 
    &= \E_{z_1} \left[ \left( \left( f(z_1^s, \theta^*) - \bar f \right) - \left( g_1 - \bar g \right) + \left(\bar f - \bar g \right) \right)^2 \right] \\
    &= \E_{z_1} \left[ \left( f(z_1^s, \theta^*) - \bar f \right)^2 \right] + \E_{z_1} \left[ \left( g_1 - \bar g \right)^2 \right] +  \E_{z_1} \left[  \left(\bar f - \bar g \right)^2 \right] - 2 \, \textup{Cov}\left(f(z^s_1, \theta^*), g_1\right) \\
    &\geq \E_{z_1} \left[ \left( g_1 - \bar g \right)^2 \right] - 2 \, \gamma_\varphi \, \sqrt{ \mathcal R_{Z_{-1}}} \sqrt{\textup{Var}\left(g_1 \right)}. 
\end{aligned}
\end{equation}
The first term on the RHS of \eqref{eq:newiclr} is the minimal loss when no information about $x_1$ is available (and, thus, the best estimator is the expectation of $g_1$). 
The second term 
indicates how the memorization of the spurious feature $y_1$
improves the trivial guess $\bar g$, and it depends again on $\mathcal R_{Z_{-1}}$ and $\gamma_\varphi$. 




\section{Main result for random features}\label{sec:rf}

The \emph{random features (RF) model} takes the form
\begin{equation}\label{eq:featmaprf}
    f_{\textup{RF}}(z, \theta) = \varrf(z)^\top \theta, \qquad \varrf(z)=\phi(V z),
\end{equation}
where $V$ is a $k \times d$ matrix s.t.\ $V_{i,j} \distas{}_{\rm i.i.d.}\mathcal{N}(0, 1 / d)$, and $\phi$ is an activation 
applied component-wise. The number of parameters of this model is $k$, as $V$ is fixed and $\theta\in \mathbb R^k$ contains trainable parameters. We denote by $\mu_l$ the $l$-th Hermite coefficient of $\phi$ (see Appendix \ref{app:Hermite} for details).


\begin{assumption}[Data distribution]\label{ass:datadist}
        The input data $(z_1,\ldots,z_N)$ are $N$ i.i.d.\ samples from 
        $\mathcal P_Z = \mathcal P_X \times \mathcal P_Y$ s.t.\ 
        $z_i \in \R^d$ can be written as $z_i = [x_i, y_i]$, with $x_i\in \R^{d_x}$, $y_i\in\R^{d_y}$ and $d = d_x + d_y$. We assume that $x_i\sim \mathcal P_X$ is independent of $y_i\sim \mathcal P_Y$, and the following 
        holds:
	\begin{enumerate}
		\item $\norm{x}_2 = \sqrt{d_x}$ and $\norm{y}_2 = \sqrt{d_y}$. 
		\item $\E[x] = 0$ and $\E[y] = 0$. 
     	\item $\mathcal P_X$ and $\mathcal P_Y$ satisfy \emph{Lipschitz concentration}.
	\end{enumerate}
\end{assumption}

The first two assumptions are achieved by 
data pre-processing 
and 
could 
be relaxed as in Assumption 1 of \cite{bombari2022memorization} at the cost of a more involved argument. 
The third assumption (see Appendix \ref{app:notation} for details) corresponds to data 
having well-behaved tails, and it covers a number of important cases, \emph{e.g.}, standard Gaussian \cite{vershynin2018high}, uniform on the sphere/hypercube \cite{vershynin2018high}, or data obtained via GANs \cite{seddik2020random}. This requirement is common in the related literature \cite{bombari2022memorization, 
bubeck2021a,tightbounds} and it is often replaced by 
a stronger requirement (\emph{e.g.}, data uniform on the sphere), see \cite{montanari2022interpolation}.

\begin{figure*}[t]
  \begin{center}
    \includegraphics[width=\textwidth]{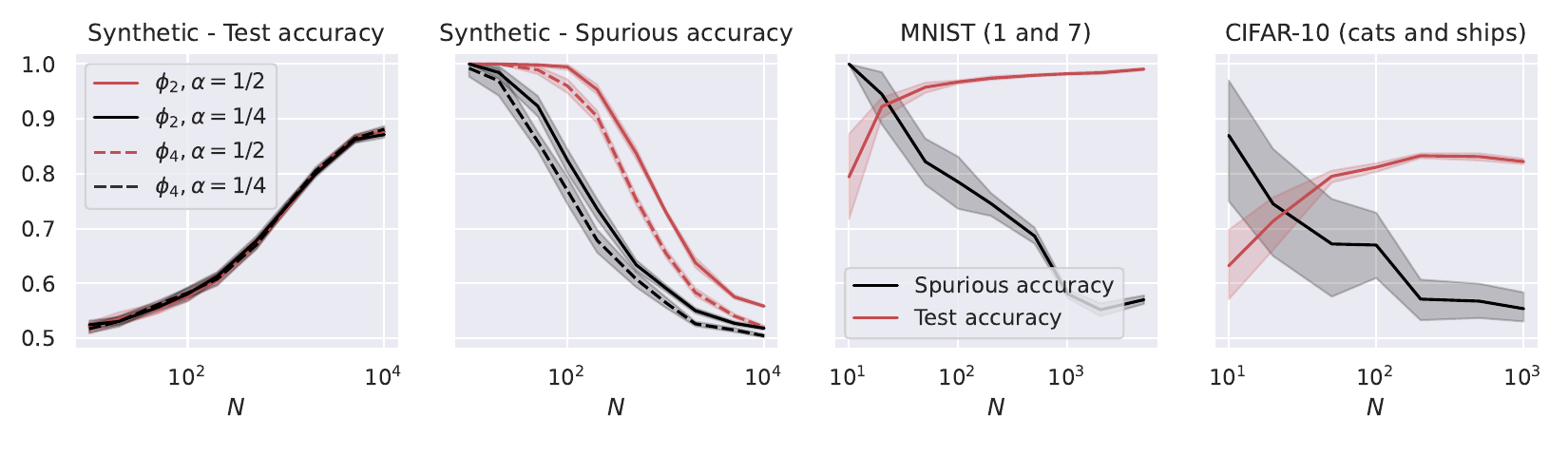}
  \end{center}
  \vspace{-1em}
  \caption{Test and spurious accuracies as a function of the number of training samples $N$, for various binary classification tasks. In the first two plots, we consider the RF model in \eqref{eq:featmaprf} with $k = 10^5$ trained over Gaussian data with $d = 1000$. The labeling function is $g(x) = \textup{sign}(u^\top x)$. We repeat the experiments for $\alpha = \{ 0.25, 0.5 \}$ and for the two activations $\phi_2 = h_1 + h_2$ and $\phi_4 = h_1 + h_4$, where $h_i$ denotes the $i$-th Hermite polynomial (see Appendix \ref{app:Hermite}). In the last two plots, we consider the same model with ReLU activation, trained over two MNIST and CIFAR-10 classes. The width of the noise background is $10$ pixels for MNIST and $8$ pixels for CIFAR-10, see Figure \ref{fig:cat}. The spurious accuracy is obtained by querying the model only with the noise background from the training set, replacing all the other pixels with $0$, and taking the sign of the output. As we consider binary classification, an accuracy of 0.5 is achieved by random guessing. We plot the average over 10 independent trials
  and the confidence band at 1 standard deviation.}
  \vspace{-1em}
  \label{fig:rf}
\end{figure*}

\begin{assumption}[Over-parameterization and high-dimensional data]\label{ass:overparam}
\begin{equation}\label{eq:overp}
    N \log^3 N = o(k), \qquad \sqrt d \log d = o(k), \qquad k \log^4 k = o( d^2).
\end{equation}
\end{assumption}
The first condition in \eqref{eq:overp} requires the number of neurons $k$ to scale faster than the number of data points $N$. This over-parameterization leads to a lower bound on the smallest eigenvalue of the kernel induced by the feature map, which in turn implies that the model interpolates the data, as required to write \eqref{eq:thetastar}. This over-parameterized regime also achieves minimum test error \cite{mei2022generalization}. Combining the second and third conditions in \eqref{eq:overp}, we have that $k$ can scale between $\sqrt d$ and $d^2$ (up to $\log$ factors). Finally, merging the first and third condition gives that $d^2$ scales faster than $N$. We notice that this holds for standard datasets (MNIST, CIFAR-10 and ImageNet).



\begin{assumption}[Activation function]\label{ass:activationfunc}
    The activation function $\phi$ is a non-linear $L$-Lipschitz function.
\end{assumption}


\begin{theorem}\label{thm:RF}
Let Assumptions \ref{ass:datadist}, \ref{ass:overparam}, 
and \ref{ass:activationfunc} hold, and let $x \sim \mathcal P_X$ be sampled independently from everything. Consider querying the trained RF model \eqref{eq:featmaprf} with $z_1^s = [x, y_1]$. Let $\alpha = d_y / d$ and $\mathcal F_\textup{RF}(z_1^s, z_1)$ be the feature alignment between $z_1^s$ and $z_1$, as defined in \eqref{eq:featalign}. Then,
\begin{equation}\label{eq:concRF}
    \left| \mathcal F_{\textup{RF}}(z_1^s, z_1) - \gamma_{\textup{RF}} \right| = o(1),
\end{equation}
with probability at least $1 - \exp(-c \log^2 N)$ over $V$, $Z$ and $x$, where 
$\gamma_{\textup{RF}} \leq 1$ does not depend on $z_1$ and $x$.
Furthermore, \simone{letting $\mu_l$ be the $l$-th Hermite coefficient of the activation $\phi$ (see Appendix \ref{app:Hermite}),} we have
\begin{equation}\label{eq:LBRF}
    \gamma_{\textup{RF}} > \frac{\sum_{l = 2}^{+\infty} \mu_l^2 \alpha^l}{\sum_{l = 1}^{+\infty} \mu_l^2} - o(1),
\end{equation}
with probability at least $1 - \exp(-c \log^2 N)$ over $V$ and $Z_{-1}$, 
\emph{i.e.}, $\gamma_{\textup{RF}}$ is bounded away from $0$ with high probability.
\end{theorem}

\simone{Theorem \ref{thm:RF} proves the concentration of the feature alignment $\mathcal F_{\textup{RF}}(z^s_1, z_1)$ to a constant $\gamma_{\textup{RF}}$ between
$\sum_{l = 2}^{+\infty} \mu_l^2 \alpha^l/\sum_{l = 1}^{+\infty} \mu_l^2>0$ and $1$. The lower bound increases with $\alpha$ (as expected, since $\alpha$ is the fraction of the input given by the spurious feature), and it depends in a non-trivial way on the activation 
via its Hermite coefficients.} \simone{This result validates the argument in \eqref{eq:pg}, where we replaced the feature alignment $\mathcal F_{\textup{RF}}(z^s_1, z_1)$ with a constant $\gamma_{\textup{RF}} > 0$. Thus, \eqref{eq:pg2} now reads
\begin{equation}\label{eq:pg2RF}
    \textup{Cov}\left(f_{\textup{RF}}(z^s_1, \theta^*), g_1\right)  \leq \gamma_{\textup{RF}}  \sqrt{ \mathcal R_{Z_{-1}}} \sqrt{\textup{Var}\left(g_1 \right)},
\end{equation}
which quantifies the memorization of spurious features in terms of the generalization error and the constant $\gamma_{\textup{RF}}$.}

These effects are clearly displayed in Figure \ref{fig:rf} for binary classification 
on synthetic (first two plots) and image (last two plots) datasets. Specifically, as the number of samples $N$ increases, the test accuracy increases and the spurious accuracy (obtained by querying the trained model with the spurious sample $f(z_i^s, \theta^*)$) decreases.  Furthermore, for the synthetic dataset, while the test accuracy does not depend on $\alpha$ or the activation function, the spurious accuracy increases with $\alpha$ and by taking an activation function with dominant low-order Hermite coefficients, as predicted by \eqref{eq:LBRF}.
For an additional experiment highlighting the dependence on $\alpha$, we refer the reader to Figure \ref{fig:varying_alpha} in Appendix \ref{app:exp}.


\paragraph{Proof sketch.} 
We set
 \begin{equation}
        \gamma_{\textup{RF}} := \frac{\E_{z_1, z_1^s} [ \varphi_{\textup{RF}}(z_1^s)^\top \Ppm^\perp \varphi_{\textup{RF}}(z_1)]}{\E_{z_1} [ \|\Ppm^\perp \varphi_{\textup{RF}}(z_1)\|_2^2 ]}.
    \end{equation}
Here, $\Ppm$ is the projector over $\Span \{\Rows (\Phi_{\textup{RF}, -1}) \}$ and $\Phi_{\textup{RF}, -1}$ the RF feature matrix after removing the first row.
With this choice, the numerator and denominator of $\gamma_{\textup{RF}}$ equal the expectations of the corresponding quantities appearing in $\mathcal F_{\textup{RF}}(z_1^s, z_1)$. Thus, the concentration result in \eqref{eq:concRF} is obtained from the general form of the Hanson-Wright inequality in \cite{HWconvex}, see Lemma \ref{lemma:newlemma}. 
The upper bound $\gamma_{\textup{RF}} \leq 1$ follows from an application of Cauchy-Schwarz inequality. In contrast, the lower bound is more involved and it is obtained via the three steps below.

\begin{figure*}[!t]
  \begin{center}
    \includegraphics[width=\textwidth]{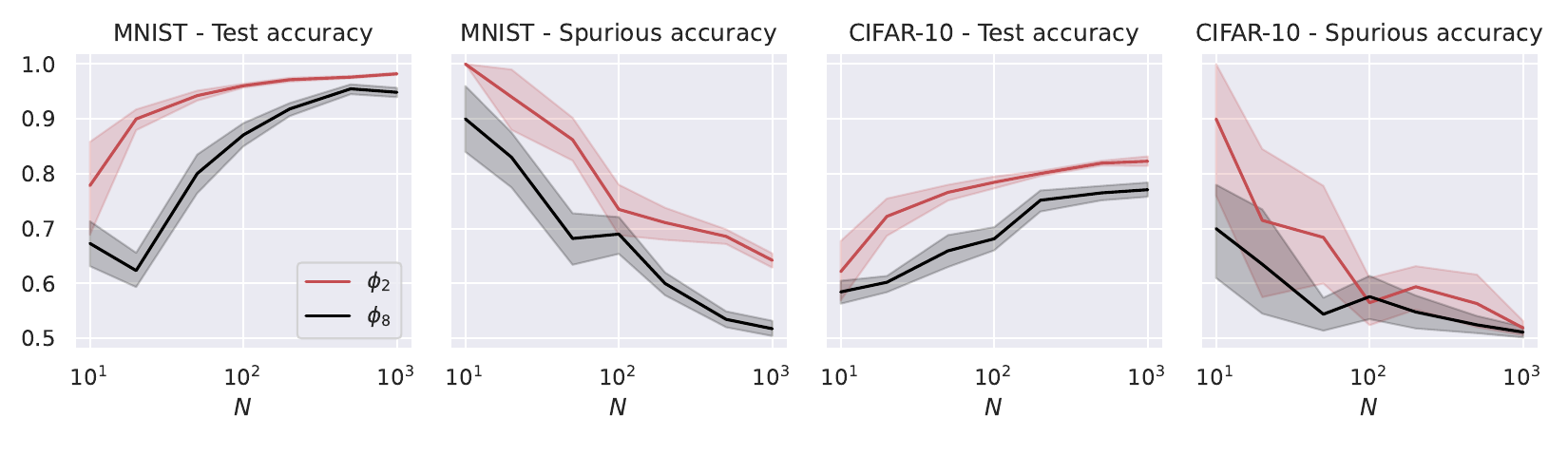}
  \end{center}
  \caption{We consider the NTK model in \eqref{eq:NTKmodel} with $k = 100$, trained on MNIST (digits 1 and 7, first and second plots), and CIFAR-10 (cats and ships, third and fourth plots). We repeat the experiments for activations whose derivatives are $\phi'_2 = h_0 + h_1$ and $\phi'_8 = h_0 + h_7$, where $h_i$ denotes the $i$-th Hermite polynomial (see Appendix \ref{app:Hermite}). The rest of the setup is the same as that of Figure \ref{fig:rf}.} 
  \label{fig:ntk_real}
\end{figure*}

\emph{\underline{Step 1: Centering the feature map $\varrf$.}} We extract the term $\E_V \left[ \phi(V z) \right]$ from the expression of $\mathcal F_{\textup{RF}}$ and 
show 
it can be neglected, due to the specific structure of $\Ppm^\perp$. Specifically, letting $\tvarrf(z) := \varrf(z)  - \E_{V} \left[ \varrf(z) \right]$, we have
\begin{equation}\label{eq:step1}
    \mathcal F_{\textup{RF}}(z_1^s, z_1) \simeq \frac{\tvarrf(z_1^s)^\top \Ppm^\perp \tvarrf(z_1)}{\norm{\Ppm^\perp \tvarrf(z_1)}_2^2} = \frac{\tvarrf(z_1^s)^\top \tvarrf(z_1) - \tvarrf(z_1^s)^\top \Ppm \tvarrf(z_1)}{\norm{\Ppm^\perp\tvarrf(z_1)}_2^2},
\end{equation}
where $\simeq$ denotes an equality up to a $o(1)$ term. This is formalized in Lemma \ref{lemma:PpRFon1}.

\emph{\underline{Step 2: Linearization of the centered feature map $\tvarrf$.}} We consider the terms $\tvarrf(z_1^s), \tvarrf(z_1)$ that multiply $\Ppm$ in the RHS of \eqref{eq:step1}, and we show that they are well approximated by their first-order Hermite expansions ($\mu_1 V z_1^s$ and $\mu_1 V z_1$, respectively). In fact, the  rest of the Hermite series scales at most as $N / d^2$, which is negligible due to Assumption \ref{ass:overparam}. Specifically, Lemma \ref{lemma:Pseeslin} implies
\begin{equation}\label{eq:step2}
     \frac{\tvarrf(z_1^s)^\top \tvarrf(z_1) - \tvarrf(z_1^s)^\top \Ppm \tvarrf(z_1)}{\norm{\Ppm^\perp\tvarrf(z_1)}_2^2} \simeq\frac{\tvarrf(z_1^s)^\top \tvarrf(z_1) - \mu_1 ^2 (V z_1^s)^\top \Ppm (V z_1)}{ \norm{\Ppm^\perp\tvarrf(z_1)}_2^2}.
\end{equation}

\emph{\underline{Step 3: Lower bound in terms of $\alpha$ and $\{\mu_l\}_{l\ge 2}$.}} To conclude, we express the RHS of \eqref{eq:step2} as follows:
\begin{equation}\label{eq:sketchrf}
\begin{aligned}
      & \frac{\tvarrf(z_1^s)^\top \tvarrf(z_1)- \mu_1 ^2 (V z_1^s)^\top (V z_1) + \mu_1 ^2 (V z_1^s)^\top \Ppm^\perp (V z_1)}{\norm{\tvarrf(z_1)}_2^2 - \norm{\Ppm \tvarrf(z_1)}_2^2} \\
     &\gtrsim \frac{\tvarrf(z_1^s)^\top\tvarrf(z_1) -\mu_1 ^2 (V z_1^s)^\top (V z_1)}{\norm{\tvarrf(z_1)}_2^2} \simeq \frac{\sum_{l = 2}^{+\infty} \mu_l^2 \alpha^l}{\sum_{l = 1}^{+\infty} \mu_l^2}>0,
\end{aligned}
\end{equation}
where $\gtrsim$ denotes an inequality up to a $o(1)$ term. As $\Ppm=I-\Ppm^\perp$,
the term in the first line equals the RHS of \eqref{eq:step2}. 
Next, we show that $\mu_1 ^2 (V z_1^s)^\top \Ppm^\perp (V z_1)$ is equal to $\mu_1 ^2 (V [y_1, 0])^\top \Ppm^\perp (V [y_1, 0])$ (which corresponds to the common noise part in $z_1, z_1^s$) plus a vanishing term, see Lemma \ref{lemma:RFlast}. As $\mu_1 ^2 (V [y_1, 0])^\top \Ppm^\perp (V [y_1, 0])\ge 0$, the inequality in the second line follows. The last step is obtained by showing concentration over $V$ of numerator and denominator. The expression on the RHS of \eqref{eq:sketchrf} is strictly positive as $\alpha>0$ and $\phi$ is non-linear by Assumption \ref{ass:activationfunc}.

\section{Main result for NTK features}\label{sec:ntk}

We consider the following two-layer neural network
\begin{equation}\label{eq:linNN}
    f_{\textup{NN}}(z, w) = \sum_{i = 1}^k \phi\left(W_{i:} z\right).
\end{equation}
The hidden layer contains $k$ neurons; $\phi$ is an activation function applied component-wise; 
$W \in \mathbb R^{k \times d}$ denotes the weights of the hidden layer; $W_{i:}$ denotes the $i$-th row of $W$; and we set the $k$ weights of the second layer to $1$. We indicate with $w$ the vector containing the parameters of this model, \emph{i.e.}, $w = [\text{vec}(W)] \in \R^p$, with $p = kd$. We initialize the network with standard (\emph{e.g.}, He’s or LeCun’s) initialization, \emph{i.e.}, $[W_0]_{i,j} \distas{}_{\rm i.i.d.} \mathcal N(0, 1 / d)$.

The \emph{NTK regression model} takes the form
\begin{equation}\label{eq:NTKmodel}
  f_{\textup{NTK}}(z, \theta) = \varntk(z)^\top \theta,  \qquad \varntk(z) = \nabla_w f_{\textup{NN}}(z, w) |_{w = w_0}=z \otimes \phi'(W_0 z),
\end{equation}
where $\nabla_w f_{\textup{NN}}$ in the second line is computed via the chain rule. The vector of trainable parameters
is $\theta \in \R^p$, with $p = kd$, which is initialized with $\theta_0 = w_0= [\text{vec}(W_0)]$. We note that
$f_{\textup{NTK}}(z, \theta)$ is the linearization of $f_{\textup{NN}}(z, w)$ around the initial point $w_0$ \cite{bartlett2021deep,JacotEtc2018}, and the model in \eqref{eq:NTKmodel} corresponds to training the two-layer network \eqref{eq:linNN} in the lazy regime \cite{chizat2019lazy,oymak2019overparameterized}. As such, it has received significant attention in the theoretical literature, see \emph{e.g.} \cite{bombari2023universal,dohmatob2022non, montanari2022interpolation}. 

\begin{assumption}[Over-parameterization and topology]\label{ass:overparamntk}
\begin{equation}\label{eq:overparamntk}
        N \log^8 N = o(kd), \qquad N > d, \qquad k = \bigO{d}.
    \end{equation}
\end{assumption}
The first condition is the smallest (up to $\log$ factors) over-parameterization that guarantees interpolation \cite{bombari2022memorization}. The second condition is rather mild (it is easily satisfied by standard datasets) and purely technical. 
The third condition is required to lower bound the smallest eigenvalue of the kernel induced by the feature map \eqref{eq:NTKmodel}, and a stronger requirement, \emph{i.e.}, the strict inequality $k<d$, has appeared in prior work \cite{QuynhICML2017, QuynhICML2018,QuynhMarco2020}. 

\vspace{0.3em}

\begin{figure*}[!t]
  \begin{center}
    \includegraphics[width=\textwidth]{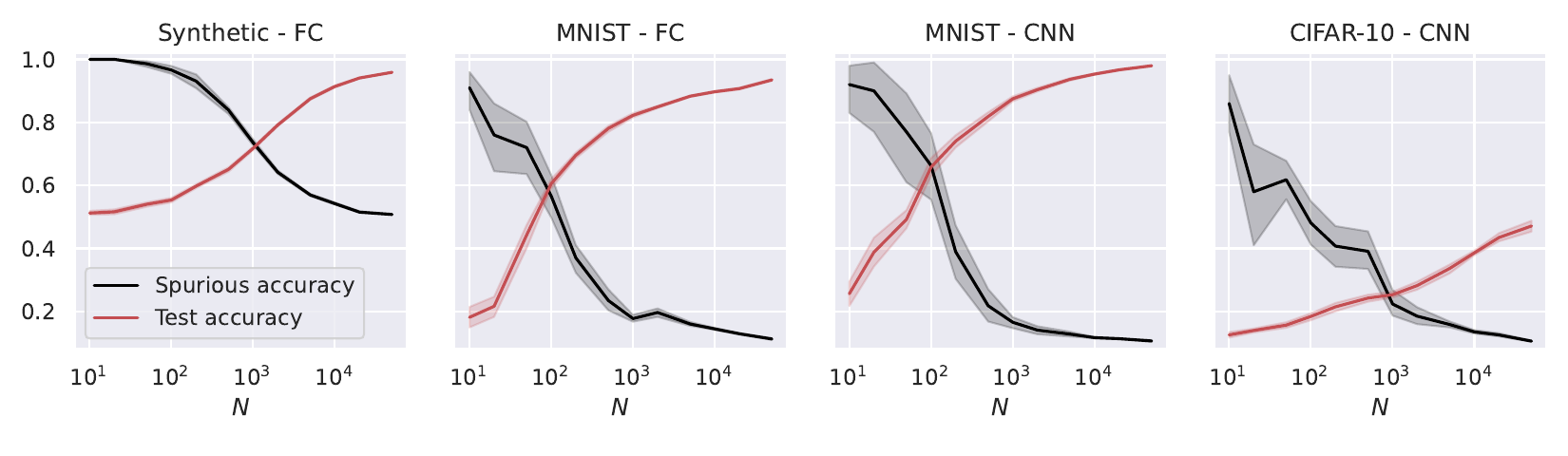}
  \end{center}
  \caption{Test and spurious accuracies as a function of the number of training samples $N$, for a fully connected (FC, first two plots), \simone{and a small 
  convolutional neural network (CNN, last two plots)}. In the first plot, we use synthetic (Gaussian) data with $d = 1000$, and the labeling function is $g(x) = \textup{sign}(u^\top x)$. 
  As we consider binary classification, the accuracy of random guessing is $0.5$. The other plots use subsets of the MNIST and CIFAR-10 datasets, with an external layer of noise added to images, see Figure \ref{fig:cat}. 
  As we consider $10$ classes, the accuracy of random guessing is $0.1$. We plot the average over 10 independent trials and the confidence band at 1 standard deviation.}
  \label{fig:nn}
\end{figure*}

\begin{assumption}[Activation function]\label{ass:activationfuncntk}
    The activation function $\phi$ is 
    non-linear and its derivative $\phi'$ is 
    $L$-Lipschitz. 
\end{assumption}

\vspace{0.3em}

We denote by $\mu'_l$ the $l$-th Hermite coefficient of $\phi'$. We remark that the invertibility of the kernel $\kntk$ induced by the feature map \eqref{eq:NTKmodel} 
follows from Lemma \ref{lemma:evminntk}. At this point, we are ready to state our main result for the NTK model, whose full proof is contained in Appendix \ref{app:NTK}.

\vspace{0.3em}

\begin{theorem}\label{thm:mainntk}
Let Assumptions \ref{ass:datadist}, \ref{ass:overparamntk}, and \ref{ass:activationfuncntk} 
hold, and let $x \sim \mathcal P_X$ be sampled independently from everything. Consider querying the trained NTK model \eqref{eq:NTKmodel} with $z_1^s = [x, y_1]$. Let $\alpha = d_y / d \in (0, 1)$ and $\mathcal F_\textup{NTK}(z_1^s, z_1)$ be the feature alignment between $z_1^s$ and $z_1$, as defined in \eqref{eq:featalign}. Then, \simone{letting $\mu'_l$ be the $l$-th Hermite coefficient of the derivative of the activation $\phi'$ (see Appendix \ref{app:Hermite}),} we have \vspace{0.2em}
\begin{equation}
\left| \mathcal F_{\textup{NTK}}(z_1^s, z_1) - \gamma_{\textup{NTK}} \right| = o(1), \qquad \mbox{where }\,\,\, 0 < \gamma_{\textup{NTK}} := \alpha \, \frac{ \sum_{l = 1}^{+\infty} {\mu'_l}^2 \alpha^l }{\sum_{l = 1}^{+\infty} {\mu'_l}^2} < 1, 
\end{equation}
with probability at least $1 - N \exp(-c \log^2 k) - \exp(-c \log^2 N)$ over $Z$, $x$ and $W_0$. 
\end{theorem}

\simone{Theorem \ref{thm:RF} proves the concentration of the feature alignment $\mathcal F_{\textup{NTK}}(z^s_1, z_1)$ to a constant $\gamma_{\textup{NTK}}$, which has an exact expression depending on
$\alpha$ and the Hermite coefficients of the derivative of the activation.}
\simone{This result validates the argument in \eqref{eq:pg}, where we replaced the feature alignment $\mathcal F_{\textup{NTK}}(z^s_1, z_1)$ with a constant $\gamma_{\textup{NTK}} > 0$. Thus, \eqref{eq:pg2} now reads
\begin{equation}\label{eq:pg2NTK}
    \textup{Cov}\left(f_{\textup{NTK}}(z^s_1, \theta^*), g_1\right)  \leq \gamma_{\textup{NTK}}  \sqrt{ \mathcal R_{Z_{-1}}} \sqrt{\textup{Var}\left(g_1 \right)},
\end{equation}
which quantifies the memorization of spurious features in terms of the generalization error and the constant $\gamma_{\textup{NTK}}$.}

Figure \ref{fig:ntk_real} considers training on MNIST and CIFAR-10, and it shows that the predictions of Theorem \ref{thm:mainntk} also hold for standard datasets: as $N$ increases, the test accuracy improves and the spurious accuracy decreases; considering activations with dominant high-order Hermite coefficients reduces memorization. For additional experiments using the same setup as Figure \ref{fig:rf} and highlighting the dependence on $\alpha$, we refer the reader to Figures \ref{fig:ntk} and \ref{fig:varying_alpha} in Appendix \ref{app:exp}.


\paragraph{Proof sketch.} The argument is more direct than for the RF model since, in this case, we are able to express $\gamma_{\textup{NTK}}$ in closed form. We denote by $\Ppm$ the projector over the span of the rows of the NTK feature matrix $\Phi_{\textup{NTK}, -1}$ without the first row. Then, the \emph{first step} is to \emph{center the feature map} $\varntk$, which gives 
\begin{equation}\label{eq:sketchntk}
\begin{aligned}
     \frac{\varntk(z_1^s)^\top \Ppm^\perp \varntk(z_1)}{\norm{\Ppm^\perp \varntk(z_1)}_2^2} \simeq \frac{\tvarntk(z_1^s)^\top \Ppm^\perp \tvarntk(z_1)}{\norm{\Ppm^\perp \tvarntk(z_1)}_2^2},   
\end{aligned}
\end{equation}
where $\tvarntk(z) := 
z \otimes \left( \phi'(W_0 z) - \E_{W_0} \left[ \phi'(W_0 z) \right]\right)$. While a similar step appeared in the analysis of the RF model, its implementation for NTK requires a different strategy. In particular, we exploit that the samples $z_1$ and $z_1^s$ are \emph{approximately} contained in the span of the rows of $Z_{-1}$ (see Lemma \ref{lemma:Ppntkon1pert}). As the rows of $Z_{-1}$ may not \emph{exactly} span all $\R^d$, we resort to an \emph{approximation} by adding a small amount of independent noise to every entry of $Z_{-1}$. The resulting perturbed dataset $\bar Z_{-1}$ satisfies $\Span \{ \Rows( \bar Z_{-1}) \} = \R^d$ (see Lemma \ref{lemma:invertible}), and we conclude via a continuity argument with respect to the magnitude of the perturbation (see Lemmas \ref{lemma:continuous} and \ref{lemma:Ppntkon1}).


The \emph{second step} is to upper bound the terms $\left| \tvarntk(z_1^s)^\top \Ppm \tvarntk(z_1) \right|$ and $\norm{\Ppm \tvarntk(z_1)}_2^2$, showing they have \emph{negligible magnitude}, which gives 
\begin{equation}\label{eq:sketchntk2}
\begin{aligned}
       \frac{\tvarntk(z_1^s)^\top \Ppm^\perp \tvarntk(z_1)}{\norm{\Ppm^\perp \tvarntk(z_1)}_2^2}
      \simeq \frac{\tvarntk(z_1^s)^\top  \tvarntk(z_1)}{\norm{\tvarntk(z_1)}_2^2}.  
\end{aligned}
\end{equation}
This is a consequence of the fact that, if $z \sim \mathcal P_Z$ is independent from $Z_{-1}$, then $\tilde \varphi(z)$ is roughly orthogonal to $\Span \{\Rows (\Phi_{\textup{NTK}, -1}) \}$, see Lemma \ref{lemma:ntkpenultimatestep}.

Finally, the \emph{third step} is to show the \emph{concentration} over $W_0$ of the numerator and denominator of the RHS of \eqref{eq:sketchntk2}, see Lemma \ref{lemma:ntklaststep}. This allows us to conclude that
\begin{equation}\label{eq:sketchntk3}
\begin{aligned}
      \frac{\tvarntk(z_1^s)^\top  \tvarntk(z_1)}{\norm{\tvarntk(z_1)}_2^2}   \simeq \alpha \, \frac{ \sum_{l = 1}^{+\infty} {\mu'_l}^2 \alpha^i }{\sum_{l = 1}^{+\infty} {\mu'_l}^2}>0.
\end{aligned}
\end{equation}
The RHS of \eqref{eq:sketchntk3} is strictly positive as $\alpha>0$ and $\phi$ is non-linear by Assumption \ref{ass:activationfuncntk}.

\begin{wrapfigure}{r}{0.6\textwidth}
  \vspace{-0.5em}
  \begin{center}
    \includegraphics[width=0.6\textwidth]{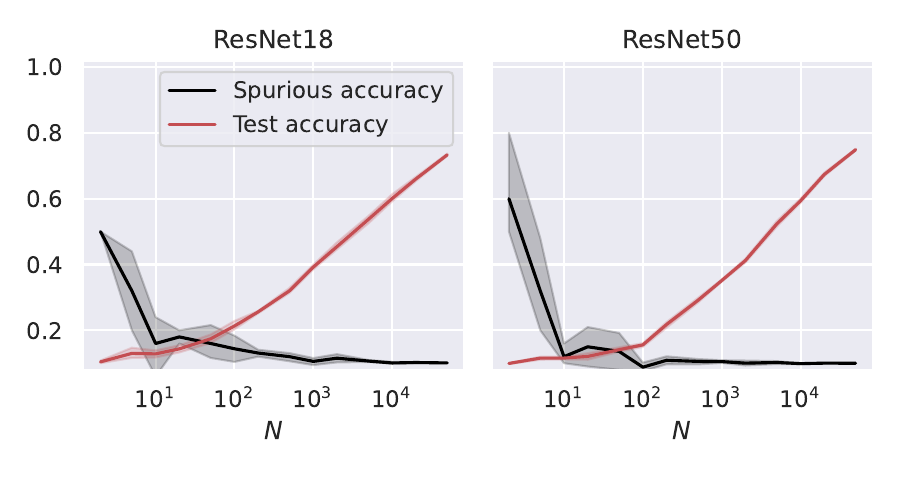}
  \end{center}
  \vspace{-1em}
  \caption{\simone{Test and spurious accuracies as a function of the number of training samples $N$, for two ResNet architectures. We use subsets of the CIFAR-10 dataset, with an external layer of noise added to images, see Figure \ref{fig:cat}. 
  As we consider $10$ classes, the accuracy of random guessing is $0.1$. We plot the average over 10 independent trials and the confidence band at 1 standard deviation.}}
  \label{fig:nn_rebuttal}
\end{wrapfigure}

\paragraph{Discussion. } 
As showed by \eqref{eq:pg2} and \eqref{eq:newiclr}, if 
$\mathcal F_\varphi(z_1, z_1^s)$ converges to 0, spurious correlations are \emph{not} memorized by the model. 
However, Theorems \ref{thm:RF} and \ref{thm:mainntk} prove that, for the RF and NTK model respectively, a strictly positive geometric overlap ($\alpha > 0$) guarantees a strictly positive feature alignment. Thus, these results imply that, as long as the generalization error is not vanishing, spurious features are memorized and, in fact, we quantify the extent to which memorization occurs. 
Remarkably, our experiments on real datasets (see Figure \ref{fig:ntk_real}) show that the effect of the activation function is in line with our predictions: taking $\phi'$ with higher order Hermite coefficients lead to models that are less prone to memorize spurious features. Finally, while we focus on generalized linear \simone{regression}, the interplay between memorization of spurious features and generalization provided by our analysis holds in far more generality and, in particular, it is displayed \simone{by neural networks also capable of feature learning}, see Figures \ref{fig:nn} and \ref{fig:nn_rebuttal}.

\section{Conclusions}

In this work, we present a theoretical framework to quantify the memorization of spurious features. 
Our characterization hinges on \emph{(i)} the classical notion of \emph{stability} of the model w.r.t.\ a training sample, and \emph{(ii)} a novel notion of \emph{feature alignment} $\mathcal F(z^s, z)$ between two samples that share the same spurious feature $y$. By providing a precise analysis of the feature alignment in the prototypical settings of random and NTK features regression, we show that the memorization is proportional to the generalization error, and we characterize the proportionality constant, revealing how it depends on the model and its activation function. Our theoretical predictions are confirmed by numerical experiments on standard datasets (see Figure \ref{fig:ntk_real}) and on different neural network architectures (see Figures \ref{fig:nn} and \ref{fig:nn_rebuttal}).

The approach we put forward is rather general, and  
our results could be extended to cases where the spurious feature $y$ is correlated with the ground-truth label $g$. Another possible extension involves testing the trained model on a new spurious feature $y'$, which is not present in the training set but is correlated with a feature $y$ that has already been seen; \simone{or capturing the role of \emph{simplicity bias} in this phenomenon.} We also remark that 
the formalism introduced by Lemma \ref{lemma:proj} applies to any feature map $\varphi$ (\emph{e.g.}, with multiple fully-connected, convolutional or attention layers). Characterizing the feature alignment of such maps would allow to compare different models and establish which of those is less prone to memorizing spurious correlations. 


\section*{Acknowledgements}

The authors were partially supported by the 2019 Lopez-Loreta prize, and they would like to thank (in alphabetical order) Grigorios Chrysos, Simone Maria Giancola, Mahyar Jafari Nodeh, Christoph Lampert, Marco Miani, GuanWen Qiu, and Peter S\'{u}ken\'{i}k for helpful discussions.

{
\small

\bibliographystyle{plain}
\bibliography{bibliography.bib}

\begin{thebibliography}{10}

\bibitem{HWconvex}
Radoslaw Adamczak.
\newblock {A note on the Hanson-Wright inequality for random vectors with
  dependencies}.
\newblock {\em Electronic Communications in Probability}, 20:1--13, 2015.

\bibitem{adlam2020neural}
Ben Adlam and Jeffrey Pennington.
\newblock The neural tangent kernel in high dimensions: Triple descent and a
  multi-scale theory of generalization.
\newblock In {\em International Conference on Machine Learning (ICML)}, 2020.

\bibitem{devansh17mem}
Devansh Arpit, Stanis\l{}aw Jastrzundefinedbski, Nicolas Ballas, David Krueger,
  Emmanuel Bengio, Maxinder~S. Kanwal, Tegan Maharaj, Asja Fischer, Aaron
  Courville, Yoshua Bengio, and Simon Lacoste-Julien.
\newblock A closer look at memorization in deep networks.
\newblock In {\em International Conference on Machine Learning (ICML)}, 2017.

\bibitem{bansal2022measures}
Rachit Bansal, Danish Pruthi, and Yonatan Belinkov.
\newblock Measures of information reflect memorization patterns.
\newblock In {\em Advances in Neural Information Processing Systems (NeurIPS)},
  2022.

\bibitem{bartlett20benign}
Peter~L. Bartlett, Philip~M. Long, Gábor Lugosi, and Alexander Tsigler.
\newblock Benign overfitting in linear regression.
\newblock {\em Proceedings of the National Academy of Sciences},
  117(48):30063--30070, 2020.

\bibitem{bartlett2021deep}
Peter~L Bartlett, Andrea Montanari, and Alexander Rakhlin.
\newblock Deep learning: a statistical viewpoint.
\newblock {\em Acta numerica}, 30:87--201, 2021.

\bibitem{bassily21}
Raef Bassily, Kobbi Nissim, Adam Smith, Thomas Steinke, Uri Stemmer, and
  Jonathan Ullman.
\newblock Algorithmic stability for adaptive data analysis.
\newblock {\em SIAM Journal on Computing}, 50(3), 2021.

\bibitem{belkin2021}
Mikhail Belkin.
\newblock Fit without fear: remarkable mathematical phenomena of deep learning
  through the prism of interpolation.
\newblock {\em Acta Numerica}, 30:203--248, 2021.

\bibitem{bombari2022differential}
Simone Bombari, Alessandro Achille, Zijian Wang, Yu-Xiang Wang, Yusheng Xie,
  Kunwar~Yashraj Singh, Srikar Appalaraju, Vijay Mahadevan, and Stefano Soatto.
\newblock Towards differential relational privacy and its use in question
  answering.
\newblock {\em arXiv preprint arXiv:2203.16701}, 2022.

\bibitem{bombari2022memorization}
Simone Bombari, Mohammad~Hossein Amani, and Marco Mondelli.
\newblock Memorization and optimization in deep neural networks with minimum
  over-parameterization.
\newblock In {\em Advances in Neural Information Processing Systems (NeurIPS)},
  2022.

\bibitem{bombari2023universal}
Simone Bombari, Shayan Kiyani, and Marco Mondelli.
\newblock Beyond the universal law of robustness: Sharper laws for random
  features and neural tangent kernels.
\newblock In {\em International Conference on Machine Learning (ICML)}, 2023.

\bibitem{bousquet2002stability}
Olivier Bousquet and Andr{\'e} Elisseeff.
\newblock Stability and generalization.
\newblock {\em The Journal of Machine Learning Research}, 2:499--526, 2002.

\bibitem{bubeck2021a}
Sebastien Bubeck and Mark Sellke.
\newblock A universal law of robustness via isoperimetry.
\newblock In {\em Advances in Neural Information Processing Systems (NeurIPS)},
  2021.

\bibitem{chang21augmentation}
C.~Chang, G.~Adam, and A.~Goldenberg.
\newblock Towards robust classification model by counterfactual and invariant
  data generation.
\newblock In {\em IEEE/CVF Conference on Computer Vision and Pattern
  Recognition (CVPR)}, 2021.

\bibitem{chizat2019lazy}
Lenaic Chizat, Edouard Oyallon, and Francis Bach.
\newblock On lazy training in differentiable programming.
\newblock In {\em Neural Information Processing Systems (NeurIPS)}, 2019.

\bibitem{dohmatob2022non}
Elvis Dohmatob and Alberto Bietti.
\newblock On the (non-)robustness of two-layer neural networks in different
  learning regimes.
\newblock {\em arXiv preprint arXiv:2203.11864}, 2022.

\bibitem{elisseeff2002}
André Elisseeff and Massimiliano Pontil.
\newblock Leave-one-out error and stability of learning algorithms with
  applications stability of randomized learning algorithms source.
\newblock {\em International Journal of Systems Science (IJSySc)}, 6, 2002.

\bibitem{fan2020spectra}
Zhou Fan and Zhichao Wang.
\newblock Spectra of the conjugate kernel and neural tangent kernel for
  linear-width neural networks.
\newblock In {\em Advances in Neural Information Processing Systems (NeurIPS)},
  2020.

\bibitem{feldman2020}
Vitaly Feldman.
\newblock Does learning require memorization? {A} short tale about a long tail.
\newblock In {\em ACM Symposium on Theory of Computing (STOC)}, pages 954--959,
  2020.

\bibitem{feldman2020b}
Vitaly Feldman and Chiyuan Zhang.
\newblock What neural networks memorize and why: Discovering the long tail via
  influence estimation.
\newblock In {\em Advances in Neural Information Processing Systems (NeurIPS)},
  2020.

\bibitem{Geirhos_2020}
Robert Geirhos, Jörn-Henrik Jacobsen, Claudio Michaelis, Richard Zemel,
  Wieland Brendel, Matthias Bethge, and Felix~A. Wichmann.
\newblock Shortcut learning in deep neural networks.
\newblock {\em Nature Machine Intelligence}, 2(11):665–673, 2020.

\bibitem{geirhos2018imagenettrained}
Robert Geirhos, Patricia Rubisch, Claudio Michaelis, Matthias Bethge, Felix~A.
  Wichmann, and Wieland Brendel.
\newblock Imagenet-trained {CNN}s are biased towards texture; increasing shape
  bias improves accuracy and robustness.
\newblock In {\em International Conference on Learning Representations (ICLR)},
  2019.

\bibitem{ghorbani2021linearized}
Behrooz Ghorbani, Song Mei, Theodor Misiakiewicz, and Andrea Montanari.
\newblock Linearized two-layers neural networks in high dimension.
\newblock {\em The Annals of Statistics}, 49(2):1029--1054, 2021.

\bibitem{bishwa23biasfeat}
Bishwamittra Ghosh, Debabrota Basu, and Kuldeep~S. Meel.
\newblock “{H}ow biased are your features?”: Computing fairness influence
  functions with global sensitivity analysis.
\newblock In {\em ACM Conference on Fairness, Accountability, and Transparency
  (FAccT)}, 2023.

\bibitem{Hermann2020whatshapes}
Katherine Hermann and Andrew Lampinen.
\newblock What shapes feature representations? {E}xploring datasets,
  architectures, and training.
\newblock In {\em Advances in Neural Information Processing Systems (NeurIPS)},
  2020.

\bibitem{hatmatrix}
David~C. Hoaglin and Roy~E. Welsch.
\newblock The hat matrix in regression and anova.
\newblock {\em The American Statistician}, 32(1):17--22, 1978.

\bibitem{JacotEtc2018}
Arthur Jacot, Franck Gabriel, and Clément Hongler.
\newblock Neural tangent kernel: Convergence and generalization in neural
  networks.
\newblock In {\em Advances in Neural Information Processing Systems (NeurIPS)},
  2018.

\bibitem{johnson1990matrix}
Charles~R. Johnson.
\newblock {\em Matrix Theory and Applications}.
\newblock American Mathematical Society, 1990.

\bibitem{kalimeris19sgd}
Dimitris Kalimeris, Gal Kaplun, Preetum Nakkiran, Benjamin Edelman, Tristan
  Yang, Boaz Barak, and Haofeng Zhang.
\newblock Sgd on neural networks learns functions of increasing complexity.
\newblock In {\em Advances in Neural Information Processing Systems (NeurIPS)},
  2019.

\bibitem{Leino2020}
Klas Leino and Matt Fredrikson.
\newblock Stolen memories: leveraging model memorization for calibrated
  white-box membership inference.
\newblock In {\em Proceedings of the 29th USENIX Conference on Security
  Symposium}, 2020.

\bibitem{loo2022evolution}
Noel Loo, Ramin Hasani, Alexander Amini, and Daniela Rus.
\newblock Evolution of neural tangent kernels under benign and adversarial
  training.
\newblock In {\em Advances in Neural Information Processing Systems (NeurIPS)},
  2022.

\bibitem{loo2024understanding}
Noel Loo, Ramin Hasani, Mathias Lechner, Alexander Amini, and Daniela Rus.
\newblock Understanding reconstruction attacks with the neural tangent kernel
  and dataset distillation.
\newblock In {\em The Twelfth International Conference on Learning
  Representations}, 2024.

\bibitem{louart2018random}
Cosme Louart, Zhenyu Liao, and Romain Couillet.
\newblock A random matrix approach to neural networks.
\newblock {\em The Annals of Applied Probability}, 28(2):1190--1248, 2018.

\bibitem{loureiro21}
Bruno Loureiro, Cedric Gerbelot, Hugo Cui, Sebastian Goldt, Florent Krzakala,
  Marc Mezard, and Lenka Zdeborov\'{a}.
\newblock Learning curves of generic features maps for realistic datasets with
  a teacher-student model.
\newblock In M.~Ranzato, A.~Beygelzimer, Y.~Dauphin, P.S. Liang, and J.~Wortman
  Vaughan, editors, {\em Advances in Neural Information Processing Systems},
  volume~34, pages 18137--18151. Curran Associates, Inc., 2021.

\bibitem{mei2022generalization}
Song Mei and Andrea Montanari.
\newblock The generalization error of random features regression: Precise
  asymptotics and the double descent curve.
\newblock {\em Communications on Pure and Applied Mathematics}, 75(4):667--766,
  2022.

\bibitem{montanari2022interpolation}
Andrea Montanari and Yiqiao Zhong.
\newblock The interpolation phase transition in neural networks: Memorization
  and generalization under lazy training.
\newblock {\em The Annals of Statistics}, 50(5):2816--2847, 2022.

\bibitem{sayan2006}
Sayan Mukherjee, Partha Niyogi, Tomaso Poggio, and Ryan Rifkin.
\newblock Learning theory: Stability is sufficient for generalization and
  necessary and sufficient for consistency of empirical risk minimization.
\newblock {\em Adv. Comput. Math.}, 25:161--193, 2006.

\bibitem{Nakkiran2020Deep}
Preetum Nakkiran, Gal Kaplun, Yamini Bansal, Tristan Yang, Boaz Barak, and Ilya
  Sutskever.
\newblock Deep double descent: Where bigger models and more data hurt.
\newblock In {\em International Conference on Learning Representations (ICLR)},
  2020.

\bibitem{QuynhICML2017}
Quynh Nguyen and Matthias Hein.
\newblock The loss surface of deep and wide neural networks.
\newblock In {\em International Conference on Machine Learning (ICML)}, 2017.

\bibitem{QuynhICML2018}
Quynh Nguyen and Matthias Hein.
\newblock Optimization landscape and expressivity of deep {CNN}s.
\newblock In {\em International Conference on Machine Learning (ICML)}, 2018.

\bibitem{QuynhMarco2020}
Quynh Nguyen and Marco Mondelli.
\newblock Global convergence of deep networks with one wide layer followed by
  pyramidal topology.
\newblock In {\em Advances in Neural Information Processing Systems (NeurIPS)},
  2020.

\bibitem{tightbounds}
Quynh Nguyen, Marco Mondelli, and Guido Montufar.
\newblock Tight bounds on the smallest eigenvalue of the neural tangent kernel
  for deep {ReLU} networks.
\newblock In {\em International Conference on Machine Learning (ICML)}, 2021.

\bibitem{booleananalysis}
Ryan O'Donnell.
\newblock {\em Analysis of Boolean Functions}.
\newblock Cambridge University Press, 2014.

\bibitem{oymak2019overparameterized}
Samet Oymak and Mahdi Soltanolkotabi.
\newblock Overparameterized nonlinear learning: Gradient descent takes the
  shortest path?
\newblock In {\em International Conference on Machine Learning (ICML)}, 2019.

\bibitem{pennington17nonlinear}
Jeffrey Pennington and Pratik Worah.
\newblock Nonlinear random matrix theory for deep learning.
\newblock In {\em Advances in Neural Information Processing Systems (NeurIPS)},
  2017.

\bibitem{plumb2022finding}
Gregory Plumb, Marco~Tulio Ribeiro, and Ameet Talwalkar.
\newblock Finding and fixing spurious patterns with explanations.
\newblock {\em Transactions on Machine Learning Research}, 2022.

\bibitem{qiu2023complexity}
GuanWen Qiu, Da~Kuang, and Surbhi Goel.
\newblock Complexity matters: Dynamics of feature learning in the presence of
  spurious correlations.
\newblock In {\em NeurIPS Workshop on Mathematics of Modern Machine Learning},
  2023.

\bibitem{rahimi2007random}
Ali Rahimi and Benjamin Recht.
\newblock Random features for large-scale kernel machines.
\newblock In {\em Advances in Neural Information Processing Systems (NIPS)},
  2007.

\bibitem{sagawa2020overparam}
Shiori Sagawa, Aditi Raghunathan, Pang~Wei Koh, and Percy Liang.
\newblock An investigation of why overparameterization exacerbates spurious
  correlations.
\newblock In {\em International Conference on Machine Learning (ICML)}, 2020.

\bibitem{schur1911}
Jssai Schur.
\newblock Bemerkungen zur theorie der beschr{\"a}nkten bilinearformen mit
  unendlich vielen ver{\"a}nderlichen.
\newblock {\em Journal f{\"u}r die reine und angewandte Mathematik (Crelles
  Journal)}, 1911(140):1--28, 1911.

\bibitem{seddik2020random}
Mohamed El~Amine Seddik, Cosme Louart, Mohamed Tamaazousti, and Romain
  Couillet.
\newblock Random matrix theory proves that deep learning representations of
  {GAN-}data behave as gaussian mixtures.
\newblock In {\em International Conference on Machine Learning (ICML)}, 2020.

\bibitem{seonguk22bias}
Seonguk Seo, Joon-Young Lee, and Bohyung Han.
\newblock Information-theoretic bias reduction via causal view of spurious
  correlation.
\newblock In {\em AAAI Conference on Artificial Intelligence}, 2022.

\bibitem{shah2020simplicity}
Harshay Shah, Kaustav Tamuly, Aditi Raghunathan, Prateek Jain, and Praneeth
  Netrapalli.
\newblock The pitfalls of simplicity bias in neural networks.
\newblock In {\em Advances in Neural Information Processing Systems (NeurIPS)},
  2020.

\bibitem{singla2022salient}
Sahil Singla and Soheil Feizi.
\newblock Salient imagenet: How to discover spurious features in deep learning?
\newblock In {\em International Conference on Learning Representations}, 2022.

\bibitem{theoreticalinsghts}
Mahdi Soltanolkotabi, Adel Javanmard, and Jason~D Lee.
\newblock Theoretical insights into the optimization landscape of
  over-parameterized shallow neural networks.
\newblock {\em IEEE Transactions on Information Theory}, 65(2):742--769, 2018.

\bibitem{stephenson2021on}
Cory Stephenson, suchismita padhy, Abhinav Ganesh, Yue Hui, Hanlin Tang, and
  SueYeon Chung.
\newblock On the geometry of generalization and memorization in deep neural
  networks.
\newblock In {\em International Conference on Learning Representations}, 2021.

\bibitem{Tropp2011}
Joel Tropp.
\newblock User-friendly tail bounds for sums of random matrices.
\newblock {\em Foundations of Computational Mathematics}, page 389–434, 2012.

\bibitem{tsilivis2022what}
Nikolaos Tsilivis and Julia Kempe.
\newblock What can the neural tangent kernel tell us about adversarial
  robustness?
\newblock In {\em Advances in Neural Information Processing Systems (NeurIPS)},
  2022.

\bibitem{vershynin2018high}
Roman Vershynin.
\newblock {\em High-dimensional probability: An introduction with applications
  in data science}.
\newblock Cambridge university press, 2018.

\bibitem{wang2021deformed}
Zhichao Wang and Yizhe Zhu.
\newblock Deformed semicircle law and concentration of nonlinear random
  matrices for ultra-wide neural networks.
\newblock {\em arXiv preprint arXiv:2109.09304}, 2021.

\bibitem{xiao2021noise}
Kai~Yuanqing Xiao, Logan Engstrom, Andrew Ilyas, and Aleksander Madry.
\newblock Noise or signal: The role of image backgrounds in object recognition.
\newblock In {\em International Conference on Learning Representations (ICLR)},
  2021.

\bibitem{yang2022understanding}
Yao-Yuan Yang, Chi-Ning Chou, and Kamalika Chaudhuri.
\newblock Understanding rare spurious correlations in neural networks.
\newblock In {\em ICML Workshop on Spurious Correlations, Invariance and
  Stability}, 2022.

\bibitem{zhou21combating}
Chunting Zhou, Xuezhe Ma, Paul Michel, and Graham Neubig.
\newblock Examining and combating spurious features under distribution shift.
\newblock In {\em International Conference on Machine Learning (ICML)}, 2021.

\bibitem{zliobaite15}
Indre Zliobaite.
\newblock On the relation between accuracy and fairness in binary
  classification.
\newblock In {\em 2nd ICML Workshop on Fairness, Accountability, and
  Transparency in Machine Learning (FATML)}, 2015.

\end{thebibliography}

}


\newpage

\appendix

\section{Additional notations and remarks}\label{app:notation}
Given a sub-exponential random variable $X$, let $\|X\|_{\psi_1} = \inf \{ t>0 \,\,: \,\,\mathbb E[\exp(|X|/t)] \le 2 \}$.
Similarly, for a sub-Gaussian random variable, let $\|X\|_{\psi_2} = \inf \{ t>0 \,\,: \,\,\mathbb E[\exp(X^2/t^2)] \le 2 \}$.
We use the analogous definitions for vectors. In particular, let $X \in \mathbb R^n$ be a random vector, then $\subGnorm{X} := \sup_{\norm{u}_2=1} \subGnorm{u^\top X}$ and $\subEnorm{X} := \sup_{\norm{u}_2=1} \subEnorm{u^\top X}$. Notice that if a vector has independent, mean-0, sub-Gaussian (sub-exponential) entries, then it is sub-Gaussian (sub-exponential). This is a direct consequence of Hoeffding's inequality and Bernstein's inequality (see Theorems 2.6.3 and 2.8.2 in \cite{vershynin2018high}).

We say that a random variable or vector respects the Lipschitz concentration property if there exists an absolute constant $c > 0$ such that, for every Lipschitz continuous function $\tau: \RR^d \to \RR$, we have $\E |\tau(X)| < + \infty$, and for all $t>0$,
\begin{equation}\label{eq:deflipconc}
    \PP\left(\abs{\tau(x)- \E_X [\tau(x)]}>t\right) \leq 2e^{-ct^2 / \norm{\tau}_{\Lip}^2}.
\end{equation}

When we state that a random variable or vector $X$ is sub-Gaussian (or sub-exponential), we implicitly mean $\subGnorm{X} = \bigO{1}$, \emph{i.e.} it doesn't increase with the scalings of the problem. Notice that, if $X$ is Lipschitz concentrated, then $X - \E[X]$ is sub-Gaussian.
If $X \in \R$ is sub-Gaussian and $\tau: \R \to \R$ is Lipschitz, we have that $\tau(X)$ is sub-Gaussian as well. Also, if a random variable is sub-Gaussian or sub-exponential, its $p$-th momentum is upper bounded by a constant (that might depend on $p$). 

In general, we indicate with $C$ and $c$ absolute, strictly positive, numerical constants, that do not depend on the scalings of the problem, \emph{i.e.} input dimension, number of neurons, or number of training samples. Their value may change from line to line.

Given a matrix $A$, we indicate with $A_{i:}$ its $i$-th row, and with $A_{:j}$ its $j$-th column. Given a square matrix $A$, we denote by $\evmin{A}$ its smallest eigenvalue. Given a matrix $A$, we indicate with $\sigma_{\min}(A) = \sqrt{\evmin{A^\top A}}$ its smallest singular value, with $\opnorm{A}$ its operator norm (and largest singular value), and with $\norm{A}_F$ its Frobenius norm ($\norm{A}^2_F = \sum_{ij} A_{ij}^2$).

Given two matrices $A, B\in \R^{m\times n}$, we denote by $A \circ B$ their Hadamard product, and by $A \ast B=[(A_{1:}\otimes B_{1:}),\ldots,(A_{m:}\otimes B_{m:})]^\top \in\RR^{m\times n^2}$ their row-wise Kronecker product (also known as Khatri-Rao product). We denote $A^{*2} = A \ast A$. We remark that $\left( A * B \right) \left( A * B \right)^\top = AA^\top \circ BB^\top$. We say that a matrix $A\in \R^{n \times n}$ is positive semi definite (p.s.d.) if it's symmetric and for every vector $v \in \R^n$ we have $v^\top A v \geq 0$.

\subsection{Hermite polynomials}\label{app:Hermite}

In this subsection, we refresh standard notions on the Hermite polynomials. For a more comprehensive discussion, we refer to \cite{booleananalysis}.
The (probabilist's) Hermite polynomials $\{ h_j \}_{j \in \mathbb N}$ are an orthonormal basis for $L^2(\R, \gamma)$, where $\gamma$ denotes the standard Gaussian measure. The following result holds. 

\begin{proposition}[Proposition 11.31, \cite{booleananalysis}]
Let $\rho_1, \rho_2$ be two standard Gaussian random variables, with correlation $\rho \in [-1, 1]$. Then,
\begin{equation}
    \E_{\rho_1, \rho_2} \left[ h_i(\rho_1) h_j(\rho_2) \right] = \delta_{ij} \rho^i,
\end{equation}
where $\delta_{ij} = 1$ if $i = j$, and $0$ otherwise.
\end{proposition}

The first 5 Hermite polynomials are
\begin{equation}
    h_0(\rho) = 1, \quad h_1(\rho) = \rho, \quad h_2(\rho) = \frac{\rho^2 - 1}{\sqrt 2}, \quad h_3(\rho) = \frac{\rho^3 - 3\rho}{\sqrt 6}, \quad h_4(\rho) = \frac{\rho^4 - 6 \rho^2 + 3}{\sqrt{24}}.
\end{equation}

\begin{proposition}[Definition 11.34, \cite{booleananalysis}]
Every function $\phi \in L^2(\R, \gamma)$ is uniquely expressible as
\begin{equation}
    \phi(\rho) = \sum_{i \in \mathbb N} \mu^\phi_i h_i(\rho),
\end{equation}
where the real numbers $\mu^\phi_i$'s are called the Hermite coefficients of $\phi$, and the convergence is in $L^2(\R, \gamma)$. More specifically, 
\begin{equation}
    \lim_{n \to + \infty} \norm{\left(\sum_{i = 0}^n \mu^\phi_i h_i(\rho)\right) - \phi(\rho)}_{L^2(\R, \gamma)} = 0.
\end{equation}
\end{proposition}

This readily implies the following result.
\begin{proposition}
    Let $\rho_1, \rho_2$ be two standard Gaussian random variables with correlation $\rho \in [-1, 1]$, and let $\phi, \tau \in L^2(\R, \gamma)$. Then,
\begin{equation}
    \E_{\rho_1, \rho_2} \left[ \phi(\rho_1) \tau(\rho_2) \right] = \sum_{i \in \mathbb N} \mu^\phi_i \mu^\tau_i \rho^i.
\end{equation}
\end{proposition}

\section{Proof of Lemma \ref{lemma:proj}}\label{app:proj}

We start by refreshing some useful notions of linear algebra. 
Let $A \in \R^{N \times p}$ be a matrix, with $p \geq N$, and $A_{-1} \in \R^{(N-1) \times p}$ be obtained from $A$ after removing the first row. We assume $AA^\top$ to be invertible, \emph{i.e.,} the rows of $A$ are linearly independent. Thus, also the rows of $A_{-1}$ are linearly independent, implying that $A_{-1}A_{-1}^\top$ is invertible as well. We indicate with $P_A \in \R^{p \times p}$ the projector over $\Span \{ \Rows (A) \}$, and we correspondingly define $P_{A_{-1}} \in \R^{p \times p}$. As $AA^\top$ is invertible, we have that $\textrm{rank}(A) = N$.

By singular value decomposition, we have $A = UDO^\top$, where $U \in \R^{N \times N}$ and $O \in \R^{p \times p}$ are orthogonal matrices, and $D \in \R^{N \times p}$ contains the (all strictly positive) singular values of $A$ in its \enquote{left} diagonal, and is 0 in every other entry. Let us define $O_1 \in \R^{N \times p}$ as the matrix containing the first $N$ rows of $O^\top$. This notation implies that if $O_1 u = 0$ for $u \in \R^ p$, then $A u =0$, \emph{i.e.}, $u \in \Span \{ \Rows (A) \}^\perp$. The opposite implication is also true, which implies that $\Span \{ \Rows (A) \} = \Span \{ \Rows (O_1) \}$. As the rows of $O_1$ are orthogonal, we can then write 
\begin{equation}
    P_A = O_1^\top O_1.
\end{equation}

We define $D_s \in \R^{N \times N}$, as the square, diagonal, and invertible matrix corresponding to the first $N$ columns of $D$. Let's also define $I_N \in \R^{p \times p}$ as the matrix containing 1 in the first $N$ entries of its diagonal, and 0 everywhere else. We have
\begin{equation}\label{eq:projform}
\begin{aligned}
    P_A =& O_1^\top O_1 = O I_N O^\top \\
    =& O D^\top  D_s^{-2} D O^\top = O D^\top U^\top U D_s^{-2} U^\top U D O^\top \\
    =& A^\top \left( U D_s^2 U^\top \right)^{-1} A = A^\top \left( U D O^\top O D^\top U^\top \right)^{-1} A \\
    =& A^\top \left( AA^\top \right)^{-1} A \equiv A^+A,
\end{aligned}
\end{equation}
where $A^+$ denotes the Moore-Penrose inverse.

Notice that this last form enables us to easily derive
\begin{equation}\label{eq:leaveoneouttrick}
    P_{A_{-1}} A^+ v = A_{-1}^+ A_{-1} A^+ v = A_{-1}^+ I_{-1} A A^+ v =  A_{-1}^+ I_{-1} v = A_{-1}^+ v_{-1},
\end{equation}
where $v \in \R^N$, $I_{-1} \in \R^{(N - 1) \times N}$ is the $N \times N$ identity matrix without the first row, and $v_{-1} \in \R^{N-1}$ corresponds to $v$ without its first entry.

\begin{lemma}\label{lemma:dendiv0}
Let $\Phi \in \R^{N \times k}$ be a matrix whose first row is denoted as $\varphi(z_1)$. Let $\Phi_{-1} \in \R^{(N-1) \times k}$ be the original matrix without the first row, and let $\Ppm$ be the projector over the \textup{span} of its rows. Then,
\begin{equation}
    \norm{\Ppm^\perp \varphi(z_1)}^2_2 \geq \evmin{\Phi \Phi^\top}.
\end{equation}
\end{lemma}
\begin{proof}
    If $\evmin{\Phi \Phi^\top} = 0$, the thesis becomes trivial. Otherwise, we have that $\Phi \Phi^\top$, and therefore $\Phi_{-1} \Phi_{-1}^\top$, are invertible. 
    
    Let $u\in \R^N$ be a vector, such that its first entry $u_1 = 1$. We denote with $u_{-1} \in \R^{N-1}$ the vector $u$ without its first component, \emph{i.e.} $u = [1, u_{-1}]$. We have
    \begin{equation}\label{eq:dendiv0}
        \norm{\Phi^\top u}_2^2 \geq \evmin{\Phi \Phi^{\top}} \norm{u}_2^2 \geq \evmin{\Phi \Phi^\top}.
    \end{equation}
    Setting $u_{-1} = - \left( \Phi_{-1}\Phi_{-1}^\top \right)^{-1} \Phi_{-1} \varphi(z_1)$, we get
    \begin{equation}
        \Phi^\top u = \varphi(z_1) + \Phi_{-1}^\top u_{-1} = \varphi(z_1) - \Ppm \varphi(z_1) = \Ppm^\perp \varphi(z_1).
    \end{equation}
    Plugging this in \eqref{eq:dendiv0}, we get the thesis.
\end{proof}

At this point, we are ready to prove Lemma \ref{lemma:proj}.

\begin{proof}[Proof of Lemma \ref{lemma:proj}]
    We indicate with $\Phi_{-1} \in \R^{(N-1) \times p}$ the feature matrix of the training set $\Phi \in \R^{N \times p}$ \emph{without} the first sample $z_1$. In other words, $\Phi_{-1}$ is equivalent to $\Phi$, without the first row. Notice that since $K = \Phi\Phi^\top$ is invertible, also $K_{-1} := \Phi_{-1} \Phi_{-1}^\top$ is.

    We can express the projector over the span of the rows of $\Phi$ in terms of the projector over the span of the rows of $\Phi_{-1}$ as follows
    \begin{equation}
        \Pp = \Ppm + \frac{\Ppm^\perp \varphi(z_1) \varphi(z_1)^\top \Ppm^\perp}{\norm{\Ppm^\perp \varphi(z_1)}_2^2}.
    \end{equation}
    The above expression is a consequence of the Gram-Schmidt formula, and the quantity at the denominator is different from zero because of Lemma \ref{lemma:dendiv0}, as $K$ is invertible.

    We indicate with $\Phi^+ = \Phi^\top K^{-1}$ the Moore-Penrose pseudo-inverse of $\Phi$. Using \eqref{eq:thetastar}, we can define $\theta_{-1}^* :=  \theta_0 + \Phi_{-1}^+ \left( G_{-1} - f(Z_{-1}, \theta_0) \right)$, \emph{i.e.}, the set of parameters the algorithm would have converged to if trained over $(Z_{-1}, G_{-1})$, the original data-set without the first pair sample-label $(z_1, g_1)$.

    Notice that $\Pp \Phi^\top = \Phi^\top$, as a consequence of \eqref{eq:projform}. Thus, again using \eqref{eq:thetastar}, for any $z$ we can write
    \begin{equation}\label{eq:bigprojeq}
    \begin{aligned}
        f(z, \theta^*) - \varphi(z)^\top \theta_0 &= \varphi(z)^\top \Phi^+ \left( G - f(Z, \theta_0) \right) \\
        &= \varphi(z)^\top \Pp  \Phi^+ \left( G - f(Z, \theta_0) \right) \\
        &= \varphi(z)^\top \left( \Ppm + \frac{\Ppm^\perp \varphi(z_1) \varphi(z_1)^\top \Ppm^\perp}{\norm{\Ppm^\perp \varphi(z_1)}_2^2}\right)  \Phi^+ \left( G - f(Z, \theta_0) \right).
    \end{aligned}
    \end{equation}

    Notice that, thanks to \eqref{eq:leaveoneouttrick}, we can manipulate the first term in the bracket as follows
    \begin{equation}\label{eq:withlootrick}
    \begin{aligned}
        \varphi(z)^\top \Ppm \Phi^+ \left( G - f(Z, \theta_0) \right) &= \varphi(z)^\top \Phi_{-1}^+ \left( G_{-1} - f(Z_{-1}, \theta_0) \right) \\
        &= f(z, \theta_{-1}^*) - \varphi(z)^\top \theta_0.
    \end{aligned}
    \end{equation}
    Thus, bringing the result of \eqref{eq:withlootrick} on the LHS, \eqref{eq:bigprojeq} becomes
    \begin{equation}\label{eq:endprojecting}
    \begin{aligned}
        f(z, \theta^*) - f(z, \theta_{-1}^*) &= \frac{\varphi(z)^\top \Ppm^\perp \varphi(z_1)}{\norm{\Ppm^\perp \varphi(z_1)}_2^2}  \varphi(z_1)^\top \Ppm^\perp \Phi^+ \left( G - f(Z, \theta_0) \right) \\
        &= \frac{\varphi(z)^\top \Ppm^\perp \varphi(z_1)}{\norm{\Ppm^\perp \varphi(z_1)}_2^2}  \varphi(z_1)^\top \left( I - \Ppm \right) \Phi^+ \left( G - f(Z, \theta_0) \right) \\
        &= \frac{\varphi(z)^\top \Ppm^\perp \varphi(z_1)}{\norm{\Ppm^\perp \varphi(z_1)}_2^2}  \left( f(z_1, \theta^*) - f(z_1, \theta_{-1}^*) \right),
    \end{aligned}
    \end{equation}
    where in the last step we again used \eqref{eq:thetastar} and \eqref{eq:withlootrick}.  
\end{proof}

\section{Useful lemmas}

\begin{lemma}\label{lemma:functiontoconvex}
    Let $x$ and $y$ be two Lipschitz concentrated, independent random vectors. Let $\zeta(x, y)$ be a Lipschitz function in both arguments, \emph{i.e.}, for every $\delta$,
    \begin{equation}
    \begin{aligned}
        |\zeta(x + \delta, y) - \zeta(x, y)| &\leq L \norm{\delta}_2,  \\
        |\zeta(x, y + \delta) - \zeta(x, y)| &\leq L \norm{\delta}_2,
    \end{aligned}
    \end{equation}
    for all $x$ and $y$.
    Then, $\zeta(x, y)$ is a Lipschitz concentrated random variable, in the joint probability space of $x$ and $y$.
\end{lemma}
\begin{proof}
    To prove the thesis, we need to show that, for every 1-Lipschitz function $\tau$, the following holds
    \begin{equation}\label{eq:liplipconcentr2}
        \P_{xy} \left( \left|\tau \left(\zeta(x,y) \right)  - \E_{xy} \left[ \tau \left(\zeta(x,y) \right)\right] \right|> t\right) < e^{-ct^2},
    \end{equation}
    where $c$ is a universal constant.
    An application of the triangle inequality gives
    \begin{equation}
    \begin{aligned}
    &\left|\tau \left(\zeta(x,y) \right)  - \E_{xy} \left[ \tau \left(\zeta(x,y) \right)\right] \right|   \\
        &\leq \left|\tau \left(\zeta(x,y) \right)  - \E_{x} \left[ \tau \left(\zeta(x,y) \right)\right] \right|  + \left| \E_{x} \left[\tau \left(\zeta(x,y) \right)\right]  - \E_y \E_{x} \left[ \tau \left(\zeta(x,y) \right)\right] \right| =: A + B.
    \end{aligned}
    \end{equation}
    Thus, we can upper bound LHS of \eqref{eq:liplipconcentr2} as follows:
\begin{equation}
        \P_{xy} \left( \left|\tau \left(\zeta(x,y) \right)  - \E_{xy} \left[ \tau \left(\zeta(x,y) \right)\right] \right|> t\right) \leq \P_{xy} \left( A + B > t\right).
    \end{equation}
    If $A$ and $B$ are positive random variables, it holds that $\P(A + B > t)\leq \P(A > t/2) + \P(B > t/2)$. Then, the LHS of \eqref{eq:liplipconcentr2} is also upper bounded by
    \begin{equation}
    \begin{aligned}
        &\P_{xy} \left( \left|\tau \left(\zeta(x,y) \right)  - \E_{x} \left[ \tau \left(\zeta(x,y) \right)\right] \right| > t/2 \right)        + \P_{xy} \left(  \left| \E_{x} \left[\tau \left(\zeta(x,y) \right)\right]  - \E_y \E_{x} \left[ \tau \left(\zeta(x,y) \right)\right] \right|> t/2 \right).
    \end{aligned}
    \end{equation}
    Since $\tau \circ \zeta$ is Lipschitz with respect to $x$ for every $y$, we have
    \begin{equation}
        \P_{xy} \left( \left|\tau \left(\zeta(x,y) \right)  - \E_{x} \left[ \tau \left(\zeta(x,y) \right)\right] \right| > t/2 \right) < e^{-c_1 t^2},
    \end{equation}
    for some absolute constant $c_1$.
  Furthermore,  $\chi(y) := \E_{x} \left[ \tau \left(\zeta(x,y) \right)\right]$ is also Lipschitz, as
    \begin{equation}
    \begin{aligned}
        |\chi(y + \delta) - \chi(y)| = |\E_{x} \left[ \tau \left(\zeta(x,y + \delta) \right) - \tau \left(\zeta(x,y) \right)\right]|  
        \leq \E_{x} \left[| \tau  \left(\zeta(x,y + \delta) \right) - \tau \left(\zeta(x,y) \right) |\right] \leq L \norm{\delta}_2.
    \end{aligned}
    \end{equation}
    Then, we can write
    \begin{equation}
    \begin{aligned}
        \P_{xy} \left( \left| \E_{x} \left[\tau \left(\zeta(x,y) \right)\right]  - \E_y \E_{x} \left[ \tau \left(\zeta(x,y) \right)\right] \right|> t/2 \right)  
        = \P_{y} \left( \left| \chi(y)  - \E_y  \left[ \chi(y) \right] \right| > t/2 \right) < e^{-c_2 t^2},
    \end{aligned}
    \end{equation}
        for some absolute constant $c_2$. Thus,
    \begin{equation}
        \P_{xy} \left( \left|\tau \left(\zeta(x,y) \right)  - \E_{xy} \left[ \tau \left(\zeta(x,y) \right)\right] \right|> t\right) < e^{-c_1 t^2} + e^{-c_2 t^2} \leq e^{-ct^2},
    \end{equation}
    for some absolute constant $c$, which concludes the proof.
\end{proof}

\begin{lemma}\label{lemma:lipcompxy}
    Let $x \sim \mathcal P_X$, $y \sim \mathcal P_Y$ and $z = [x, y] \sim \mathcal P_Z$. Let Assumption \ref{ass:datadist} hold. Then, $z$ is a Lipschitz concentrated random vector.
\end{lemma}
\begin{proof}
    We want to prove that, for every 1-Lipschitz function $\tau$, the following holds
    \begin{equation}\label{eq:liplipconcentr}
        \P_{z} \left( \left|\tau \left(z \right)  - \E_{z} \left[ \tau \left( z \right)\right] \right|> t\right) < e^{-ct^2},
    \end{equation}
 for some universal constant $c$.   
    As we can write $z = [x, y]$, defining $z' = [x', y]$, we have
    \begin{equation}
         \left|\tau \left(z \right) - \tau \left(z' \right) \right| \leq \norm{z - z'}_2 = \norm{x - x'}_2,
    \end{equation}
    \emph{i.e.}, for every $y$, $\tau$ is 1-Lipschitz with respect to $x$. The same can be shown for $y$, with an equivalent argument.
    Since $x$ and $y$ are independent random vectors, both Lipschitz concentrated, Lemma \ref{lemma:functiontoconvex} gives the thesis.
\end{proof}

\begin{lemma}\label{lemma:justBern}
    Let $\tau$ and $\zeta$ be two Lipschitz functions. Let $z, z' \in \R^d$ be two fixed vectors such that $\norm{z}_2 = \norm{z'}_2 = \sqrt d$. Let $V$ be a $k \times d$ matrix such that $V_{i,j} \distas{}_{\rm i.i.d.}\mathcal{N}(0, 1 / d)$. Then, for any $t > 1$,
    \begin{equation}
        \left| \tau(Vz)^\top \zeta(Vz') - \E_V \left[ \tau(Vz)^\top \zeta(Vz') \right] \right| = \bigO{\sqrt k \log t},
    \end{equation}
    with probability at least $1 - \exp(-c \log^2 t)$ over $V$. Here, $\tau$ and $\zeta$ act component-wise on their arguments. Furthermore, by taking $\tau = \zeta$ and $z = z'$, we have that
    \begin{equation}
        \E_V \left[ \norm{ \tau(Vz)}_2^2 \right] = k \E_\rho \left [\tau^2(\rho) \right],
    \end{equation}
    where $\rho\sim\mathcal N(0, 1)$. This implies that $\norm{\tau(Vz)}_2^2 = \bigO{k}$ with probability at least $1 - \exp(-c k)$ over $V$.
\end{lemma}
\begin{proof}
    We have
    \begin{equation}
        \tau(Vz)^\top \zeta(Vz') = \sum_{j = 1}^k \tau(v_j^\top z) \zeta(v_j^\top z'),
    \end{equation}
    where we used the shorthand $v_j := V_{j:}$. As $\tau$ and $\zeta$ are Lipschitz, $v_j \sim \mathcal N(0, I/d)$, and $\norm{z}_2 = \norm{z'}_2 = \sqrt d$, we have that $\tau(Vz)^\top \zeta(Vz')$ is the sum of $k$ independent sub-exponential random variables, in the probability space of $V$.
    Thus, by Bernstein inequality (cf. Theorem 2.8.1 in \cite{vershynin2018high}), we have
    \begin{equation}\label{eq:toE}
        \left| \tau(Vz)^\top \zeta(Vz') - \E_V \left[ \tau(Vz)^\top \zeta(Vz') \right] \right| = \bigO{\sqrt k \log t}.
    \end{equation}
    with probability at least $1 - \exp(-c \log^2 t)$, over the probability space of $V$, which gives the thesis. The second statement is again implied by the fact that $v_j \sim \mathcal N(0, I/d)$ and $\norm{z}_2 = \sqrt d$.
\end{proof}

\begin{lemma}\label{lemma:verycorr}
    Let $x, x_1 \sim \mathcal P_X$ and $y_1 \sim \mathcal P_Y$ be independent random variables, with $x, x_1\in \mathbb R^{d_x}$ and $y_1\in\mathbb R^{d_y}$, and let Assumption \ref{ass:datadist} hold.
    Let $d = d_x + d_y$, $V$ be a $k \times d$ matrix, such that $V_{i,j} \distas{}_{\rm i.i.d.}\mathcal{N}(0, 1 / d)$, and let $\tau$ be a Lipschitz function. Let $z_1 := [x_1, y_1]$ and $z_1^s := [x, y_1]$. Let $\alpha = d_y / d \in (0, 1)$ and $\mu_l$ be the $l$-th Hermite coefficient of $\tau$. Then, for any $t > 1$,
    \begin{equation}
        \left| \tau(V z_1^s)^\top \tau(V z_1) - k \sum_{l = 0}^{+\infty} \mu_l^2 \alpha^l \right| = \bigO{\sqrt k \left( \sqrt{\frac{k}{d}} + 1 \right) \log t},
    \end{equation}
    with probability at least $1- \exp(- c \log^2 t) - \exp(- c k)$ over $V$ and $x$, where $c$ is a universal constant.
\end{lemma}
\begin{proof}
    Define the vector $x'$ as follows
    \begin{equation}
        x' = \frac{ \sqrt {d_x} \left( I - \frac{x_1 x_1^\top}{d_x} \right) x }{\norm{ \left( I - \frac{x_1 x_1^\top}{d_x} \right) x}_2}.
    \end{equation}

    Note that, by construction, $x_1^\top x' = 0$ and $\norm{x'}_2 = \sqrt {d_x}$. Also, consider a vector $y$ orthogonal to both $x_1$ and $x$. Then, a fast computation returns $y^\top x' = 0$. This means that $x'$ is the vector on the $\sqrt {d_x}$-sphere, lying on the same plane of $x_1$ and $x$, orthogonal to $x_1$. Thus, we can easily compute 
    \begin{equation}
        \frac{\left| x^\top x' \right|}{d_x}= \sqrt{1 - \left( \frac{x^\top x_1}{d_x} \right)^2} \geq 1 - \left( \frac{x^\top x_1}{d_x} \right)^2,
    \end{equation}
where the last inequality derives from $\sqrt{1 - a} \geq 1 - a$ for $a \in [0, 1]$. Then,
    \begin{equation}
        \norm{x - x'}_2^2 = \norm{x}_2^2 + \norm{x'}_2^2 - 2 x^\top x' \leq 2 d_x \left(1 - \left( 1 - \left( \frac{x^\top x_1}{d_x} \right)^2\right)\right) = 2 \frac{\left( x^\top x_1\right)^2}{d_x}.
    \end{equation}

    As $x$ and $x_1$ are both sub-Gaussian, mean-0 vectors, with $\ell_2$ norm equal to $\sqrt {d_x}$, we have that
    \begin{equation}\label{eq:normissubG}
        \P \left( \norm{x - x'}_2 > t \right) \leq \P \left( | x^\top x_1 | > \sqrt {d_x} t / \sqrt 2 \right) < \exp (- c t^2),
    \end{equation}
    where $c$ is an absolute constant. Here the probability is referred to the space of $x$, for a fixed $x_1$. Thus, $\norm{x - x'}_2$ is sub-Gaussian.
    
    We now define $z' := [x', y_1]$. Notice that $z_1^\top z' = \norm{y_1}_2^2 = d_y$ and $\norm{z_1^s - z'}_2 = \norm{x - x'}_2$. We can write
    \begin{equation}\label{eq:verycorr1}
    \begin{aligned}
        \left| \tau(V z_1^s)^\top \tau(V z_1) - \tau(V z')^\top \tau(V z_1) \right| &\leq \norm{\tau(V z_1^s) - \tau(V z')}_2 \norm{\tau(V z_1)}_2 \\
        & \leq C \opnorm{V} \norm{z_1^s - z'}_2 \norm{\tau(V z_1)}_2 \\
        & \leq C_1 \left( \sqrt{\frac{k}{d}} + 1 \right) \norm{x - x'}_2 \sqrt k \\
        & = \bigO{\sqrt k \left( \sqrt{\frac{k}{d}} + 1 \right) \log t}.
    \end{aligned}
    \end{equation}
    Here the second step holds as $\tau$ is Lipschitz; the third step holds with probability at least $1- \exp(- c_1 \log^2 t) - \exp(-c_2 k)$, and it uses rTheorem 4.4.5 of \cite{vershynin2018high} and Lemma \ref{lemma:justBern};
    the fourth step holds with probability at least $1- \exp(- c \log^2 t)$, and it uses \eqref{eq:normissubG}. This probability is intended over $V$ and $x$.
    We further have
    \begin{equation}\label{eq:verycorr2}
        \left| \tau(V z')^\top \tau(V z_1) - \E_V \left[ \tau(V z')^\top \tau(V z_1) \right] \right| = \bigO{\sqrt k \log t},
    \end{equation}
    with probability at least $1- \exp(- c_3 \log^2 t) - \exp(-c_2 k)$ over $V$, because of Lemma \ref{lemma:justBern}.

    We have
    \begin{equation}\label{eq:verycorr3}
        \E_V \left[ \tau(V z')^\top \tau(V z_1) \right] = k \E_{\rho_1 \rho_2} \left[ \tau(\rho_1) \tau(\rho_2) \right],
    \end{equation}
    where we indicate with $\rho_1$ and $\rho_2$ two standard Gaussian random variables, with correlation
    \begin{equation}
        \textup{corr}(\rho_1, \rho_2) = \frac{z_1^\top z}{\norm{z_1}_2 \norm{z'}_2} = \frac{d_y}{d} = \alpha.
    \end{equation}

    Then, exploiting the Hermite expansion of $\tau$, we have
    \begin{equation}\label{eq:verycorr4}
        \E_{\rho_1 \rho_2} \left[ \tau(\rho_1) \tau(\rho_2) \right] = \sum_{l = 0}^{+\infty} \mu_l^2 \alpha^l.
    \end{equation}

    Putting together \eqref{eq:verycorr1}, \eqref{eq:verycorr2}, \eqref{eq:verycorr3}, and \eqref{eq:verycorr4} gives the thesis. 
\end{proof}

\section{Proofs for random features}\label{app:rf}

In this section, we indicate with $Z \in \R^{N \times d}$ the data matrix, such that its rows are sampled independently from $\mathcal P_Z$ (see Assumption \ref{ass:datadist}). We denote by $V \in \R^{k \times d}$ the random features matrix, such that $V_{ij} \distas{}_{\rm i.i.d.}\mathcal{N}(0, 1/d)$. Thus, the feature map is given by 
(see \eqref{eq:featmaprf})
\begin{equation}
    \varphi(z) := \phi(Vz) \in \R^k,
\end{equation}
where $\phi$ is the activation function, applied component-wise to the pre-activations $Vz$. We use the shorthands $\Phi := \phi(ZV^\top) \in \R^{N \times k}$ and $K := \Phi \Phi^\top \in \R^{N \times N}$, we indicate with $\Phi_{-1} \in \R^{(N - 1) \times k}$ the matrix $\Phi$ without the first row, and we define $K_{-1} := \Phi_{-1} \Phi_{-1}^\top$. We call $\Pp$ the projector over the span of the rows of $\Phi$, and $\Ppm$ the projector over the span of the rows of $\Phi_{-1}$. We use the notations $\tilde \varphi(z) := \varphi(z) - \E_V [\varphi(z)]$ and $\tilde \Phi_{-1} := \Phi_{-1} - \E_V [\Phi_{-1}]$ to indicate the centered feature map and matrix respectively, where the centering is with respect to $V$. We indicate with $\mu_l$ the $l$-th Hermite coefficient of $\phi$. We use the notation $z_1^s = [x, y_1]$, where $x\sim \mathcal P_X$ is sampled independently from $V$ and $Z$. We denote by $V_x$ ($V_y$) the first $d_x$ (last $d_y$) columns of $V$, \emph{i.e.}, $V=[V_x, V_y]$. We define $\alpha = d_y/d$.
Throughout this section, for compactness, we drop the subscripts \enquote{RF} from these quantities, as we will only treat the proofs related to Section \ref{sec:rf}. Again for the sake of compactness, we will not re-introduce such quantities in the statements or the proofs of the following lemmas.

The content of this section can be summarized as follows:
\begin{itemize}
    \item In Lemma \ref{lemma:evminRF} we prove a lower bound on the smallest eigenvalue of $K$, adapting to our settings Lemma C.5 of \cite{bombari2023universal}. As our assumptions are less restrictive than those in \cite{bombari2023universal}, we will crucially exploit Lemma \ref{lemma:mgeq2}.
    \item In Lemma \ref{lemma:PpRFon1}, we treat separately a term that derives from $\E_V \left[ \phi(V z) \right] = \mu_0 \mathbf{1}_k$, showing that we can \emph{center} the activation function, without changing our final statement in Theorem \ref{thm:RF}. This step is necessary only if $\mu_0 \neq 0$.
    \item In Lemma \ref{lemma:Pseeslin}, we show that the non-linear component of the features $\tilde \varphi(z_1) - \mu_1 V z_1$ and $\tilde \varphi(z_1^s) - \mu_1 V z_1^s$ have a negligible component in the space spanned by the rows of $\Phi_{-1}$.
    \item In Lemma \ref{lemma:newlemma}, we provide concentration results for $\varphi(z_1^s)^\top \Ppm^\perp \varphi(z_1)$, and we lower bound this same term in Lemma \ref{lemma:tempfinRF}, exploiting also the intermediate result provided in Lemma \ref{lemma:RFlast}.
    \item Finally, we prove Theorem \ref{thm:RF}.
\end{itemize}

\begin{lemma}\label{lemma:mgeq2}
    Let  $A := \left( Z^{*m} \right) \in \R^{N \times d^m}$, for some natural $m \geq 2$, where $*$ refers to the Khatri-Rao product, defined in Appendix \ref{app:notation}. We have
    \begin{equation}
        \evmin{AA^\top} = \Omega(d^m),
    \end{equation}
    with probability at least $1 - \exp (-c \log^2 N)$ over $Z$, where $c$ is an absolute constant.
\end{lemma}
\begin{proof}
    As $m \geq 2$, we can write $A =  \left( Z^{*2} \right) \ast \left( Z^{*(m-2)} \right) =: A_2 \ast A_m$ (where $\left( Z^{*0} \right)$ is defined to be the vector full of ones $\mathbf{1}_{N} \in \R^N$). We can provide a lower bound on the smallest eigenvalue of such product through the following inequality \cite{schur1911}:
    \begin{equation}\label{eq:readytoschur}
        \evmin{AA^\top} = \evmin{A_2 A_2^\top \circ A_m A_m^\top} \geq \evmin{A_2 A_2^\top} \min_i \norm{(A_m)_{i:}}_2^2.
    \end{equation}

Note that the rows of $Z$ are mean-0 and Lipschitz concentrated by Lemma \ref{lemma:lipcompxy}. Then, by following the argument of Lemma C.3 in \cite{bombari2023universal}, we have
    \begin{equation}
        \evmin{A_2 A_2^\top} = \Omega(d^2),
    \end{equation}
    with probability at least $1 - \exp (-c \log^2 N)$ over $Z$. We remark that, for the argument of Lemma C.3 in \cite{bombari2023universal} to go through, it suffices that $N = o(d^2 / \log^4 d)$ and $N \log^4 N = o(d^2)$ (see Equations (C.23) and (C.26) in \cite{bombari2023universal}), which is implied by Assumption 
    \ref{ass:overparam}, despite it being milder than Assumption 4 in \cite{bombari2023universal}.

    For the second term of \eqref{eq:readytoschur}, we have
    \begin{equation}
        \norm{(A_m)_{i:}}_2^2 = \norm{z_i}_2^{2 (m-2)} = d ^{m-2},
    \end{equation}
    due to Assumption \ref{ass:datadist}. Thus, the thesis readily follows. 
\end{proof}

\begin{lemma}\label{lemma:evminRF}
We have that
\begin{equation}
    \evmin{K} = \Omega(k),
\end{equation}
with probability at least $1 - \exp \left( -c \log^2 N \right)$ over $V$ and $Z$, where $c$ is an absolute constant. This implies that $\evmin{K_{-1}} = \Omega(k)$.
\end{lemma}
\begin{proof}
    The proof follows the same path as Lemma C.5 of \cite{bombari2023universal}.    
    In particular, we define a truncated version of $\Phi$ as follows
    \begin{equation}
    \bar \Phi_{:j} = \phi(Z v_j) \chi \left(\norm{\phi(Z v_j)}_2^2 \leq R \right),
    \end{equation}
    where $\chi$ is the indicator function and we introduce the shorthand $v_i := V_{i:}$. In this case, $\chi = 1$ if $\norm{\phi(Z v_j)}^2_2 \leq R$, and $\chi = 0$ otherwise. As this is a column-wise truncation, it's easy to verify that $\Phi \Phi^\top \succeq \bar \Phi \bar \Phi^\top$.
    Over such truncated matrix, we can use Matrix Chernoff inequality (see Theorem 1.1 of \cite{Tropp2011}), which gives that $\evmin{\bar \Phi \bar \Phi^\top} = \Omega(\evmin{\bar G})$, where $\bar G := \E_V \left[\bar \Phi \bar \Phi^\top \right]$. Finally, we prove closeness between $\bar G$ and $G$, which is analogously defined as $G := \E_V \left[\Phi \Phi^\top \right]$.
    
    To be more specific, setting $R = k / \log^2 N$, we have
    \begin{equation}
        \evmin{K} \geq \evmin{\bar \Phi \bar \Phi^\top} \geq \evmin{\bar G} / 2 \geq \evmin{G} / 2 - o(k),
    \end{equation}
    where the second inequality holds with probability at least $1 - \exp(c_1 \log^2 N)$ over $V$, if $\evmin{G} = \Omega(k)$ (see Equation (C.47) of \cite{bombari2023universal}), and the third comes from Equation (C.45) in \cite{bombari2023universal}. To perform these steps, our Assumptions \ref{ass:overparam} and \ref{ass:activationfunc} are enough, despite the second one being milder than Assumption 2 in \cite{bombari2023universal}.

    To conclude the proof, we are left to prove that $\evmin{G} = \Omega(k)$ with probability at least $1 - \exp(-c_2 \log^2 N)$ over $V$ and $Z$.
    
    We have that
    \begin{equation}
        G = \E_V \left[K\right] = \E_V \left[ \sum_{i = 1}^k \phi(ZV_{i:}^\top) \phi(ZV_{i:}^\top)^\top \right] = k \E_v \left[\phi(Zv)  \phi(Zv)^\top \right] := kM,
    \end{equation}
    where we use the shorthand $v$ to indicate a random variable distributed as $V_{1:}$. We also indicate with $z_i$ the $i$-th row of $Z$.
    Exploiting the Hermite expansion of $\phi$, we can write
    \begin{equation}
        M_{ij} = \E_v \left[\phi(z_i^\top v) \phi(z_j^\top v)\right] = \sum_{l = 0}^{+ \infty} \mu_l^2 \frac{\left(z_i^\top z_j\right)^l}{d^l} = \sum_{l = 0}^{+ \infty} \mu_l^2 \frac{ \left[ \left( Z^{*l} \right)  \left( Z^{*l} \right)^\top \right]_{ij}}{d^l},
    \end{equation}
    where $\mu_l$ is the $l$-th Hermite coefficient of $\phi$. Note that the previous expansion was possible since $\norm{z_i} = \sqrt{d}$ for all $i \in [N]$. As $\phi$ is non-linear, there exists $m \geq 2$ such that $\mu_m ^2 > 0$. In particular, we have $M \succeq \frac{\mu_m^2}{d^m} AA^\top$ in a PSD sense, where we define
    \begin{equation}
        A :=  \left( Z^{*m} \right).
    \end{equation}
    By Lemma \ref{lemma:mgeq2}, the desired result readily follows.
\end{proof}

\begin{lemma}\label{lemma:PpRFon1}
Let $\mu_0\neq 0$. Then, 
    \begin{equation}
        \norm{\Ppm^\perp \mathbf{1}_k}_2 = o(\sqrt k),
    \end{equation}
with probability at least $1 - e^{-cd} - e^{-cN}$ over $V$ and $Z$, where $c$ is an absolute constant.
\end{lemma}
\begin{proof}
Note that $\Phi_{-1}^\top = \mu_0 \mathbf 1_k \mathbf 1_{N-1}^\top + \tilde \Phi_{-1}^\top$. Here, $\tilde \Phi_{-1}^\top$ is a $k \times (N-1)$ matrix with i.i.d.\ and mean-0 rows, whose sub-Gaussian norm (in the probability space of $V$) can be bounded as     \begin{equation}\label{eq:iidrows1}
        \subGnorm{\tilde \Phi_{:i}} = \subGnorm{\phi(Z V_{i:}) - \E_V [\phi(Z V_{i:})]} \leq L \frac{\opnorm{Z}}{\sqrt d} = \bigO{\sqrt{N/d} + 1},
    \end{equation}
    where first inequality holds since $\phi$ is $L$-Lipschitz and $V_{i:}$ is a Gaussian (and hence, Lipschitz concentrated) vector with covariance $I / d$. The last step holds with probability at least $1 - e^{-cd}$ over $Z$, because of Lemma B.7 in \cite{bombari2022memorization}.
    
    Thus, another application of Lemma B.7 in \cite{bombari2022memorization} gives
    \begin{equation}\label{eq:iidrows2}
        \opnorm{\tilde \Phi_{-1}^\top} = \bigO{\left(\sqrt{k} + \sqrt N \right)\left(\sqrt{N/d} + 1 \right)} = \bigO{\sqrt{k}\left(\sqrt{N/d} + 1 \right)},
    \end{equation}
    where the first equality holds with probability at least $1 - e^{-cN}$ over $V$, and the second is a direct consequence of Assumption \ref{ass:overparam}.

    We can write
    \begin{equation}
        \Phi_{-1}^\top \frac{\mathbf{1}_{N-1}}{\mu_0 (N-1)}= \left( \mu_0 \mathbf 1_k \mathbf 1_{N-1}^\top + \tilde \Phi_{-1}^\top \right) \frac{\mathbf{1}_{N-1}}{\mu_0 (N-1)} =
        \mathbf 1_k + \tilde \Phi_{-1}^\top \frac{\mathbf{1}_{N-1}}{\mu_0 (N-1)} =: \mathbf 1_k + v,
    \end{equation}
    where
    \begin{equation}
        \norm{v}_2 \leq \frac{1}{\mu_0 (N-1)} \opnorm{\tilde\Phi_{-1}^\top} \norm{\mathbf{1}_{N-1}}_2 = \bigO{\sqrt{\frac{k}{N}} \left(\sqrt{N/d} + 1 \right)} = o(\sqrt k).
    \end{equation}
    
    Thus, we can conclude
    \begin{equation}
    \begin{aligned}
        \norm{\Ppm^\perp \mathbf{1}_k}_2 &= \norm{\Ppm^\perp \left( \Phi_{-1}^\top \frac{\mathbf{1}_{N-1}}{\mu_0 (N-1)} - v \right)}_2 \\
        &\leq \norm{\Ppm^\perp \Ppm \Phi_{-1}^\top \frac{\mathbf{1}_{N-1}}{\mu_0 (N-1)}}_2 + \norm{v}_2 = o(\sqrt k),
    \end{aligned}
    \end{equation}
    where in the second step we use the triangle inequality, $\Phi_{-1}^\top = \Ppm \Phi_{-1}^\top$, and $\norm{\Ppm^\perp v}_2 \leq \norm{v}_2$.
\end{proof}

\begin{lemma}\label{lemma:Pseeslin}
    Let $z \sim \mathcal P_Z$, sampled independently from $Z_{-1}$, and denote $\tilde\phi(x):= \phi(x)-\mu_0$. Then,
    \begin{equation}
        \norm{\Ppm \left(\tilde \phi(Vz) - \mu_1 Vz \right)}_2 = o(\sqrt k),
    \end{equation}
    with probability at least $1 - \exp \left( -c \log^2 N \right)$ over $V$, $Z_{-1}$ and $z$, where $c$ is an absolute constant.
\end{lemma}
\begin{proof}
As 
$\Ppm = \Phi_{-1}^+ \Phi_{-1}$, we have
    \begin{equation}\label{eq:Plin1}
    \begin{aligned}
        \norm{\Ppm \left( \tilde \phi(Vz) - \mu_1 Vz \right)}_2 &\leq \opnorm{\Phi_{-1}^+} \norm{\Phi_{-1} \left( \tilde \phi(Vz) - \mu_1 Vz \right)}_2 \\
        &= \bigO{\frac{\norm{\Phi_{-1} \left( \tilde \phi(Vz) - \mu_1 Vz \right)}_2}{\sqrt k}},
    \end{aligned}
    \end{equation}
    where the last equality holds with probability at least $1 - \exp \left( -c \log^2 N \right)$ over $V$ and $Z_{-1}$, because of Lemma \ref{lemma:evminRF}. 
    
    An application of Lemma \ref{lemma:justBern} with $t = N$ gives
    \begin{equation}
        \left| u_i - \E_V[u_i] \right| = \bigO{\sqrt k \log N},
    \end{equation}
    where $u_i$ is the $i$-th entry of the vector $u := \Phi_{-1} \left( \tilde \phi(Vz) - \mu_1 Vz \right)$. This can be done since both $\phi$ and $\tilde \phi \equiv \phi - \mu_0$ are Lipschitz, $v_j \sim \mathcal N(0, I/d)$, and $\norm{z}_2 = \norm{z_{i+1}}_2 = \sqrt d$.
    Performing a union bound over all entries of $u$, we can guarantee that the previous equation holds for every $1 \leq i \leq N-1$, with probability at least $1 - (N-1) \exp(-c \log^2 N) \geq 1 - \exp(-c_1 \log^2 N)$. Thus, we have
    \begin{equation}\label{eq:Plinfluct}
        \norm{u - \E_V [u]}_2 = \bigO{\sqrt k \sqrt N \log N} = o(k),
    \end{equation}
    where the last equality holds because of Assumption \ref{ass:overparam}.
    
    Note that the function $f(x):=\tilde \phi(x) - \mu_1 x$ has the first 2 Hermite coefficients equal to 0. Hence, as $v_i^\top z$ and $v_i^\top z_i$ are standard Gaussian random variables with correlation $\frac{z^\top z_i}{\norm{z}_2\norm{z_i}_2}$, we have
    \begin{equation}
    \begin{aligned}
        \left| \E_{V} \left[ u_i \right] \right| \leq &\, k \sum_{l = 2}^{+ \infty} \mu_l^2 \left( \frac{ \left| z^\top z_i \right|}{\norm{z}_2\norm{z_i}_2} \right) ^l \\
        \leq & \,k \max_l \mu_l^2 \, \sum_{l = 2}^{+ \infty} \left( \frac{ \left| z^\top z_i \right| }{\norm{z}_2\norm{z_i}_2} \right)^l \\
        = & \,k \max_l \mu_l^2   \left( \frac{z^\top z_i}{\norm{z}_2\norm{z_i}_2} \right)^2 \frac{1}{1 - \frac{\left| z^\top z_i \right|}{\norm{z}_2\norm{z_i}_2}} \\
        \leq & \, 2 k \max_l \mu_l^2   \left( \frac{z^\top z_i}{\norm{z}_2\norm{z_i}_2} \right)^2 = \bigO{\frac{k \log^2 N}{d}},
    \end{aligned}
    \end{equation}
    where the last inequality holds with probability at least $1 - \exp \left( -c \log^2 N \right)$ over $z$ and $z_i$, as they are two independent, mean-0, sub-Gaussian random vectors. Again, performing a union bound over all entries of $\E_{V} \left[ u \right] $, we can guarantee that the previous equation holds for every $1 \leq i \leq N-1$, with probability at least $1 - (N-1) \exp(-c \log^2 N) \geq 1 - \exp(-c_1 \log^2 N)$. Then, we have
    \begin{equation}\label{eq:PlinE}
        \norm{\E_V [u]}_2 = \bigO{\sqrt N \, \frac{k \log^2 N}{d}} = o(k),
    \end{equation}
    where the last equality is a consequence of Assumption \ref{ass:overparam}.
    
    Finally, \eqref{eq:Plinfluct} and \eqref{eq:PlinE} give
    \begin{equation}
        \norm{\Phi_{-1} \left( \tilde \phi(Vz) - \mu_1 Vz \right)}_2 \leq \norm{\E_V [u]}_2 + \norm{u - \E_V [u]}_2 = o(k),
    \end{equation}
    which plugged in \eqref{eq:Plin1} readily provides the thesis.
\end{proof}

\begin{lemma}\label{lemma:RFlast}
We have 
    \begin{equation}
        \left| \left(V z_1^s \right)^\top \Ppm^\perp V z_1 - \norm{\Ppm^\perp V_y y_1}_2^2 \right| = o(k),
    \end{equation}
    with probability at least $1 - \exp(-c \log^2 N)$ over $x$, $z_1$ and $V$, where $c$ is an absolute constant.
\end{lemma}
\begin{proof}
    We have
    \begin{equation}
        V z_1^s = V_x x + V_y y_1, \qquad V z_1 = V_x x_1 + V_y y_1.
    \end{equation}
    Thus, we can write
    \begin{equation}\label{eq:RFlast}
    \begin{aligned}
        \left| \left(V z_1^s \right)^\top \Ppm^\perp V z_1 - \norm{\Ppm^\perp V_y y_1}_2^2 \right| &= \left| \left(V_x x \right)^\top \Ppm^\perp V z_1 + \left(V_y y_1 \right)^\top \Ppm^\perp V_x x_1 \right| \\
        & \leq \left|  x^\top  V_x^\top \Ppm^\perp V z_1 \right| + \left| y_1^\top V_y^\top \Ppm^\perp V_x x_1 \right|.
    \end{aligned}
    \end{equation}
    Let's look at the first term of the RHS of the previous equation. Notice that $\opnorm{V} = \bigO{\sqrt{k/d} + 1}$ with probability at least $1 - 2e^{-cd}$, because of Theorem 4.4.5 of \cite{vershynin2018high}. We condition on such event until the end of the proof, which also implies having the same bound on $\opnorm{V_x}$ and $\opnorm{V_y}$. Since $x$ is a mean-0 sub-Gaussian vector, independent from $V_x^\top \Ppm^\perp V z_1$, we have
    \begin{equation}\label{eq:RFlast1}
    \begin{aligned}
         \left| x^\top  V_x^\top \Ppm^\perp V z_1 \right| &\leq \log N \norm{V_x^\top \Ppm^\perp V z_1}_2 \\
         &\leq \log N \opnorm{V_x} \opnorm{\Ppm^\perp} \opnorm{V} \norm{z_1}\\
         &= \bigO{\log N \left(\frac{k}{d} + 1 \right) \sqrt{d}} = o(k),
    \end{aligned}
    \end{equation}
    where the first inequality holds with probability at least $1 - \exp(-c \log^2 N)$ over $x$, and the last line holds because $\opnorm{\Ppm^\perp} \leq 1$, $\norm{z_1} = \sqrt d$, and because of Assumption \ref{ass:overparam}.

    Similarly, exploiting the independence between $x_1$ and $y_1$, we can prove that $\left| y_1^\top V_y^\top \Ppm^\perp V_x x_1 \right|  = o(k)$, with probability at least $1 - \exp(-c \log^2 N)$ over $y_1$. Plugging this and \eqref{eq:RFlast1} in \eqref{eq:RFlast} readily gives the thesis.
\end{proof}

\begin{lemma}\label{lemma:tempfinRF}
    We have
    \begin{equation}
        \left|  \varphi(z_1^s)^\top \Ppm^\perp \varphi(z_1) - \left( k \left( \sum_{l = 2}^{+\infty} \mu_l^2 \alpha^l \right) + \mu_1^2 \norm{\Ppm^\perp V_y y_1}_2^2 \right)\right| = o(k),
    \end{equation}
    with probability at least $1 - \exp(-c \log^2 N)$ over $V$ and $Z$, where $c$ is an absolute constant.
\end{lemma}
\begin{proof}
    An application of Lemma \ref{lemma:justBern} and Assumption \ref{ass:overparam} gives
    \begin{equation}\label{eq:upboundsnorm}
    \begin{aligned}
        \norm{\varphi(z_1)}_2 &= \bigO{\sqrt k}, \qquad \norm{\varphi(z_1^s)}_2 = \bigO{\sqrt k},\\
        \norm{V z_1}_2 &= \bigO{\sqrt k}, \qquad \norm{V z_1^s}_2 = \bigO{\sqrt k},
    \end{aligned}
    \end{equation}
    with probability at least $1 - \exp(-c_1 \log^2 N)$ over $V$, where $c_1$ is an absolute constant. We condition on such high probability event until the end of the proof.

    Let's suppose $\mu_0 \neq 0$. Then, we have
    \begin{equation}\label{eq:RFfinal1}
        \left| \varphi(z_1^s)^\top \Ppm^\perp \varphi(z_1) - \tilde \phi(V z_1^s)^\top \Ppm^\perp \tilde \phi(V z_1) \right| = o(k),
    \end{equation}
    with probability at least $1 - \exp(c_2 \log^2 N)$ over $V$ and $Z$, because of \eqref{eq:upboundsnorm} and Lemma \ref{lemma:PpRFon1}. Note that \eqref{eq:RFfinal1} trivially holds even when $\mu_0 = 0$, as $\phi \equiv \tilde \phi$. Thus, \eqref{eq:RFfinal1} is true in any case with probability at least $1 - \exp(c_2 \log^2 N)$ over $V$ and $Z$.
    
    Furthermore, because of \eqref{eq:upboundsnorm} and Lemma \ref{lemma:Pseeslin}, we have
    \begin{equation}\label{eq:RFfinal2}
        \left| \tilde \phi(V z_1^s)^\top \Ppm \tilde \phi(V z_1) - \mu_1^2 (V z_1^s)^\top \Ppm (V z_1)\right| = o(k),
    \end{equation}
    with probability at least $1 - \exp(-c_3 \log^2 N)$ over $V$ and $Z$.

    Thus, putting \eqref{eq:RFfinal1} and \eqref{eq:RFfinal2} together, and using Lemma \ref{lemma:RFlast}, we get
    \begin{equation}\label{eq:laststepnumeratorrf1}
        \begin{aligned}
            &\left| \varphi(z_1^s)^\top \Ppm^\perp \varphi(z_1) - \left( \tilde \phi(V z_1^s)^\top \tilde \phi(V z_1) - \mu_1 ^2 (V z_1^s)^\top (V z_1) + \mu_1^2 \norm{\Ppm^\perp V_y y_1}_2^2 \right) \right| \\
            &\qquad\leq \left| \varphi(z_1^s)^\top \Ppm^\perp \varphi(z_1) - \tilde \phi(V z_1^s)^\top \Ppm^\perp \tilde \phi(V z_1) \right| \\
            &\qquad\quad + \left| - \tilde \phi(V z_1^s)^\top \Ppm \tilde \phi(V z_1) + \mu_1^2 (V z_1^s)^\top \Ppm (V z_1) \right| \\
            &\qquad\quad + \left|\mu_1^2 (V z_1^s)^\top \Ppm^\perp (V z_1) - \mu_1^2 \norm{\Ppm^\perp V_y y_1}_2^2 \right| 
             = o(k),
        \end{aligned}
    \end{equation}
    with probability at least $1 - \exp(-c_4 \log^2 N)$ over $V$ and $X$ and $x$.
    To conclude we apply Lemma \ref{lemma:verycorr} setting $t = N$, together with Assumption \ref{ass:overparam}
    , to get
    \begin{equation}\label{eq:laststepnumeratorrf2}
        \left| \tilde \phi(V z_1^s)^\top \tilde \phi(V z_1) - k \left( \sum_{l = 1}^{+\infty} \mu_l^2 \alpha^l \right) \right| = \bigO{\sqrt k \left( \sqrt{\frac{k}{d}} + 1 \right) \log N} = o(k),
    \end{equation}
    and
    \begin{equation}\label{eq:laststepnumeratorrf3}
        \left| \mu_1 ^2 (V z_1^s)^\top (V z_1) - k \mu_1^2 \alpha \right| = \bigO{\sqrt k \left( \sqrt{\frac{k}{d}} + 1 \right) \log N} = o(k),
    \end{equation}
    which jointly hold with probability at least  $1- \exp(- c_5 \log^2 N)$ over $V$ and $x$.
    
    Applying the triangle inequality to \eqref{eq:laststepnumeratorrf1}, \eqref{eq:laststepnumeratorrf2}, and \eqref{eq:laststepnumeratorrf3}, we get the thesis.
\end{proof}

\begin{lemma}\label{lemma:newlemma}
    We have that
    \begin{equation}
        \left|  \norm{\Ppm^\perp \varphi(z_1)}_2^2 - \E_{z_1} \left[ \norm{\Ppm^\perp \varphi(z_1)}_2^2 \right]  \right| = o(k),
    \end{equation}
    \begin{equation}
        \left| \varphi(z_1^s)^\top \Ppm^\perp \varphi(z_1) - \E_{z_1, z_1^s} \left[ \varphi(z_1^s)^\top \Ppm^\perp \varphi(z_1) \right] \right| = o(k),
    \end{equation}
    jointly hold with probability at least $1 - \exp(-c \log^2 N)$ over $z_1$, $V$ and $x$, where $c$ is an absolute constant.
\end{lemma}

\begin{proof}
    Let's condition until the end of the proof on both $\opnorm{V_x}$ and $\opnorm{V_y}$ to be $\bigO{\sqrt{k / d} + 1}$, which happens with probability at least $1 - e^{-c_1 d}$ by Theorem 4.4.5 of \cite{vershynin2018high}. This also implies that $\opnorm{V} = \bigO{\sqrt{k / d} + 1}$.

    We indicate with $\nu := \E_{z_1} \left[ \varphi (z_1) \right] =  \E_{z_1^s} \left[ \varphi (z_1^s) \right] \in \R^k$, and with $\hat \varphi(z) := \varphi (z) - \nu$. Note that, as $\varphi$ is a $C \left( \sqrt{k / d} + 1 \right)$-Lipschitz function, for some constant $C$, and as $z_1$ is Lipschitz concentrated, by Assumption
    \ref{ass:overparam}, we have
    \begin{equation}\label{eq:newlemmaconcx}
        \left| \norm{\varphi(z_1)}_2 - \E_{z_1} \left[ \norm{\varphi (z_1)}_2 \right] \right| = o \left(\sqrt{k}\right),
    \end{equation}
    with probability at least $1 - \exp(-c_2 \log^2 N)$ over $z_1$ and $V$. In addition, by the last statement of Lemma \ref{lemma:justBern} and Assumption \ref{ass:overparam}, we have that $\norm{\varphi(z_1)}_2 = \bigO{\sqrt k}$ with probability $1 - \exp(-c_3 \log^2 N)$ over $V$. Thus, taking the intersection between these two events, we have 
    \begin{equation}\label{eq:newlemmaconcexpx}
        \E_{z_1} \left[ \norm{\varphi (z_1)}_2 \right] = \bigO{\sqrt{k}},
    \end{equation}
    with probability at least $1 - \exp(-c_4 \log^2 N)$ over $z_1$ and $V$. As this statement is independent of $z_1$, it holds with the same probability just over the probability space of $V$. Then, by Jensen inequality, we have
    \begin{equation}\label{eq:normnu}
        \norm{\nu}_2 = \norm{\E_{z_1} \left[ \varphi (z_1) \right]}_2 \leq \E_{z_1} \left[ \norm{\varphi (z_1)}_2 \right] = \bigO{\sqrt k}.
    \end{equation}
    
    We can now rewrite the LHS of the first statement as
    \begin{equation}\label{eq:firsttermnewlemma}
    \begin{aligned}
    &    \left|  \norm{\Ppm^\perp \varphi(z_1)}_2^2 - \E_{z_1} \left[ \norm{\Ppm^\perp \varphi(z_1)}_2^2 \right]  \right|  \\
     &=   \left|  \norm{\Ppm^\perp \left( \hat \varphi(z_1) + \nu \right)}_2^2 - \E_{z_1} \left[ \norm{\Ppm^\perp \left( \hat \varphi(z_1) + \nu \right)}_2^2 \right]  \right|  \\
&=        \left|  \hat \varphi(z_1)^\top \Ppm^\perp \hat \varphi(z_1) +  2 \nu^\top \Ppm^\perp \hat \varphi(z_1) - \E_{z_1} \left[ \hat \varphi(z_1)^\top \Ppm^\perp \hat \varphi(z_1) \right] \right|  \\
  & \leq      \left|  \hat \varphi(z_1)^\top \Ppm^\perp \hat \varphi(z_1) - \E_{z_1} \left[ \hat \varphi(z_1)^\top \Ppm^\perp \hat \varphi(z_1) \right] \right| + 2 \left| \nu^\top \Ppm^\perp \hat \varphi(z_1)  \right|.
    \end{aligned}
    \end{equation}
    The second term is the inner product between $\hat \varphi(z_1)$, a mean-0 sub-Gaussian vector (in the probability space of $z_1$) such that $\subGnorm{\hat \varphi(z_1)} = \bigO{\sqrt{k / d} + 1}$, and the independent vector $\Ppm^\perp \nu$, such that $\norm{\Ppm^\perp \nu}_2 \leq \norm{\nu}_2 = \bigO{\sqrt k}$, because of \eqref{eq:normnu}. Thus, by Assumption
    \ref{ass:overparam}, we have that
    \begin{equation}\label{eq:linearpartnewlemma}
        \left| \nu^\top \Ppm^\perp \hat \varphi(z_1)  \right| = o(k),
    \end{equation}
    with probability at least $1 - \exp(-c_5 \log^2 N)$ over $z_1$ and $V$.
    Then, as $\left( \sqrt{k / d} + 1 \right)^{-1} \hat \varphi(z_1)$ is a mean-0, Lipschitz concentrated random vector (in the probability space of $z_1$), by the general version of the Hanson-Wright inequality given by Theorem 2.3 in \cite{HWconvex}, we can write
    \begin{equation}
    \begin{aligned}
        \P & \left( \left|  \norm{\Ppm^\perp \hat \varphi(z_1)}_2^2 - \E_{z_1} \left[ \norm{\Ppm^\perp \hat \varphi(z_1)}_2^2 \right]  \right| \geq k / \log N\right) \\
        &\leq 2 \exp \left( -c_6 \min \left( \frac{k^2}{\log^2 N \left( (k / d)^2 + 1 \right) \norm{\Ppm^\perp}^2_F} , \frac{k}{\log N \left(k / d + 1 \right) \opnorm{\Ppm^\perp}} \right) \right)  \\
        &\leq 2 \exp \left( -c_6 \min \left( \frac{k}{\log^2 N \left( (k / d)^2 + 1 \right)} , \frac{k}{\log N \left(k / d + 1 \right)} \right) \right) \\
        &\leq \exp \left( -c_7 \log^2 N \right),
    \end{aligned}
    \end{equation}
    where the last inequality comes from Assumption \ref{ass:overparam}.
    This, together with \eqref{eq:firsttermnewlemma} and \eqref{eq:linearpartnewlemma}, proves the first part of the statement.

    For the second part of the statement, we have
    \begin{equation}\label{eq:secondtermnewlemma}
        \begin{aligned}
          &      \left| \varphi(z_1^s)^\top \Ppm^\perp \varphi(z_1) - \E_{z_1, z_1^s} \left[ \varphi(z_1^s)^\top \Ppm^\perp \varphi(z_1) \right] \right|  \\
            & \leq    \left|  \hat \varphi(z_1^s)^\top \Ppm^\perp \hat \varphi(z_1) - \E_{z_1, z_1^s} \left[ \hat \varphi(z_1^s)^\top \Ppm^\perp \hat \varphi(z_1) \right] \right| + \left| \nu^\top \Ppm^\perp \hat \varphi(z_1)  \right| + \left| \nu^\top \Ppm^\perp \hat \varphi(z_1^s)  \right|.
        \end{aligned}
    \end{equation}
    Following the same argument that led to \eqref{eq:linearpartnewlemma}, we obtain
    \begin{equation}\label{eq:linearpartnewlemma2}
        \left| \nu^\top \Ppm^\perp \hat \varphi(z_1^s)  \right| = o(k),
    \end{equation}
    with probability at least $1 - \exp(-c_8 \log^2 N)$ over $z_1^s$ and $V$.
    Let us set
    \begin{equation}
        P_2 := \frac{1}{2}
        \left(\begin{array}{@{}c|c@{}}
            0 & \Ppm^\perp \\
            \hline
            \Ppm^\perp & 0
        \end{array}\right), \qquad
        V_2 :=
        \left(\begin{array}{@{}c|c|c@{}}
            V_x & V_y & 0 \\
            \hline
            0 & V_y & V_x
        \end{array}\right),
    \end{equation}
    and
    \begin{equation}
        \hat \varphi_2 := \phi \left( V_2 [x_1, y_1, x]^\top \right) - \E_{x_1, y_1, x} \left[ \phi \left( V_2 [x_1, y_1, x]^\top \right)\right] \equiv [\hat \varphi(z_1), \hat \varphi(z_1^s)]^\top.
    \end{equation}
    We have that $\opnorm{P_2} \leq 1$, $\norm{P_2}_F^2 \leq k$, $\opnorm{V_2} \leq 2\opnorm{V_x} +2\opnorm{V_y} = \bigO{\sqrt{k/d} + 1}$, and that $[x_1, y_1, x]^\top$ is a Lipschitz concentrated random vector in the joint probability space of $z_1$ and $z_1^s$, which follows from applying Lemma \ref{lemma:lipcompxy} twice. Also, we have
    \begin{equation}\label{eq:symnewlemma}
\hat         \varphi(z_1^s)^\top \Ppm^\perp \hat\varphi(z_1) = \hat \varphi_2 ^\top P_2 \hat \varphi_2.
    \end{equation}

    Thus, as $\left( \sqrt{k / d} + 1 \right)^{-1} \hat \varphi_2$ is a mean-0, Lipschitz concentrated random vector (in the probability space of $z_1$ and $z_1^s$), again by the general version of the Hanson-Wright inequality given by Theorem 2.3 in \cite{HWconvex}, we can write
    \begin{equation}
    \begin{aligned}
        \P & \left( \left|  \hat \varphi_2 ^\top P_2 \hat \varphi_2 - \E_{z_1, z_1^s} \left[ \hat \varphi_2 ^\top P_2 \hat \varphi_2 \right]  \right| \geq k / \log N\right) \\
        &\leq 2 \exp \left( -c_9 \min \left( \frac{k^2}{\log^2 N \left( (k / d)^2 + 1 \right) \norm{P_2}^2_F} , \frac{k}{\log N \left(k / d + 1 \right) \opnorm{P_2}} \right) \right)  \\
        &\leq 2 \exp \left( -c_9 \min \left( \frac{k}{\log^2 N \left( (k / d)^2 + 1 \right) } , \frac{k}{\log N \left(k / d + 1 \right)} \right) \right) \\
        &\leq \exp \left( -c_{10} \log^2 N \right),
    \end{aligned}
    \end{equation}
    where the last inequality comes from Assumption \ref{ass:overparam}. 
    This, together with \eqref{eq:secondtermnewlemma}, \eqref{eq:linearpartnewlemma}, \eqref{eq:linearpartnewlemma2}, and \eqref{eq:symnewlemma}, proves the second part of the statement, and therefore the desired result.

\end{proof}

Finally, we are ready to give the proof of Theorem \ref{thm:RF}.

\begin{proof}[Proof of Theorem \ref{thm:RF}]
    We will prove the statement for the following definition of $\gamma_{\textup{RF}}$, independent from $z_1$ and $z_1^s$,
    \begin{equation}
        \gamma_{\textup{RF}} := \frac{\E_{z_1, z_1^s} \left[ \varphi(z_1^s)^\top \Ppm^\perp \varphi(z_1) \right]}{\E_{z_1} \left[ \norm{\Ppm^\perp \varphi(z_1)}_2^2 \right]}.
    \end{equation}
    
    By Lemma \ref{lemma:dendiv0} and \ref{lemma:evminRF}, we have
    \begin{equation}\label{eq:expectrf2}
        \norm{\Ppm^\perp \varphi(z)}_2^2 = \Omega(k)
    \end{equation}
    with probability at least $1 - \exp(-c_1 \log^2 N)$ over $V$, $Z_{-1}$ and $z$. This, together with Lemma \ref{lemma:newlemma}, gives
    \begin{equation}\label{eq:firststatementrf}
        \left| \frac{\varphi(z_1^s)^\top \Ppm^\perp \varphi(z_1)}{\norm{\Ppm^\perp \varphi(z_1)}_2^2} - \frac{\E_{z_1, z_1^s} \left[ \varphi(z_1^s)^\top \Ppm^\perp \varphi(z_1) \right]}{\E_{z_1} \left[ \norm{\Ppm^\perp \varphi(z_1)}_2^2 \right]}\right| = o(1),
    \end{equation}
    with probability at least $1 - \exp(-c_2 \log^2 N)$ over $V$, $Z$ and $x$, which proves the first part of the statement.

    The upper-bound on $\gamma_{\textup{RF}}$ can be obtained applying Cauchy-Schwarz twice
    \begin{equation}\label{eq:firstubrf}
    \begin{aligned}
        \frac{\E_{z_1, z_1^s} \left[ \varphi(z_1^s)^\top \Ppm^\perp \varphi(z_1) \right]}{\E_{z_1} \left[ \norm{\Ppm^\perp \varphi(z_1)}_2^2\right]} &\leq \frac{\E_{z_1, z_1^s} \left[\norm{\Ppm^\perp \varphi(z_1^s)}_2 \norm{\Ppm^\perp \varphi(z_1)}_2 \right]}{\E_{z_1} \left[ \norm{\Ppm^\perp \varphi(z_1)}_2^2 \right]}\\
        &\leq \frac{\sqrt{\E_{z_1^s} \left[\norm{\Ppm^\perp \varphi(z_1^s)}_2^2 \right]} \sqrt{\E_{z_1} \left[ \norm{\Ppm^\perp \varphi(z_1)}^2_2 \right]}}{\E_{z_1} \left[ \norm{\Ppm^\perp \varphi(z_1)}_2^2 \right]} = 1.
    \end{aligned}
    \end{equation}

    Let's now focus on the lower bound. By Assumption
    \ref{ass:overparam} and Lemma \ref{lemma:verycorr} (in which we consider the degenerate case $\alpha = 1$ and set $t=N$), we have
    \begin{equation}\label{eq:hermitelemmarf1}
        \left| \norm{\tilde \phi \left( V z_1 \right)}_2^2 - k \sum_{l = 1}^{+\infty} {\mu^2_l} \right| = o(k),
    \end{equation}
    with probability at least $1 - \exp(-c_3 \log^2 N)$ over $V$ and $z_1$. Then, a few applications of the triangle inequality give
    \begin{equation}
    \begin{aligned}
        \frac{\E_{z_1, z_1^s} \left[ \varphi(z_1^s)^\top \Ppm^\perp \varphi(z_1) \right]}{\E_{z_1} \left[ \norm{\Ppm^\perp \varphi(z_1)}_2^2 \right]} &\geq \frac{\varphi(z_1^s)^\top \Ppm^\perp \varphi(z_1)}{\norm{\Ppm^\perp \varphi(z_1)}_2^2}-o(1) \\
        &\geq \frac{\varphi(z_1^s)^\top \Ppm^\perp \varphi(z_1)}{\norm{\Ppm^\perp \tilde \varphi(z_1)}_2^2} - o(1) \\
        & \geq \frac{k \left( \sum_{l = 2}^{+\infty} \mu_l^2 \alpha^l \right) + \mu_1^2\norm{\Ppm^\perp V_y y_1}_2^2 }{\norm{\tilde \varphi(z_1)}_2^2} - o(1) \\
        &\geq \frac{k \left( \sum_{l = 2}^{+\infty} \mu_l^2 \alpha^l \right) + \mu_1^2\norm{\Ppm^\perp V_y y_1}_2^2 }{k \sum_{l = 1}^{+\infty} \mu_l^2} - o(1) \\
        & \geq \frac{\sum_{l = 2}^{+\infty} \mu_l^2 \alpha^l}{\sum_{l = 1}^{+\infty} \mu_l^2} - o(1),
    \end{aligned}
    \end{equation}
    where the first inequality is a consequence of \eqref{eq:firststatementrf}, the second of Lemma \ref{lemma:PpRFon1} and \eqref{eq:expectrf2}, the third of Lemma \ref{lemma:tempfinRF} and again \eqref{eq:expectrf2}, and the fourth of \eqref{eq:hermitelemmarf1}, and they jointly hold with probability $1 - \exp(-c_4 \log^2 N)$ over $V$, $Z_{-1}$ and $z_1$. Again, as the statement does not depend on $z_1$, we can conclude that it holds with the same probability only over the probability spaces of $V$ and $Z_{-1}$, and the thesis readily follows.
\end{proof}

\section{Proofs for NTK features}\label{app:NTK}

In this section, we will indicate with $Z \in \R^{N \times d}$ the data matrix, such that its rows are sampled independently from $\mathcal P_Z$ (see Assumption \ref{ass:datadist}). We denote by $W \in \R^{k \times d}$ the weight matrix at initialization, such that $W_{ij} \distas{}_{\rm i.i.d.}\mathcal{N}(0, 1/d)$. Thus, the feature map is given by (see \eqref{eq:NTKmodel})
\begin{equation}
    \varphi(z) := z \otimes \phi'(W z) \in \R^{dk},
\end{equation}
where $\phi'$ is the derivative of the activation function $\phi$, applied component-wise to the vector $Wz$. We use the shorthands $\Phi := Z * \phi'(ZW^\top) \in \R^{N \times p}$ and $K := \Phi \Phi^\top \in \R^{N \times N}$, where $*$ denotes the Khatri-Rao product, defined in Appendix \ref{app:notation}. We indicate with $\Phi_{-1} \in \R^{(N - 1) \times k}$ the matrix $\Phi$ without the first row, and we define $K_{-1} := \Phi_{-1} \Phi_{-1}^\top$. We call $\Pp$ the projector over the span of the rows of $\Phi$, and $\Ppm$ the projector over the span of the rows of $\Phi_{-1}$. We use the notations $\tilde \varphi(z) := \varphi(z) - \E_W [\varphi(z)]$ and $\tilde \Phi_{-1} := \Phi_{-1} - \E_W [\Phi_{-1}]$ to indicate the centered feature map and matrix respectively, where the centering is with respect to $W$. We indicate with $\mu'_l$ the $l$-th Hermite coefficient of $\phi'$. We use the notation $z_1^s = [x, y_1]$, where $x \sim \mathcal P_X$ is sampled independently from $V$ and $Z$.  We define $\alpha = d_y/d$.
Throughout this section, for compactness, we drop the subscripts \enquote{NTK} from these quantities, as we will only treat the proofs related to Section \ref{sec:ntk}. Again for the sake of compactness, we will not re-introduce such quantities in the statements or the proofs of the following lemmas.

The content of this section can be summarized as follows:
\begin{itemize}
    \item In Lemma \ref{lemma:evminntk}, we prove the lower bound on the smallest eigenvalue of $K$, adapting to our settings the main result of \cite{bombari2022memorization}.
    \item In Lemma \ref{lemma:Ppntkon1}, we treat separately a term that derives from $\E_W \left[ \phi'(W z) \right] = \mu'_0 \mathbf{1}_k$, showing that we can \emph{center} the derivative of the activation function (Lemma \ref{lemma:ntkthirdlaststep}), without changing our final statement in Theorem \ref{thm:mainntk}. This step is necessary only if $\mu'_0 \neq 0$. Our proof tackles the problem proving the thesis on a set of ``perturbed'' inputs $\bar Z_{-1}(\delta)$ (Lemma \ref{lemma:Ppntkon1pert}), critically exploiting the non degenerate behaviour of their rows (Lemma \ref{lemma:invertible}), and transfers the result on the original term, using continuity arguments with respect to the perturbation (Lemma \ref{lemma:continuous}).
    \item In Lemma \ref{lemma:ntkpenultimatestep}, we show that the centered features $\tilde \varphi(z_1)$ and $\tilde \varphi(z_1^s)$ have a negligible component in the space spanned by the rows of $\Phi_{-1}$. To achieve this, we exploit the bound proved in Lemma \ref{lemma:Pntksmall}.
    \item To conclude, we prove Theorem \ref{thm:mainntk}, exploiting also the concentration result provided in Lemma \ref{lemma:ntklaststep}.
\end{itemize}

\begin{lemma}\label{lemma:evminntk}
We have that
\begin{equation}
    \evmin{K} = \Omega(kd),
\end{equation}
with probability at least $1 - N e^{-c \log^2 k} - e^{-c \log^2 N}$ over $Z$ and $W$, where $c$ is an absolute constant.
\end{lemma}
\begin{proof}
The result follows from Theorem 3.1 of \cite{bombari2022memorization}. Notice that our assumptions on the data distribution $\mathcal P_Z$ are stronger, and that our initialization of the very last layer (which differs from the Gaussian initialization in \cite{bombari2022memorization}) does not change the result. Assumption 
\ref{ass:overparamntk}, \emph{i.e.}, $k = \bigO{d}$, satisfies the \textit{loose pyramidal topology} condition (cf. Assumption 2.4 in \cite{bombari2022memorization}), and Assumption \ref{ass:overparamntk} is the same as Assumption 2.5 in \cite{bombari2022memorization}. An important difference is that we do not assume the activation function $\phi$ to be Lipschitz anymore. This, however, stops being a necessary assumption since we are working with a 2-layer neural network, and $\phi$ doesn't appear in the expression of NTK. 
\end{proof}

\begin{lemma}\label{lemma:continuous}
    Let $A \in \R^{(N-1) \times d}$ be a generic matrix, and let $\bar Z_{-1}(\delta)$ and $\bar \Phi_{-1} (\delta)$ be defined as
    \begin{equation}
        \bar Z_{-1}(\delta) := Z_{-1} + \delta A,
    \end{equation}
    \begin{equation}
        \bar \Phi_{-1} (\delta) := \bar Z_{-1}(\delta) * \phi' \left(Z_{-1} W^\top \right).
    \end{equation}
    Let $\bar P_{\Phi_{-1}}(\delta) \in \R^{dk \times dk}$ be the projector over the $\Span$ of the rows of $\bar \Phi_{-1} (\delta)$.
    Then, we have that $\bar P^\perp_{\Phi_{-1}}(\delta)$ is continuous in $\delta = 0$ with probability at least $1 - N e^{-c \log^2 k} - e^{-c \log^2 N}$ over $Z$ and $W$, where $c$ is an absolute constant and where the continuity is with respect to $\opnorm{\cdot}$.
\end{lemma}
\begin{proof}
    In this proof, when we say that a matrix is continuous with respect to $\delta$, we always intend with respect to the operator norm $\opnorm{\cdot}$. Then, $\bar \Phi_{-1} (\delta)$ is continuous in $0$, as
    \begin{equation}
        \opnorm{\bar \Phi_{-1} (\delta) - \bar \Phi_{-1} (0)} = \opnorm{ \delta A * \phi' \left(Z_{-1} W^\top \right)} \leq \delta \opnorm{A} \max_{2 \leq i \leq N} \norm{\phi' \left(W z_i \right)}_2,
    \end{equation}
    where the second step follows from Equation (3.7.13) in \cite{johnson1990matrix}.

    By Weyl's inequality, this also implies that $\evmin{\bar \Phi_{-1} (\delta) \bar \Phi_{-1} (\delta)^\top}$ is continuous in $\delta = 0$. Recall that, by Lemma \ref{lemma:evminntk}, $\det \left( \bar \Phi_{-1} (0) \bar \Phi_{-1} (0)^\top \right) \equiv \det \left(\Phi_{-1} \Phi_{-1} ^\top \right) \neq 0$ with probability at least $1 - N e^{-c \log^2 k} - e^{-c \log^2 N}$ over $Z$ and $W$. This implies that $\left( \bar \Phi_{-1} (\delta) \bar \Phi_{-1} (\delta) ^\top \right)^{-1}$ is also continuous, as for every invertible matrix $M$ we have $M^{-1} = \textup{Adj}(M) / \det(M)$ (where $\textup{Adj}(M)$ denotes the Adjugate of the matrix $M$), and both $\textup{Adj}(\cdot)$ and $\det (\cdot)$ are continuous mappings.    Thus, as $\bar P_{\Phi_{-1}}(0) = \bar \Phi_{-1} (0) ^\top \left( \bar \Phi_{-1} (0) \bar \Phi_{-1} (0) ^\top \right)^{-1} \bar \Phi_{-1} (0) $ (see \eqref{eq:projform}), we also have the continuity of $\bar P_{\Phi_{-1}}(\delta)$ in $\delta = 0$, which gives the thesis.
\end{proof}

\begin{lemma}\label{lemma:invertible}
    Let $A \in \R^{(N-1) \times d}$ be a matrix with entries sampled independently (between each other and from everything else) from a standard Gaussian distribution. Then, for every $\delta > 0$, with probability 1 over $A$, the rows of $\bar Z_{-1} := Z_{-1} + \delta A$ span $\R ^d$.
\end{lemma}
\begin{proof}
    As $N-1 \geq d$, by Assumption \ref{ass:overparamntk},
    negating the thesis would imply that the rows of $\bar Z_{-1}$ are linearly dependent, and that they belong to a subspace with dimension at most $d - 1$. This would imply that there exists a row of $\bar Z_{-1}$, call it $\bar z_j$, such that $\bar z_j$ belongs to the space spanned by all the other rows of $\bar Z_{-1}$, with dimension at most $d - 1$. This means that $A_{j:}$ has to belong to an affine space with the same dimension, which we can consider fixed, as it's not a function of the random vector $A_{j:}$, but only of $Z_{-1}$ and $\{ A_{i:} \}_{i \neq j}$. As the entries of $A_{j:}$ are sampled independently from a standard Gaussian distribution, this happens with probability 0.
\end{proof}

\begin{lemma}\label{lemma:Ppntkon1pert}
    Let $A \in \R^{(N-1) \times d}$ be a matrix with entries sampled independently (between each other and from everything else) from a standard Gaussian distribution. Let $\bar Z_{-1}(\delta) := Z_{-1} + \delta A$ and $\bar \Phi_{-1} (\delta) := \bar Z_{-1}(\delta) * \phi' \left(Z_{-1} W^\top \right)$. Let $\bar P_{\Phi_{-1}}(\delta) \in \R^{dk \times dk}$ be the projector over the $\Span$ of the rows of $\bar \Phi_{-1} (\delta)$. Let $\mu'_0 \neq 0$. Then, for $z \sim \mathcal P_Z$, and for any $\delta > 0$, we have
    \begin{equation}
        \norm{\bar P^\perp_{\Phi_{-1}}(\delta) \left(z \otimes \mathbf{1}_k\right)}_2 = o(\sqrt {dk}),
    \end{equation}
    with probability at least $1 - \exp(- c \log^2 N)$ over $Z$, $W$, and $A$, where $c$ is an absolute constant.
\end{lemma}
\begin{proof}
    Let $B_{-1} := \phi'(Z_{-1} W^\top) \in \R^{(N-1) \times k}$. Notice that, for any $\zeta \in \R^{N - 1}$, the following identity holds
    \begin{equation}\label{eq:tensoridentity}
        \bar \Phi_{-1}^\top(\delta) \zeta = \left( \bar Z_{-1} (\delta) \ast B_{-1} \right)^\top \zeta = \left(\bar Z_{-1}^\top (\delta) \zeta \right) \otimes \left(B_{-1}^\top \mathbf 1_{N-1} \right).
    \end{equation}
    Note that $B_{-1}^\top = \mu'_0 \mathbf 1_k \mathbf 1_{N-1}^\top + \tilde B_{-1}^\top$, where $\tilde B_{-1}^\top = \phi'(W Z_{-1}^\top) - \E_W \left[ \phi'(W Z_{-1}^\top) \right]$ is a $k \times (N - 1)$ matrix with i.i.d. and mean-0 rows. For an argument equivalent to the one used for \eqref{eq:iidrows1} and \eqref{eq:iidrows2}, we have
    \begin{equation}
        \opnorm{\tilde B_{-1}^\top} = \bigO{\left( \sqrt k + \sqrt N \right) \left( \sqrt{N / d} + 1 \right)},
    \end{equation}
    with probability at least $1 - \exp(- c \log^2 N)$ over $Z_{-1}$ and $W$.
    Thus, we can write
    \begin{equation}\label{eq:splitBntk}
        B_{-1}^\top \frac{\mathbf{1}_{N-1}}{\mu'_0 (N-1)}= \left( \mu'_0 \mathbf 1_k \mathbf 1_{N-1}^\top + \tilde B_{-1}^\top \right) \frac{\mathbf{1}_{N-1}}{\mu'_0 (N-1)} =
        \mathbf 1_k + \tilde B_{-1}^\top \frac{\mathbf{1}_{N-1}}{\mu'_0 (N-1)} =: \mathbf 1_k + v,
    \end{equation}
    where we have
    \begin{equation}\label{eq:vissmallntk}
        \norm{v}_2 \leq \opnorm{\tilde B_{-1}^\top} \norm{\frac{\mathbf{1}_{N-1}}{\mu'_0 (N-1)}}_2 = \bigO{\left( \sqrt{k /N} + 1 \right) \left( \sqrt{N / d} + 1 \right)} = o(\sqrt k),
    \end{equation}
    where the last step is a consequence of Assumption \ref{ass:overparamntk}.
    Plugging \eqref{eq:splitBntk} in \eqref{eq:tensoridentity} we get
    \begin{equation}
        \frac{1}{\mu'_0 (N-1)} \bar \Phi_{-1}^\top(\delta) \zeta = \frac{1}{\mu'_0 (N-1)} \left(\bar Z_{-1}(\delta) \ast B_{-1} \right)^\top \zeta = \left(\bar Z_{-1}^\top(\delta) \zeta \right) \otimes \left( \mathbf 1_k + v \right) .
    \end{equation}
    
    By Lemma \ref{lemma:invertible}, we have that the rows of $\bar Z_{-1}(\delta)$ span $\R^d$, with probability 1 over $A$. Thus, conditioning on this event, we can set $\zeta$ to be a vector such that $z = \bar Z_{-1}^\top(\delta) \zeta$. We can therefore rewrite the previous equation as
    \begin{equation}
        \frac{1}{\mu'_0 (N-1)}  \bar \Phi_{-1}^\top(\delta)  \zeta = z \otimes \mathbf 1_k + z \otimes v.
    \end{equation}
    
    Thus, we can conclude
    \begin{equation}
    \begin{aligned}
        \norm{\bar P^\perp_{\Phi_{-1}}(\delta) \left(z \otimes \mathbf{1}_k\right)}_2 &= \norm{\Ppm^\perp \left( \frac{\Phi_{-1}^\top(\delta) \zeta}{\mu'_0 (N-1)}  - z \otimes v \right)}_2 \\
        &\leq \norm{\bar P^\perp_{\Phi_{-1}}(\delta) \Phi_{-1}^\top(\delta) \frac{\zeta}{\mu'_0 (N - 1)}}_2 + \norm{z \otimes v}_2 \\
        &= \norm{z}_2 \norm{v}_2
        = o(\sqrt{dk}),   
    \end{aligned}
    \end{equation}
    where in the second step we use the triangle inequality, in the third step we use that $\Phi_{-1}^\top(\delta) = \bar P_{\Phi_{-1}}(\delta) \Phi_{-1}^\top(\delta)$, and in the last step we use \eqref{eq:vissmallntk}. The desired result readily follows.
\end{proof}

\begin{lemma}\label{lemma:Ppntkon1}
    Let $\mu'_0 \neq 0$. Then, for any $z \in \R^d$, we have
    \begin{equation}
        \norm{\Ppm^\perp \left(z \otimes \mathbf{1}_k\right)}_2 = o(\sqrt {dk}),
    \end{equation}
    with probability at least $1 - N e^{-c \log^2 k} - e^{-c \log^2 N}$ over $Z$ and $W$, where $c$ is an absolute constant.
\end{lemma}
\begin{proof}
    Let $A \in \R^{(N-1) \times d}$ be a matrix with entries sampled independently (between each other and from everything else) from a standard Gaussian distribution. Let $\bar Z_{-1}(\delta) := Z_{-1} + \delta A$ and $\bar \Phi_{-1} (\delta) := \bar Z_{-1}(\delta) * \phi' \left(Z_{-1} W^\top \right)$. Let $\bar P_{\Phi_{-1}}(\delta) \in \R^{dk \times dk}$ be the projector over the $\Span$ of the rows of $\bar \Phi_{-1} (\delta)$.
    
    By triangle inequality, we can write
    \begin{equation}\label{eq:usingcontinuity}
        \norm{\Ppm^\perp \left(z \otimes \mathbf{1}_k\right)}_2 \leq \opnorm{\Ppm^\perp - \bar P_{\Phi_{-1}}^\perp (\delta)} \norm{z \otimes \mathbf{1}_k}_2 + \norm{\bar P_{\Phi_{-1}}^\perp (\delta) \left(z \otimes \mathbf{1}_k\right)}_2.
    \end{equation}
    Because of Lemma \ref{lemma:continuous}, with probability at least $1 - N e^{-c \log^2 k} - e^{-c \log^2 N}$ over $Z$ and $W$, $\bar P^\perp_{\Phi_{-1}}(\delta)$ is continuous in $\delta = 0$, with respect to $\opnorm{\cdot}$. Thus, there exists $\delta^* > 0$ such that, for every $\delta \in [0, \delta^*]$,
    \begin{equation}
        \opnorm{\Ppm^\perp - \bar P_{\Phi_{-1}}^\perp (\delta)} \equiv \opnorm{\bar P_{\Phi_{-1}}^\perp (0) - \bar P_{\Phi_{-1}}^\perp (\delta)} < \frac{1}{N}.
    \end{equation}
    Hence, setting $\delta = \delta^*$ in \eqref{eq:usingcontinuity}, we get
    \begin{equation}
    \begin{aligned}
        \norm{\Ppm^\perp \left(z \otimes \mathbf{1}_k\right)}_2 &\leq \opnorm{\Ppm^\perp - \bar P_{\Phi_{-1}}^\perp (\delta^*)} \norm{z \otimes \mathbf{1}_k}_2 + \norm{\bar P_{\Phi_{-1}}^\perp (\delta^*) \left(z \otimes \mathbf{1}_k\right)}_2 \\
        &\leq \norm{z}_2 \norm{\mathbf{1}_k}_2 / N + \norm{\bar P_{\Phi_{-1}}^\perp (\delta^*) \left(z \otimes \mathbf{1}_k\right)}_2 \\
        & = o(\sqrt{dk}),
    \end{aligned}
    \end{equation}
    where the last step is a consequence of Lemma \ref{lemma:Ppntkon1pert}, and it holds with probability at least $1 - \exp(- c \log^2 N)$ over $Z$, $W$, and $A$. As the LHS of the previous equation doesn't depend on $A$, the statements holds with the same probability, just over the probability spaces of $Z$ and $W$, which gives the desired result.
\end{proof}

\begin{lemma}\label{lemma:ntklaststep}
    We have
    \begin{equation}\label{eq:1st}
        \left| \frac{\tilde \varphi(z_1^s)^\top  \tilde \varphi(z_1)}{\norm{\tilde \varphi(z_1)}_2^2} - \alpha \, \frac{ \sum_{l = 1}^{+\infty} {\mu'_l}^2 \alpha^i }{\sum_{l = 1}^{+\infty} {\mu'_l}^2} \right| = o(1),
    \end{equation}
    with probability at least $1 - \exp(-c \log^2 N) - \exp(-c \log^2 k)$ over $W$ and $z_1$, where $c$ is an absolute constant. With the same probability, we also have
    \begin{equation}
        \tilde \varphi(z_1^s)^\top  \tilde \varphi(z_1) = \Theta(dk), \qquad \norm{\tilde \varphi(z_1)}_2^2 = \Theta(dk).
    \end{equation}
\end{lemma}
\begin{proof}
    We have
    \begin{equation}
        \norm{\tilde \varphi(z_1)}_2^2 = \norm{z_1 \otimes \tilde \phi' \left( W z_1 \right)}_2^2 
        = \norm{z_1}_2^2 \norm{\tilde \phi' \left( W z_1 \right)}_2^2 
        = d \norm{\tilde \phi' \left( W z_1 \right)}_2^2.
    \end{equation}
    By Assumption 
    \ref{ass:overparamntk} and Lemma \ref{lemma:verycorr} (in which we consider the degenerate case $\alpha = 1$ and set $t=k$), we have
    \begin{equation}\label{eq:hermitelemmantk1}
        \left| \norm{\tilde \phi' \left( W z_1 \right)}_2^2 - k \sum_{l = 1}^{+\infty} {\mu'_l}^2 \right| = o(k),
    \end{equation}
    with probability at least $1 - \exp(-c \log^2 k)$ over $W$ and $z_1$. Thus, we have
    \begin{equation}\label{eq:ntkdenominator11}
        \left| \norm{\tilde \varphi(z_1)}_2^2 - dk \sum_{l = 1}^{+\infty} {\mu'_l}^2\right| = o(dk).
    \end{equation}
    Notice that the second term in the modulus is $\Theta(dk)$, since the $\mu'_l$-s cannot be all 0, because of Assumption \ref{ass:activationfuncntk}; this shows that $\norm{\tilde \varphi(z_1)}_2^2 = \Theta(dk)$.

    Similarly, we can write
    \begin{equation}\label{eq:ntknumerator1}
         \tilde \varphi(z_1^s)^\top  \tilde \varphi(z_1) = \left( z_1^\top z_1^s \right) \left( \tilde \phi' \left( W z_1 \right) ^\top \tilde \phi' \left( W z_1^s \right) \right).
    \end{equation}
    We have
    \begin{equation}\label{eq:ntknumerator2}
        \left| z_1^\top z_1^s - \alpha d\right| = \left| x_1^\top x \right| \leq \sqrt{d_x} \log d = o(d),
    \end{equation}
    where the inequality holds with probability at least $1 - \exp(-c_1 \log^2 d) \geq 1 - \exp(-c_2 \log^2 N)$ over $x_1$, as we are taking the inner product of two independent and sub-Gaussian vectors with norm $\sqrt {d_x}$. Furthermore, again by Assumption
    \ref{ass:overparamntk} and Lemma \ref{lemma:verycorr}, we have
    \begin{equation}\label{eq:ntknumerator3}
        \left| \tilde \phi' \left( W z_1 \right) ^\top \tilde \phi' \left( W z_1^s \right) - k \sum_{l = 1}^{+\infty} {\mu'_l}^2 \alpha^l \right| = o(k),
    \end{equation}
    with probability at least $1 - \exp(-c_3 \log^2 k)$ over $W$ and $z_1$. Notice that the second term in the modulus is $\Theta(k)$, because of Assumption \ref{ass:activationfuncntk}.

    Thus, putting \eqref{eq:ntknumerator1}, \eqref{eq:ntknumerator2} and \eqref{eq:ntknumerator3} together, we get
    \begin{equation}\label{eq:ntknumerator4}
         \left| \tilde \varphi(z_1^s)^\top  \tilde \varphi(z_1) - dk \alpha \sum_{l = 1}^{+\infty} {\mu'_l}^2 \alpha^l \right| = o(dk),
    \end{equation}
    with probability at least $1 - \exp(-c_3 \log^2 k) - \exp(-c_2 \log^2 N)$ over $W$ and $z_1$; this shows that $\tilde \varphi(z_1^s)^\top \tilde \varphi(z_1) = \Theta(dk)$.

    Finally, merging \eqref{eq:ntknumerator4} with \eqref{eq:ntkdenominator11} and applying triangle inequality, \eqref{eq:1st} follows and the proof is complete.
\end{proof}

\begin{lemma}\label{lemma:Pntksmall}
    Let $z \sim \mathcal P_Z$ be sampled independently from $Z_{-1}$. Then,
    \begin{equation}
        \norm{\Phi_{-1} \tilde \varphi(z)}_2 = o\left(dk\right),
    \end{equation}
    with probability at least $1 - \exp(-c \log^2 N)$ over $W$ and $z$, where $c$ is an absolute constant.
\end{lemma}
\begin{proof}
    Let's look at the $i$-th entry of the vector $\Phi_{-1} \tilde \varphi(z)$, \emph{i.e.},
    \begin{equation}\label{eq:Pntksmall0}
        \varphi(z_{i + 1})^\top \tilde \varphi(z) = \left(z_{i+1}^\top z \right) \left( \phi'(Wz_{i+1})^\top \tilde \phi'(Wz) \right).
    \end{equation}

    As $z$ and $z_{i+1}$ are sub-Gaussian and independent with norm $\sqrt d$, we can write $\left| z^\top z_{i+1} \right| = \bigO{\sqrt d \log N}$ with probability at least $1 - \exp(-c \log^2 N)$ over $z$. We will condition on such high probability event until the end of the proof.
    
    By Lemma \ref{lemma:justBern}, setting $t = N$, we have
    \begin{equation}\label{eq:Pntksmall1}
        \left|  \phi'(Wz_{i+1})^\top \tilde \phi'(Wz) - \E_W \left[\phi'(Wz_{i+1})^\top \tilde \phi'(Wz) \right]\right| = \bigO{\sqrt k \log N},
    \end{equation}
    with probability at least $1 - \exp(-c_1 \log^2 N)$ over $W$.
    Exploiting the Hermite expansion of $\phi'$ and $\tilde \phi'$, we have
    \begin{equation}\label{eq:Pntksmall2}
    \begin{aligned}
        \left| \E_W \left[\phi'(Wz_{i+1})^\top \tilde \phi'(Wz) \right] \right| \leq & \, k \sum_{l = 1}^{+ \infty} {\mu'_l}^2 \left( \frac{ \left| z_{i + 1}^\top z \right|}{\norm{z_{i + 1}}_2\norm{z}_2} \right) ^l \\
        \leq & \,k \max_l {\mu'_l}^2 \, \sum_{l = 1}^{+ \infty} \left( \frac{ \left| z_{i + 1}^\top z \right| }{\norm{z_{i + 1}}_2\norm{z}_2} \right)^l \\
        = & \,k \max_l {\mu'_l}^2   \frac{ \left| z_{i + 1}^\top z \right|}{\norm{z_{i + 1}}_2\norm{z}_2} \frac{1}{1 - \frac{\left| z_{i + 1}^\top z \right|}{\norm{z_{i + 1}}_2\norm{z}_2}} \\
        \leq & \, 2 k \max_l {\mu'_l}^2  \frac{\left| z_{i + 1}^\top z \right|}{\norm{z_{i + 1}}_2\norm{z}_2}  = \bigO{\frac{k \log N}{\sqrt d}}.
    \end{aligned}
    \end{equation}

    Putting together \eqref{eq:Pntksmall1} and \eqref{eq:Pntksmall2}, and applying triangle inequality, we get
    \begin{equation}
         \left|  \phi'(Wz_{i+1})^\top \tilde \phi'(Wz) \right| = \bigO{\sqrt k \log N + \frac{k \log N}{\sqrt d}} = \bigO{\sqrt k \log N},
    \end{equation}
    where the last step is a consequence of Assumption \ref{ass:overparamntk}.
    Comparing this last result with \eqref{eq:Pntksmall0}, we obtain
    \begin{equation}
        \left| \varphi(z_{i + 1})^\top \tilde \varphi(z) \right| = \bigO{\sqrt {dk} \log^2 N},
    \end{equation}
    with probability at least $1 - \exp(-c_2 \log^2 N)$ over $W$ and $z$.

    We want the previous equation to hold for all $1 \leq i \leq N-1$. Performing a union bound, we have that this is true with probability at least  $1 - (N - 1) \exp(-c_2 \log^2 N) \geq 1 - \exp(-c_3 \log^2 N)$ over $W$ and $z$. Thus, with such probability, we have
    \begin{equation}
    \begin{aligned}
        \norm{\Phi_{-1} \tilde \varphi(z)}_2 \leq& \sqrt{N - 1} \max_{i} \left| \varphi(z_{i + 1})^\top \tilde \varphi(z) \right| \\
        =& \bigO{\sqrt {dk} \sqrt N \log^2 N} = o(dk),
    \end{aligned}
    \end{equation}
    where the last step follows from Assumption \ref{ass:overparamntk}.
\end{proof}

\begin{lemma}\label{lemma:ntkpenultimatestep}
    We have
    \begin{equation}
        \left| \frac{\tilde \varphi(z_1^s)^\top  \tilde \varphi(z_1) -\tilde \varphi(z_1^s)^\top \Ppm \tilde \varphi(z_1)}{\norm{\tilde \varphi(z_1) - \Ppm \tilde \varphi(z_1)}_2^2} - \frac{\tilde \varphi(z_1^s)^\top  \tilde \varphi(z_1)}{\norm{\tilde \varphi(z_1)}_2^2} \right| = o(1),
    \end{equation}
    with probability at least $1 - N \exp(-c \log^2 k) - \exp(-c \log^2 N)$ over $Z$, $x$ and $W$, where $c$ is an absolute constant. With the same probability, we also have
    \begin{equation}
        \tilde \varphi(z_1^s)^\top  \tilde \varphi(z_1) -\tilde \varphi(z_1^s)^\top \Ppm \tilde \varphi(z_1) = \Theta(dk), \qquad \norm{\tilde \varphi(z_1) - \Ppm \tilde \varphi(z_1)}_2^2 = \Theta(dk).
    \end{equation}
\end{lemma}
\begin{proof}
    Notice that, with probability at least $1 - \exp(-c \log^2 N) - \exp(-c \log^2 k)$ over $W$ and $z_1$, we have both
    \begin{equation}\label{eq:bigtermspenultimatentk}
        \tilde \varphi(z_1^s)^\top  \tilde \varphi(z_1) = \Theta(dk) \qquad \norm{\tilde \varphi(z_1)}_2^2 = \Theta(dk).
    \end{equation}
    by the second statement of Lemma \ref{lemma:ntklaststep}. Furthermore,
%
    \begin{equation}\label{eq:smallcrossPnumer}
    \begin{aligned}
        \left| \tilde \varphi(z_1^s)^\top \Ppm \tilde \varphi(z_1) \right| =& \left| \tilde \varphi(z_1^s)^\top \Phi_{-1}^\top K^{-1}_{-1} \Phi_{-1} \tilde \varphi(z_1) \right| \\
        \leq& \norm{\Phi_{-1}\tilde \varphi(z_1^s)}_2 \evmin{K_{-1}}^{-1} \norm{\Phi_{-1} \tilde \varphi(z_1)}_2 \\
        =& o(dk) \bigO{\frac{1}{dk}} o(dk) = o(dk),
    \end{aligned}
    \end{equation}
    where the third step is justified by Lemmas \ref{lemma:evminntk} and \ref{lemma:Pntksmall}, and holds with probability at least $1 - N e^{-c \log^2 k} - e^{-c \log^2 N}$ over $Z$, $x$, and $W$. A similar argument can be used to show that $\norm{\Ppm \tilde \varphi(z_1)}_2^2 =  o(dk)$, which, together with \eqref{eq:smallcrossPnumer} and \eqref{eq:bigtermspenultimatentk}, and a straightforward application of the triangle inequality, provides the thesis.
\end{proof}

\begin{lemma}\label{lemma:ntkthirdlaststep}
    We have
    \begin{equation}
        \left| \frac{\varphi(z_1^s)^\top \Ppm^\perp \varphi(z_1)}{\norm{\Ppm^\perp \varphi(z_1)}_2^2}  - \frac{\tilde \varphi(z_1^s)^\top \Ppm^\perp \tilde \varphi(z_1)}{\norm{\Ppm^\perp \tilde \varphi(z_1)}_2^2} \right| = o(1),
    \end{equation}
    with probability at least $1 - N \exp(-c \log^2 k) - \exp(-c \log^2 N)$ over $Z$, $x$ and $W$, where $c$ is an absolute constant.
\end{lemma}
\begin{proof}
    If $\mu'_0 = 0$, the thesis is trivial, as $\varphi \equiv \tilde \varphi$. If $\mu'_0 \neq 0$, we can apply Lemma \ref{lemma:Ppntkon1}, and the proof proceeds as follows.
    
    First, we notice that the second term in the modulus in the statement corresponds to the first term in the statement of Lemma \ref{lemma:ntkpenultimatestep}. We will condition on the result of Lemma \ref{lemma:ntkpenultimatestep} to hold until the end of the proof. Notice that this also implies
    \begin{equation}\label{eq:ntkfirststep0}
        \tilde \varphi(z_1^s)^\top \Ppm^\perp \tilde \varphi(z_1) = \Theta(dk), \qquad \norm{\Ppm^\perp \tilde \varphi(z_1)}_2^2 = \Theta(dk),
    \end{equation}
    with probability at least $1 - N \exp(-c \log^2 k) - \exp(-c \log^2 N)$ over $Z$, $x$, and $W$.
    Due to Lemma \ref{lemma:Ppntkon1}, we jointly have
    \begin{equation}\label{eq:ntkfirststep1}
        \norm{\Ppm^\perp \left( z_1 \otimes \mathbf 1_k \right)}_2 = o(\sqrt {dk}), \qquad \norm{\Ppm^\perp  \left( z_1^s \otimes \mathbf 1_k \right)}_2 = o(\sqrt {dk}),
    \end{equation}
    with probability at least $1 - \exp(c \log^2 N)$ over $Z_{-1}$ and $W$.
    Also, by Lemma \ref{lemma:justBern} and Assumption \ref{ass:overparam}, we jointly have
    \begin{equation}\label{eq:ntkfirststep2}
        \norm{\Ppm^\perp \varphi(z_1^s)}_2 \leq \norm{\varphi(z_1^s)}_2 = \norm{z_1^s}_2 \norm{\phi'(W z_1^s)}_ 2 = \bigO{\sqrt{dk}},
    \end{equation}
    and
    \begin{equation}\label{eq:ntkfirststep3}
        \norm{\Ppm^\perp \tilde \varphi(z_1)}_2 \leq \norm{\tilde \varphi(z_1)}_2 = \norm{z_1}_2 \norm{\tilde \phi'(W z_1)}_ 2 = \bigO{\sqrt{dk}},
    \end{equation}
    with probability at least $1 - \exp(-c_1 \log^2 N)$ over $W$. We will condition also on such high probability events (\eqref{eq:ntkfirststep1}, \eqref{eq:ntkfirststep2}, \eqref{eq:ntkfirststep3}) until the end of the proof.
    Thus, we can write
    \begin{equation}\label{eq:ntkfirststep4}
    \begin{aligned}
        &\left| \varphi(z_1^s)^\top \Ppm^\perp \varphi(z_1) - \tilde \varphi(z_1^s)^\top \Ppm^\perp \tilde \varphi(z_1) \right|  \\
        &\leq \left| \varphi(z_1^s)^\top \Ppm^\perp \left( \varphi(z_1) - \tilde \varphi(z_1) \right)  \right| + \left|\left( \varphi(z_1^s) -\tilde \varphi(z_1^s) \right)^\top \Ppm^\perp \tilde \varphi(z_1)  \right|  \\
        &\leq \norm{\Ppm^\perp \varphi(z_1^s)}_2 \norm{\Ppm^\perp  \left( z_1 \otimes \mu_0 \mathbf 1_k \right)}_2 + \norm{\Ppm^\perp \tilde \varphi(z_1)}_2 \norm{\Ppm^\perp  \left( z_1^s \otimes \mu_0 \mathbf 1_k \right)}_2  = o(dk),
    \end{aligned}
    \end{equation}
    where in the last step we use \eqref{eq:ntkfirststep1}, \eqref{eq:ntkfirststep2}, and \eqref{eq:ntkfirststep3}.
    Similarly, we can show that
    \begin{equation}\label{eq:ntkfirststep5}
    \begin{aligned}
        \left| \norm{\Ppm^\perp \varphi(z_1)}_2 - \norm{\Ppm^\perp \tilde \varphi(z_1)}_2 \right| &\leq \norm{\Ppm^\perp \varphi(z_1) - \Ppm^\perp \tilde \varphi(z_1)}_2 \\
        &\leq \norm{\Ppm^\perp \left( z_1 \otimes \mu_0 \mathbf 1_k  \right)}_2= o(\sqrt{dk}).
    \end{aligned}
    \end{equation}
    By combining 
    \eqref{eq:ntkfirststep0}, \eqref{eq:ntkfirststep4}, and \eqref{eq:ntkfirststep5}, the desired result readily follows.
\end{proof}

Finally, we are ready to give the proof of Theorem \ref{thm:mainntk}.

\begin{proof}[Proof of Theorem \ref{thm:mainntk}]
    We have
    \begin{equation}
    \begin{aligned}
         \left| \frac{\varphi(z_1^s)^\top \Ppm^\perp \varphi(z_1)}{\norm{\Ppm^\perp \varphi(z_1)}_2^2} - \alpha \, \frac{ \sum_{l = 1}^{+\infty} {\mu'_l}^2 \alpha^i }{\sum_{l = 1}^{+\infty} {\mu'_l}^2}  \right| & \leq \left| \frac{\varphi(z_1^s)^\top \Ppm^\perp \varphi(z_1)}{\norm{\Ppm^\perp \varphi(z_1)}_2^2}  - \frac{\tilde \varphi(z_1^s)^\top \Ppm^\perp \tilde \varphi(z_1)}{\norm{\Ppm^\perp \tilde \varphi(z_1)}_2^2} \right| \\
         & + \left| \frac{\tilde \varphi(z_1^s)^\top  \tilde \varphi(z_1) -\tilde \varphi(z_1^s)^\top \Ppm \tilde \varphi(z_1)}{\norm{\tilde \varphi(z_1) - \Ppm \tilde \varphi(z_1)}_2^2} - \frac{\tilde \varphi(z_1^s)^\top  \tilde \varphi(z_1)}{\norm{\tilde \varphi(z_1)}_2^2} \right| \\
         & + \left| \frac{\tilde \varphi(z_1^s)^\top  \tilde \varphi(z_1)}{\norm{\tilde \varphi(z_1)}_2^2} - \alpha \, \frac{ \sum_{l = 1}^{+\infty} {\mu'_l}^2 \alpha^i }{\sum_{l = 1}^{+\infty} {\mu'_l}^2} \right| \\
         & = o(1),
    \end{aligned}
    \end{equation}
    where the first step is justified by the triangle inequality, and the second by Lemmas \ref{lemma:ntkthirdlaststep}, \ref{lemma:ntkpenultimatestep}, and \ref{lemma:ntklaststep}, and it holds with probability at least $1 - N \exp(-c \log^2 k) - \exp(-c \log^2 N)$ over $Z$, $x$, and $W$.
\end{proof}

\section{Additional experiments}\label{app:exp}


\begin{figure*}[!t]
  \begin{center}
    \includegraphics[width=.99\textwidth]{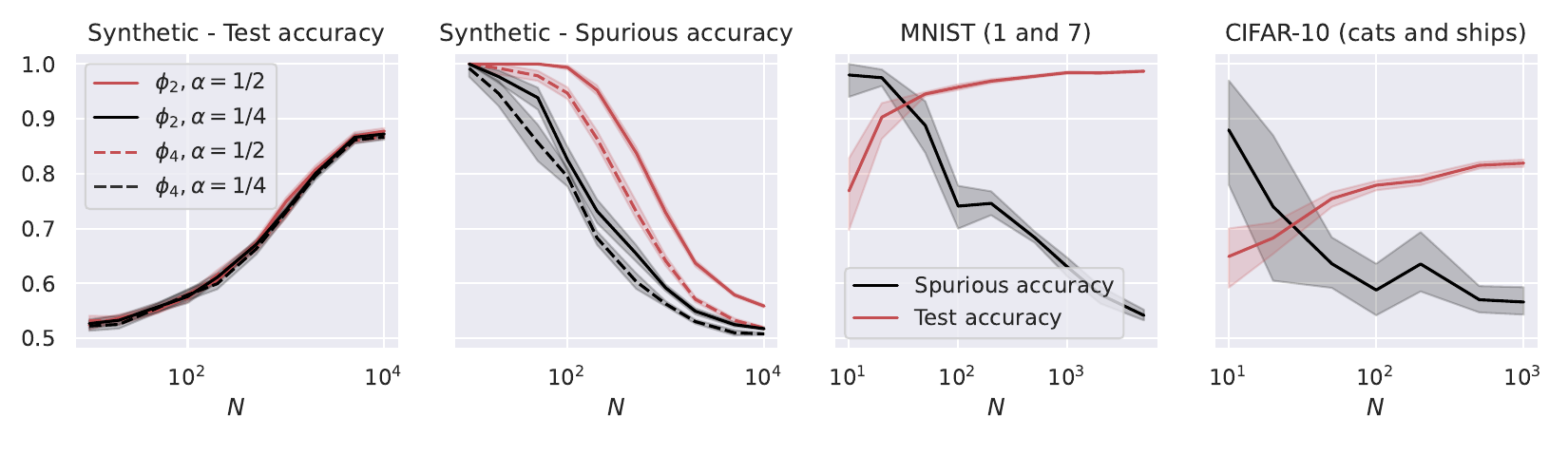}
  \end{center}
  \caption{Test and spurious accuracies as a function of the number of training samples $N$, for various binary classification tasks. In the first two plots, we consider the NTK model in \eqref{eq:NTKmodel} with $k = 100$ trained over Gaussian data with $d = 1000$. The labeling function is $g(x) = \textup{sign}(u^\top x)$. We repeat the experiments for $\alpha = \{ 0.25, 0.5 \}$, and for the two activations whose derivatives are $\phi'_2 = h_0 + h_1$ and $\phi'_4 = h_0 + h_3$, where $h_i$ denotes the $i$-th Hermite polynomial (see Appendix \ref{app:Hermite}). In the last two plots, we consider the same model with ReLU activation, trained over two MNIST and CIFAR-10 classes. The width of the noise background is $10$ pixels for MNIST and $8$ pixels for CIFAR-10, see Figure \ref{fig:cat}. The spurious accuracy is obtained by querying the model only with the noise background from the training set, replacing all the other pixels with $0$, and taking the sign of the output. As we consider binary classification, an accuracy of 0.5 is achieved by random guessing. We plot the average over 10 independent trials
  and the confidence band at 1 standard deviation.}
  \label{fig:ntk}
\end{figure*}

Figure \ref{fig:ntk} reports the experiments on NTK features for the same setting considered in Figure \ref{fig:rf} for random features. We consider binary classification tasks involving synthetic (first two plots) and standard (last two plots) datasets. As predicted by Theorem \ref{thm:mainntk}, when the number of samples $N$ increases, the test accuracy increases and, correspondingly, the spurious accuracy decreases. Furthermore, for the synthetic dataset, while the test accuracy does not depend on $\alpha$ and on the activation function, the spurious accuracy increases with $\alpha$ and by taking an activation function with dominant low-order Hermite coefficients.

\begin{figure}[!t]
  \begin{center}
    \includegraphics[width=\textwidth]{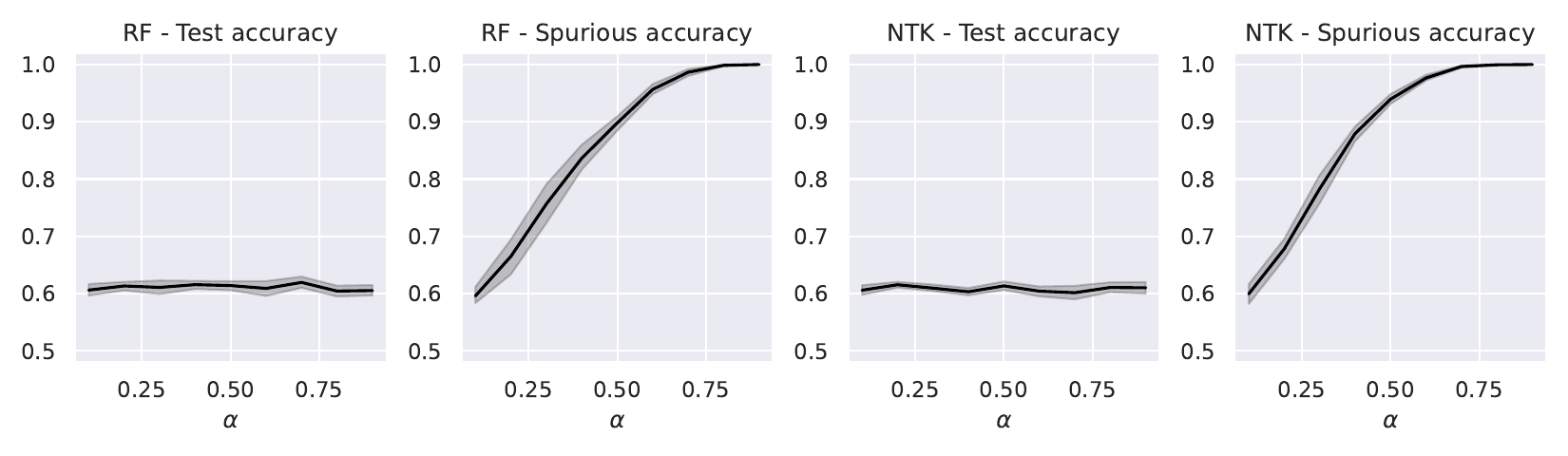}
  \end{center}
  \caption{Test and spurious accuracies as a function of $\alpha$. We consider RF (first and second plot) and NTK (third and fourth plot) models trained on a synthetic dataset. The settings are the same as in Figures \ref{fig:rf} and \ref{fig:ntk}, and we use a ReLU activation function. The number of training samples is fixed to $N = 200$.}
  \label{fig:varying_alpha}
\end{figure}

In Figure \ref{fig:varying_alpha}, we plot the test and the spurious accuracies as a function of $0 < \alpha < 1$. While the test accuracy does not depend on $\alpha$, the spurious accuracy monotonically grows with $\alpha$. This is in agreement with the results of Theorems \ref{thm:RF} and \ref{thm:mainntk}.

\end{document}